\documentclass{article}

\usepackage{hyperref}

\usepackage[accepted]{icml2023}

\usepackage[T1]{fontenc}
\usepackage[latin9]{inputenc}
\usepackage{refstyle}
\usepackage{mathrsfs}
\usepackage{amsmath}
\usepackage{amsthm}
\usepackage{amssymb}

\usepackage{graphicx}
\usepackage{caption}
\usepackage{subcaption}
\usepackage{xcolor}
\usepackage{dblfloatfix}    
\definecolor{mycolor}{rgb}{0,0,0}

\usepackage{booktabs}
\usepackage{xurl}
\usepackage{array}
\usepackage{multirow}
\newcolumntype{C}[1]{>{\centering\arraybackslash$}m{#1}<{$}}

\usepackage{mathtools}
\usepackage{thmtools}
\usepackage{thm-restate}

\theoremstyle{plain}
\newtheorem{theorem}{Theorem}[section]
\newtheorem{proposition}[theorem]{Proposition}
\newtheorem{lemma}[theorem]{Lemma}
\newtheorem{corollary}[theorem]{Corollary}
\theoremstyle{definition}
\newtheorem{definition}[theorem]{Definition}

\theoremstyle{remark}

\usepackage[textsize=tiny]{todonotes}

\makeatletter
\AtBeginDocument{\providecommand\propref[1]{\ref{prop:#1}}}
\AtBeginDocument{\providecommand\thmref[1]{\ref{thm:#1}}}
\AtBeginDocument{\providecommand\algref[1]{\ref{alg:#1}}}
\AtBeginDocument{\providecommand\lemref[1]{\ref{lem:#1}}}
\floatstyle{ruled}
\newfloat{algorithm}{tbp}{loa}
\providecommand{\algorithmname}{Algorithm}
\floatname{algorithm}{\protect\algorithmname}
\RS@ifundefined{subsecref}
  {\newref{subsec}{name = \RSsectxt}}
  {}
\RS@ifundefined{thmref}
  {\def\RSthmtxt{theorem~}\newref{thm}{name = \RSthmtxt}}
  {}
\RS@ifundefined{lemref}
  {\def\RSlemtxt{lemma~}\newref{lem}{name = \RSlemtxt}}
  {}
\theoremstyle{plain}

\theoremstyle{plain}

\theoremstyle{definition}

\theoremstyle{plain}

\newref{prop}{name=Proposition~,Name=Proposition~}
\newref{alg}{name=Algorithm~,Name=Algorithm~}
\newref{def}{name=Definition~,Name=Definition~}
\newref{thm}{name=Theorem~,Name=Theorem~}
\newref{lem}{name=Lemma~,Name=Lemma~}
\newref{cor}{name=Corollary~,Name=Corollary~}
\newref{sec}{name=Section~,Name=Section~}

\makeatother

\providecommand{\definitionname}{Definition}
\providecommand{\lemmaname}{Lemma}
\providecommand{\propositionname}{Proposition}
\providecommand{\theoremname}{Theorem}

\begin{document}
\global\long\def\R{\mathbb{R}}%
\global\long\def\S{\mathbb{S}}%
\global\long\def\AA{\mathcal{A}}%
\global\long\def\FF{\mathcal{F}}%
\global\long\def\C{\mathcal{C}}%
\global\long\def\L{\mathscr{L}}%
\global\long\def\one{\mathbf{1}}%
\global\long\def\gap{\mathrm{gap}}%
\global\long\def\feas{\mathrm{feas}}%
\global\long\def\inner#1#2{\left\langle #1,#2\right\rangle }%

\global\long\def\ub{\mathrm{ub}}%
\global\long\def\lb{\mathrm{lb}}%

\global\long\def\x{\mathbf{x}}%
\global\long\def\e{\mathbf{e}}%
\global\long\def\u{\mathbf{u}}%
\global\long\def\v{\mathbf{v}}%
\global\long\def\z{\mathbf{z}}%
\global\long\def\f{\mathbf{f}}%
\global\long\def\F{\mathcal{F}}%
\global\long\def\G{\mathcal{G}}%
\global\long\def\H{\mathcal{H}}%

\global\long\def\bvec{\mathsf{b}}%
\global\long\def\Wmat{\mathsf{W}}%
\global\long\def\wvec{\mathsf{w}}%
\global\long\def\image{\mathsf{x}}%

\global\long\def\X{\mathbf{X}}%
\global\long\def\U{\mathbf{U}}%
\global\long\def\V{\mathbf{V}}%
\global\long\def\b#1{\mathbf{#1}}%
\global\long\def\I{\mathcal{I}}%
\global\long\def\E{\mathcal{E}}%

\global\long\def\rank{\operatorname{rank}}%

\global\long\def\dist{\operatorname{dist}}%

\global\long\def\diag{\operatorname{diag}}%

\global\long\def\tr{\operatorname{tr}}%


\global\long\def\forall{\mbox{for all }}%

\global\long\def\vec{\operatorname{vec}}%

\renewcommand{\arraystretch}{1.2}
\twocolumn[
\icmltitlerunning{Tight Certification of Adversarially Trained Neural Networks 
via Nonconvex Low-Rank Semidefinite Relaxations}
\icmltitle{Tight Certification of Adversarially Trained Neural Networks \\
via Nonconvex Low-Rank Semidefinite Relaxations}




\begin{icmlauthorlist}
\icmlauthor{Hong-Ming Chiu}{yyy}
\icmlauthor{Richard Y. Zhang}{yyy}
\end{icmlauthorlist}

\icmlaffiliation{yyy}{Department of Electrical and Computer Engineering, University of Illinois at Urbana--Champaign}

\icmlcorrespondingauthor{Hong-Ming Chiu}{hmchiu2@illinois.edu}
\icmlcorrespondingauthor{Richard Y. Zhang}{ryz@illinois.edu}

\icmlkeywords{Machine Learning}

\vskip 0.3in]



\printAffiliationsAndNotice{}  

\begin{abstract}
Adversarial training is well-known to produce high-quality neural network models that are empirically robust against adversarial perturbations. Nevertheless, once a model has been adversarially trained, one often desires a certification that the model is truly robust against all future attacks. Unfortunately, when faced with adversarially trained models, all existing approaches have significant trouble making certifications that are strong enough to be practically useful. Linear programming (LP) techniques in particular face a ``convex relaxation barrier'' that prevent them from making high-quality certifications, even after refinement with mixed-integer linear programming (MILP)  \textcolor{mycolor}{and branch-and-bound (BnB)} techniques. 
In this paper, we propose a nonconvex certification technique, based on a low-rank restriction of a semidefinite programming (SDP) relaxation. The nonconvex relaxation makes strong certifications comparable to much more expensive SDP methods, while optimizing over dramatically fewer variables comparable to much weaker LP methods. Despite nonconvexity, we show how off-the-shelf local optimization algorithms can be used to achieve and to certify global optimality in polynomial time. Our experiments find that the nonconvex relaxation almost completely closes the gap towards exact certification of adversarially trained models.
\end{abstract}

\section{Introduction}
\begin{figure}[t!]
    \centering
    \includegraphics[width=\linewidth]{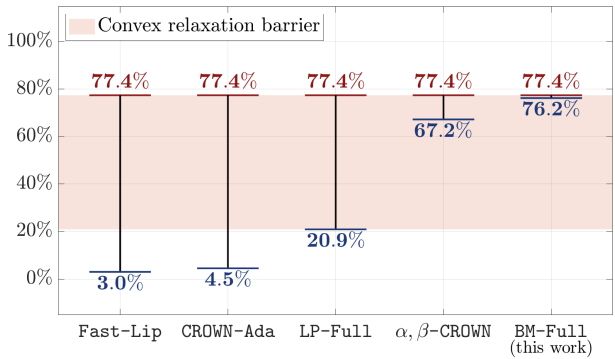}
    \captionsetup{font=footnotesize}
    \caption{\textbf{Certified adversarial training vs. $\ell_2$ attacks}. While PGD attacks indicate an empirical upper-bound of 77.4\% (red), state-of-the-art LP-based verifiers face a ``convex relaxation barrier'' (pink shading) that prevent them from certifying lower-bounds better than 20.9\% (blue). Even an award-winning state-of-the-art branch-and-bound verifier like $\alpha,\beta$-CROWN cannot significantly improve past 67.2\% in reasonable time. Our nonconvex relaxation overcomes the convex relaxation barrier, certifying a lower-bound of 76.2\% that almost fully closes the empirical-certified gap. 
    (See Section~\ref{sec:experiment} for details.)}
    \label{fig:page1}
    \vspace{-1em}
\end{figure}

To make neural network models robust to adversarial perturbation attacks, one popular strategy,
known as \emph{adversarial training}~\citep{kurakin2016adversarial,goodfellow2015explaining,madry2017towards,shafahi2019adversarial,wong2019fast},
is to attack a pre-trained model, and then to re-train the model with
the training set augmented or replaced by the attack. Despite its
simplicity, adversarial training works remarkably well in practice.
For example, the robust models adversarially trained by \citet{madry2017towards}
back in 2017 remain essentially unbroken in 2022, after more than
four years of white-box penetration testing by researchers world-wide.
Later, \citet{shafahi2019adversarial,wong2019fast} have extended the idea to train robust ImageNet classifiers, that achieve a similar level of accuracy, and within a comparable amount of training time, to nonrobust classifiers.

Nevertheless, adversarial training is an empirical strategy that does not promise a truly robust model. Given a model that has been made empirically robust through adversarial training, one often desires a formal mathematical proof or \emph{certification} that the model is truly robust against all future attacks. Unfortunately, when faced with an adversarially trained model, all existing approaches have significant trouble making certifications that are strong enough to be practically useful. A model that achieves $77.4\%$ test accuracy on adversarial inputs might only have a certified robust accuracy of $20.9\%$, using state-of-the-art methods \citep{weng2018towards,zhang2018efficient_nn,salman2019convex} based on a linear programming (LP) relaxation \citep{wong2018provable} of the ReLU activation (see Figure~\ref{fig:page1}).

In fact, recent work by \citet{salman2019convex} suggest that it is fundamentally impossible for \emph{any} method based on the LP relaxation to make substantially better certifications for adversarially trained models. Even mixed-integer linear programming (MILP) techniques like branch-and-bound and cutting planes \citep{tjeng2017evaluating,xu2021fast,wang2021beta}, which in theory are capable of exact certification given unlimited time, cannot in practice significantly close the gap left by the LP relaxation within reasonable time. 
Applying $\alpha,\beta$-CROWN \citep{zhang2018crown,wang2021beta,xu2021fast,zhang2022general}, the winning entry in the International Verification of Neural Networks Competition (VNN-COMP) competitions of 2021 and 2022, only improves the certified robust accuracy to $67.2\%$ before timing out, still leaving an unsatisfactory optimality gap of $10.2\%$. 
There appears to be an insurmountable ``convex relaxation barrier'', between the high degree of robustness that is empirically observed for adversarially trained models, and the low degree of robustness that can be rigorously certified via the LP relaxation, and mixed-integer programming methods based on the LP relaxation.


One promising direction for overcoming the barrier faced by the LP relaxation is to develop methods based on the semidefinite programming (SDP) relaxation
of \citet{raghunathan2018semidefinite}. Indeed, early experiments \citep{raghunathan2018certified,raghunathan2018semidefinite,zhang2020tightness,dathathri2020enabling,batten2021efficient} all suggest that the SDP relaxation can make significantly stronger certifications than the LP relaxation.
However, the cost of solving the SDP relaxation---while technically
still polynomial time---is so high as to be completely inaccessible. At its core,
the SDP relaxation requires optimizing over an $n\times n$ matrix variable, 
where $n$ is equal
to the total number of ReLU activations, plus the dimension of the
input layer. The fundamental difficulty is the need to store and to optimize over $n^2$ variables: even a single-layer MNIST classifier with 200 ReLU activations requires storing and optimizing over the $\approx 10^6$ elements of a $986\times 986$ matrix, which is already nearing
the limit of state-of-the-art SDP solvers like \citet{mosek2019mosek}, whose worst-case runtime scales as $O(n^6)$. Despite widespread speculation, it is currently unknown whether SDP-based methods will truly be able to provide tight certification for adversarially trained models.


\paragraph{Contributions}

This paper proposes a nonconvex certification technique, based on a \emph{low-rank restriction} of the usual convex SDP
relaxation of the ReLU activation. Rigorously, we show that the nonconvex
relaxation is always at least as tight as the SDP relaxation, but
that it optimizes over the $nr$ elements of an $n\times r$ rectangular
matrix. 
If a small value of $r$ can be used (we later validate this experimentally), then the nonconvex relaxation can make strong certifications comparable to the SDP relaxation, while optimizing over a dramatically smaller number of variables comparable to the LP relaxation. 

The nonconvexity of the relaxation poses a serious issue.
The correctness of our certification hinges critically on our ability
to solve a nonconvex verification problem to \emph{global} optimality, 
and then to \emph{certify} this global optimality,
but large-scale optimization algorithms are only capable of achieving and certifying \emph{local} optimality. 
In this paper, we establish that if a local
minimum for the verification problem is also global, then under a mild constraint qualification, the Lagrange multipliers that certify its local optimality are also guaranteed to certify its global optimality via Lagrangian duality (see \propref{dual} and \thmref{zerodual}). Conversely, if the
local minimum is non-global, then under the same mild constraint qualification, 
the Lagrange multipliers generate
a direction of global improvement (see \thmref{escape}), 
thereby ensuring that local optimization will eventually achieve certified global optimality.

Our experiments provide empirical confirmation of our theoretical
claims. We re-examine the models originally used by \citet{salman2019convex}
to demonstrate the existence of a ``convex relaxation barrier''
for the LP relaxation. Using a relaxation rank $r$ of no more than
10, we re-certify these models using our nonconvex relaxation, in
time comparable to that of the best-possible LP relaxation. Our results
find that even a basic nonconvex relaxation offers a significant reduction
in conservatism. Augmenting the nonconvex relaxation by bound propagation (as
is commonly done for the LP relaxation) allows us to almost fully
close the gap towards exact certification 
(see Figure~\ref{fig:page1}).


\paragraph{Related work}
Robustness certification methods can be broadly divided into exact and conservative methods. Exact methods based on mixed-integer linear programming (MILP)~\citep{tjeng2017evaluating,xu2021fast}
and Satisfiability Modulo Theories (SMT)~\citep{katz2017reluplex} can make necessary and sufficient certifications of robustness, but have worst-case runtimes that scale exponentially with the number of activations. Conservative methods can decline to certify a robust model, but have polynomial worst-case runtime, and therefore tend to be much more scalable in practice. Today, most state-of-the-art certification methods are conservative methods based on a triangle-shaped LP relaxation of the ReLU activation function introduced by \citet{wong2018provable}. In particular, \citep{weng2018evaluating,weng2018towards,zhang2018crown} proposed techniques for  strengthening these LP-based relaxations, by propagating tighter layer-wise upper- and lower-bounds on the ReLU activation function. Later work by \citet{wang2021beta} and \citet{xu2021fast} progressively refine these bounds using MILP \textcolor{mycolor}{and BnB} techniques. \citet{salman2019convex} proposed an optimal LP relaxation that unifies all the existing bound-propagating LP-based relaxation methods. This last paper pointed out that even the optimal LP relaxation has a gap cannot be improved; they refer to this inherent looseness as the ``convex relaxation barrier''. \textcolor{mycolor}{Another line of conservative methods are based on SDP relaxationof the ReLU gate \citep{raghunathan2018semidefinite,dathathri2020enabling}. Later work by \citet{batten2021efficient} further tighten the SDP relaxation using linear cut constraints.}


Our proposed approach can be interpreted as an application
of the Burer--Monteiro approach~\citep{burer2002rank,burer2005local} 
and the Riemannian staircase~\citep{boumal2016non,boumal2020deterministic} for
solving the SDP relaxation of \citet{raghunathan2018semidefinite}. 
Here, we emphasize that the rigorous 
applicability of these prior techniques hinges critically on the
linear independence constraint qualification (LICQ), a highly restrictive condition that is difficult to verify in practice. 
If LICQ does not hold, then the Riemannian staircase can become get stuck at a non-LICQ point, so rigorous global guarantees are lost.
In the existing literature, LICQ is often taken as a strong blanket assumption~\citep{rosen2014rise,carlone2015lagrangian,cohen2019certified,rosen2019se},
but this reduces the Riemannian staircase from a provable algorithm
to an empirical heuristic. In this paper, we formally establish LICQ for the verification problem in \lemref{npcq}. Assuming that local optimization does not get stuck at the ``corner'' of the ReLU (this is the same assumption that allows ReLU models to be trained via gradient descent), it immediately follows that our nonconvex relaxation can be globally optimized in polynomial time.

\section{Background}

\paragraph{Notations}
We use $(x_1,\ldots,x_\ell)$ to denote the vertical concatenation of $x_1,\ldots,x_\ell$, with $x_{k}$ stacking on top of $x_{k+1}$. 
We use ($\{x_k\}_{k=1}^\ell)$ as a shorthand notation for $(x_1,\ldots,x_\ell)$. 
We use the square bracket $x[i]$ and $X[i,j]$ to denote indexing. 
The size-$n$ identity matrix is written as $I_n$; we suppress the subscript $n$ whenever it can be inferred from context.
We write $\mathbf{e}_i$ to denote the $i$-th canonical basis, i.e. the $i$-th column of the appropriate identity matrix. 
We write $\mathrm{diag}(X)$ to extract the diagonal from the matrix $X$, and $\mathrm{diag}(x)$ to convert length-$n$ vector into an $n\times n$ diagonal matrix.
The $i$-th largest eigenvalue of a matrix $X$ is denoted $\lambda_i(X)$. We use $\odot$ to denote the elementwise product.

Consider the task of classifying a data point $\hat{x}\in\R^{p}$
as belonging to the $\hat{c}$-th of $q$ classes. The standard approach
is to train a classifier model $f:\R^{p}\to\R^{q}$ such that the prediction vector $f(\hat{x})$ takes on its maximum value at the $\hat{c}$-th element, as in $f(\hat{x})[\hat{c}]>f(\hat{x})[c]$ for all incorrect labels $c\ne\hat{c}$. In this paper, we focus our
attention on $\ell$-layer feedforward ReLU-based neural networks,
defined recursively 
\begin{align}
f(x)\equiv W_{\ell}x_{\ell}+b_{\ell},\quad x_{k+1}=\max\{0,W_{k}x_{k}+b_{k}\} \nonumber
\end{align}
for $k\in\{1,2,\dots,\ell-1\}$, where $x_{1}\equiv x$ is the input. Throughout the paper, we will use $n_{k}$
to denote the number of neurons at the $k$-th layer, and $n=\sum_{k=1}^{\ell}n_{k}$
to denote the total number of neurons. Note that our convention includes the neurons at the input layer, i.e. $x_{1}\equiv x$, but excludes
those at the output layer, i.e. $f(x)\equiv W_{\ell}x_{\ell}+b_{\ell}$. 

To compute an adversarial example $x\approx\hat{x}$, the standard
approach is to apply projected gradient descent (PGD) to the following \emph{semi-targeted attack} problem, which was first introduced by \cite{carlini2017towards}:
\begin{alignat}{2}
\phi[c]=\min_{x=(x_{1},\dots,x_{\ell})\in\R^{n}}\quad & w_{\ell}^{T}x_{\ell}+w_{0}x_{0} & \tag{A}\label{eq:attack}\\
\text{s.t. }\quad & x_{k+1}=\max\{0,W_{k}x_{k}+b_{k}\},\nonumber \\
 & \|x_{1}-\hat{x}\|\le\rho,\nonumber 
\end{alignat}
for all $k\in\{1,2,\dots,\ell-1\}$, where $w_\ell=(\e_{\hat{c}}-\e_{c})^{T}W_{\ell}$ and $w_0=(\e_{\hat{c}}-\e_{c})^{T}b_{\ell}$.  \footnote{To simplify presentation, we focus our attention on the $\ell_2$ norm, and assume $w_0=0$ without loss of generality. Our results can be extended for the $\ell_\infty$ norm and are included in the appendix.}Robustness to adversarial perturbations can be certified
by verifying that (\ref{eq:attack}) achieves a positive global minimum
$\phi[c]>0$ for every incorrect class $c\ne\hat{c}$. The numerical
value of the minimum global minimum, written
$\phi^{\star}=\min_{c\ne\hat{c}}\phi[c],$
is a \emph{robustness margin} that measures how robust the model is
to adversarial perturbations. The more positive is the robustness
margin $\phi^{\star}$, the more the model is able to resist misclassification.

\section{Rank-constrained SDP relaxation}
Our goal in this paper is to develop better lower-bounds on the semi-targeted attack problem (\ref{eq:attack}). Following existing work on the SDP relaxation, we begin by substituting
the rank-1 SDP reformulation of the ReLU activation in \citep{raghunathan2018semidefinite,zhang2020tightness}
to (\ref{eq:attack}). But instead
of deleting the rank-$1$ constraint altogether, we propose to slightly
relax it to a rank-$r$ constraint with a bounded trace, where $1 \le r \le n+1$, to result in a family of \emph{nonconvex} relaxations
\[
\phi_{r}[c]=\min_{X\in\S^{n+1}} \quad w_{\ell}^{T}x_{\ell} \tag{SDP-\ensuremath{r}}\label{eq:sdp}
\]
 subject to
\begin{align}
 & \tr(X_{1,1})-2x_{1}^{T}\hat{x}_{1}+\|\hat{x}_{1}\|^{2}x_{0}\le \rho^{2}x_{0}, \tag*{$(y_{0})$}\\
 & x_{k+1}\ge0,\quad x_{k+1}\ge W_{k}x_{k}+b_{k}x_{0}, \tag*{$(y_{k,1},y_{k,2})$} \\
 & \diag(X_{ k+1,k+1}-W_{k}X_{k,k+1}-b_{k}x_{k+1}^{T})=0,  \tag*{$(z_{k})$} \\
 & x_{0}=1,\quad \tr(X)\le R^{2} \tag*{$(z_{0},\mu)$}
\end{align}
for all $k\in\{1,2,\dots,\ell-1\}$, whose optimization variable $X$ is an $(n+1)\times(n+1)$ rank-$r$ constrained symmetric
positive semidefinite matrix
\[
X=\left[\begin{array}{c|ccc}
x_{0} & x_{1}^{T} & \cdots & x_{\ell}^{T}\\
\hline x_{1} & X_{1,1} & \cdots & X_{1,\ell}\\
\vdots & \vdots & \ddots & \vdots\\
x_{\ell} & X_{1,\ell}^{T} & \cdots & X_{\ell,\ell}
\end{array}\right]\succeq0,\quad \rank(X)\le r.
\]
Here we assign the dual variable associated with each constraint in the parenthesis. We will assume throughout the paper that the trace bound $R$
has been chosen large enough so that $\tr(X^{\star})<R^{2}$, or equivalently $\mu^\star=0$, holds at optimality. It follows from \citep{zhang2020tightness}
that $r=1$ instance of (\ref{eq:sdp}) coincides with (\ref{eq:attack})
exactly. Due to our use of a rank upper-bound, every subsequent instance
then provides a lower-bound on its previous relaxation:
\[
\phi[c]=\phi_{1}[c]\ge\phi_{2}[c]\ge\cdots\ge\phi_{n+1}[c].
\]
Finally, setting $r=n+1$ has the same effect as deleting the rank
constraint. Therefore, the $r=n+1$ instance coincides with the
convex semidefinite relaxation as originally proposed by \citep{raghunathan2018semidefinite}.

For relaxation ranks of $r<n+1$, the corresponding nonconvex instances
of (\ref{eq:sdp}) are NP-hard in general to solve to global optimality.
Even if we are provided with a globally optimal solution $X^{\star}$,
there is generally no way to (rigorously) tell that $X^{\star}$ is
indeed globally optimal. The most we can say is that $X^{\star}$
provides an upper-bound on the global minimum of (\ref{eq:sdp}). Unfortunately, 
this upper-bound is not helpful in our goal of lower-bounding (\ref{eq:attack}).

Instead, we will derive a lower-bound on  (\ref{eq:sdp}) via Lagrangian duality, which will also serve as a valid lower-bound on (\ref{eq:attack}). Our motivating insight is that all instances of (\ref{eq:sdp}), including
those nonconvex instances with $r<n+1$, have the \emph{same} convex
Lagrangian dual. Define dual variables $y=(y_{0},\{y_{k,1},y_{k,2}\}_{k=1}^{\ell-1})\ge 0$, $z=(y_{0},\{z_{k}\}_{k=1}^{\ell-1})$ and $\mu \le 0$ to correspond to the linear
constraints in (\ref{eq:sdp}) as shown in parentheses. Then, the
dual problem of (\ref{eq:sdp}) is written:
\begin{align}
\max_{y\ge0,\ z,\ \mu\le0}\ z_{0}+R^{2}\mu\quad \text{s.t.}\quad S(y,z)\succeq\mu I,\tag{SDD}\label{eq:sdd}
\end{align}
in which the components of the slack matrix
\[
S(y,z)\equiv\frac{1}{2}\left[\begin{array}{c|cccc}
s_{0} & s_{1}^{T} & s_{2}^{T} & \cdots & s_{\ell}^{T}\\
\hline s_{1} & S_{1,1} & S_{1,2}\\
s_{2} & S_{1,2}^{T} & S_{2,2} & \ddots\\
\vdots &  & \ddots & \ddots & S_{\ell-1,\ell}\\
s_{\ell} &  &  & S_{\ell-1,\ell}^{T} & S_{\ell,\ell}
\end{array}\right]
\]
are written
\begin{gather*}
s_{0}=2\left[y_{0}(\|\hat{x}\|^{2}-\rho^{2})+\sum_{k=1}^{\ell-1}b_{k}^{T}y_{k,2}-z_{0}\right],\\
s_{1}=W_{1}^{T}y_{1,2}-2\hat{x}y_{0},\\
s_{k+1}=W_{k+1}^{T}y_{k+1,2}-\left(Z_{k}b_{k}+y_{k,1}+y_{k,2}\right),\\
s_{\ell}=w_{\ell}-\left[Z_{\ell-1}b_{\ell-1}+y_{\ell-1,1}+y_{\ell-1,2}\right],\\
S_{1,1}=2y_{0}I,\ S_{k,k+1}=-W_{k}^{T}Z_{k},\ S_{k+1,k+1}=2Z_{k},
\end{gather*}
where $Z_{k}=\diag(z_{k})$. Here, $s_{k+1}$ is defined for all $k\in\{1,\ldots\ell-2\}$, and $S_{k,k+1}$ and $S_{k+1,k+1}$ are defined for all $k\in\{1,\ldots\ell-1\}$. We therefore obtain the following lower-bound on the semi-targeted attack problem in (\ref{eq:attack}),
which is valid for \emph{any} choice of multipliers $(y,z)$.

\begin{proposition}[Dual lower-bound]\label{prop:dual}
Let $X^{\star}$ denote the global solution of
(\ref{eq:sdp}) with rank $r\ge1$ and $\tr(X^\star)<R^2$. Then, any dual multipliers $y=(y_{0},\{y_{k,1},y_{k,2}\}_{k=1}^{\ell-1})$
and $z=(z_{0},\{z_{k}\}_{k=1}^{\ell-1})$ that satisfy $y\ge0$ provide
the following lower-bound
\[
\phi[c]\ge\phi_{r}[c]\ge z_{0}+R^{2}\cdot\min\{0,\lambda_{\min}[S(y,z)]\}.
\]
\end{proposition}

Let $X^{\star}$ denote the globally optimal solution for the convex
instance of (\ref{eq:sdp}) with $r=n+1$. It turns out that \emph{strong
duality} is satisfied in this convex case, meaning that there exists
optimal multipliers $y^{\star},z^{\star}$ that exactly satisfy 
\[
\phi[c]\ge\phi_{n+1}[c]=z_{0}^{\star}+R^{2}\cdot\min\{0,\lambda_{\min}[S(y^{\star},z^{\star})]\}
\]
and therefore \emph{certify} the global optimality $X^{\star}$ via
\propref{dual}. Now, suppose that the convex solution $X^{\star}$ is in fact low-rank,
as in $r^{\star}=\rank(X^{\star})\ll n$. The statement below says
that the \emph{nonconvex} instance of (\ref{eq:sdp}) with $r=r^{\star}$
also admits optimal multipliers $y^{\star},z^{\star}$ that certify
global optimality.
\begin{theorem}[Existence of global optimality certificate]
\label{thm:rstar}Let $r^{\star}=\rank(X^{\star})$ and $\tr(X^{\star})<R^2$, where $X^{\star}$ denotes the maximum-rank solution to the convex instance of (\ref{eq:sdp})
with $r=n+1$. Then, there exists optimal multipliers $y^{\star}=(y_{0}^{\star},\{y_{k,1}^{\star},y_{k,2}^{\star}\}_{k=1}^{\ell-1})$
and $z^{\star}=(z_{0}^{\star},\{z_{k}^{\star}\}_{k=1}^{\ell-1})$
that satisfy $y^{\star}\ge0$ and the following
\[
\phi[c]\ge\phi_{r^{\star}}[c]=z_{0}^{\star}+R^{2}\cdot\min\{0,\lambda_{\min}[S(y^{\star},z^{\star})]\}.
\]
\end{theorem}

In the following section, we use an approach of \citet{burer2003nonlinear} and \cite{boumal2016non,boumal2020deterministic} 
to constructively compute the optimal multipliers 
$y^{\star},z^{\star}$ that have been asserted to exist by \thmref{rstar}. In turn, plugging $y^{\star},z^{\star}$ into \propref{dual} produces a tight lower-bound on the semi-targeted attack problem (\ref{eq:attack}), thereby achieving the original goal of this section. 

\section{Solution via Nonlinear Programming}
In order to expose the underlying degrees of freedom in the rank-$r$
matrix $X$, we reformulate problem (\ref{eq:sdp}) into a low-rank
factorization form first proposed by \citep{burer2003nonlinear}:
\begin{align}
 &\phi_{r}[c]=\min_{u_{0},u,V}\quad  u_{0}\cdot(w_{\ell}^{T}u_{\ell}) \tag{BM-\ensuremath{r}}\label{eq:bm}\\
 &\text{subject to} \nonumber \\
 & \|u_{1}-u_{0}\hat{x}\|^{2}+\|V_{1}\|^{2}\le \rho^{2},\quad u_{0}^{2}=1, \tag*{$(y_{0},z_{0})$}\nonumber \\
 & u_{0}\cdot u_{k+1}\ge0, \tag*{$(y_{k,1})$}\nonumber \\
 & u_{0}\cdot(u_{k+1}-W_{k}u_{k}-b_{k}u_{0})\ge0, & \tag*{$(y_{k,2})$}\nonumber \\
  \begin{split}
 	&\diag[(u_{k+1}-W_{k}u_{k}-b_{k}u_{0})u_{k+1}^{T}\\
 	&\qquad +(V_{k+1}-W_{k}V_{k})V_{k+1}^{T}]=0,
  \end{split} \tag*{$(z_{k})$}\nonumber \\
 & u_{0}^{2}+\sum_{k=1}^{\ell-1}(\|u_{k}\|^{2}+\|V_{k}\|^{2})\le R^{2}, \tag*{$(\mu)$}\nonumber
\end{align}
for all $k\in\{1,\dots,\ell-1\}$, and over optimization variables
are $u_{0}\in\R$ and $u=(u_{1},\dots,u_{\ell})\in\R^{n}$ and $V=(V_{1},\dots,V_{\ell})\in\R^{n\times(r-1)}$.
Problem (\ref{eq:bm}) is obtained by substituting the following into
(\ref{eq:sdp})
\begin{equation}
X=\left[\begin{array}{c|c}
u_{0} & 0\\
\hline u_{1} & V_{1}\\
\vdots & \vdots\\
u_{\ell} & V_{\ell}
\end{array}\right]\left[\begin{array}{c|c}
u_{0} & 0\\
\hline u_{1} & V_{1}\\
\vdots & \vdots\\
u_{\ell} & V_{\ell}
\end{array}\right]^{T}=UU^{T}.\label{eq:bmpart}
\end{equation}
The equivalence between these two problems follows because every $(n+1)\times(n+1)$
matrix $X$ of rank $r$ can be factored as $X=LL^{T}$ into a low-rank
Cholesky factor $L$ that is both lower-triangular and of dimensions
$(n+1)\times r$. The advantage of the formulation (\ref{eq:bm}) is that it reduces
the number of explicit variables from the $\frac{1}{2}n(n+1)\approx\frac{1}{2}n^{2}$
in the original matrix $X$ to $nr+1\approx nr$ in the factor matrix
$U$, while also allowing the positive semidefinite constraint $X\succeq0$
to be enforced for free. For moderate values of $r\ll n$, the resulting
instance of (\ref{eq:bm}) contains just $O(n)$ variables and constraints.

We propose solving (\ref{eq:bm}) as an instance of the standard-form
nonlinear program,
\begin{equation}
\min_{\|x\|\le R}\quad f(x)\quad\text{ s.t. }\quad g(x)\le0,\quad h(x)=0,\tag{NLP}\label{eq:nlp}
\end{equation}
using a high-performance general-purpose solver like \texttt{fmincon}
or \texttt{knitro}. These are primal-dual solvers, and are designed
to output a primal point $x=(u_{0},u,\vec(V))$ that is \emph{first-order
optimal}, and dual multipliers $y=(y_{0},\{y_{k,1},y_{k,2}\}_{k=1}^{\ell-1})$
and $z=(z_{0},\{z_{k}\}_{k=1}^{\ell-1})$ that \emph{certify} the first-order
optimality of $x$. 
Below, the notion of first-order optimality is
taken from \citep[Theorem 12.3]{nocedal2006numerical}, and the notion
of certifiability follows from the proof of \citep[Theorem 12.1]{nocedal2006numerical}.
\begin{definition}[Certifiably first-order optimal]
The point $x$ is said to be \emph{first-order optimal} if it satisfies
the constraints $g(x)\le0$ and $h(x)=0$, and there exists no escape
path $x(t)$ that begins at $x(0)=x$ and makes a first-order improvement
to the objective while satisfying all constraints, as in
\begin{equation}
f(x(t))\le f(x)-\delta t,\quad g(x(t))\le0,\quad h(x(t))=0, \nonumber
\end{equation}
for all $t\in[0,\epsilon)$ with sufficiently small $\delta>0$ and $\epsilon>0$. 
Additionally, $x$ is said to be \emph{certifiably} first-order optimal if there
exist dual multipliers $y$ and $z$ that satisfy the Karush--Kuhn--Tucker
(KKT) equations:
\begin{equation}
\begin{gathered}
\nabla f(x)+\nabla g(x)y+\nabla h(x)z=0, \; y\odot g(x)=0,\;  y\ge0. \nonumber
\end{gathered}
\end{equation}
\end{definition}
Our main idea is to simply take the dual multipliers $y,z$ computed
by the nonlinear programming solver, round them $y\gets\max\{0,y\}$ to ensure that $y\ge0$, and then to plug them back into
\propref{dual}. Our main result is that, if $x$ is globally optimal and satisfies a mild constraint qualification, then the corresponding dual multipliers $y,z$ exist and are unique. Therefore, if $x$ is indeed globally optimal, then the dual multipliers $y,z$
that certify the local optimality of $x$ must also coincide with the optimal multipliers $y^{\star},z^{\star}$ that were asserted to existed earlier in \thmref{rstar}.

\begin{lemma}[Nonzero preactivation]\label{lem:npcq}
Suppose we have $x=(u_{0},u_{1},\dots,u_{\ell},\vec(V_{1}),\dots,\vec(V_{\ell}))$ that satisfies:
\begin{equation}
\e_{i}^{T}(W_{k}u_{k}+b_{k}u_{0})\ne0,\quad\e_{i}^{T}W_{k}V_{k}\ne0,\tag{NPCQ}\label{eq:npcq}
\end{equation}
for all $k\in\{1,\dots,\ell-1\}$ and $i\in\{1,\dots,n_{k+1}\}$. Then, $x$ is first-order optimal if and only if there exist dual
multipliers $y$ and $z$ to certify $x$ as being first-order optimal.
Moreover, the choice of dual multipliers $y,z$ is unique. 
\end{lemma}

\begin{theorem}[Zero duality gap]
\label{thm:zerodual}Let $r\ge r^{\star}$ where $r^{\star}$ is
defined in \thmref{rstar}. If $x$ is globally optimal and satisfies
(\ref{eq:npcq}), then the dual multipliers $y$ and $z$ that certify
$x$ to be first-order optimal must also certify $x$ to be globally
optimal, as in
\[
\phi_{r}[c]=u_{0}\cdot(w_{\ell}^{T}u_{\ell})=z_{0}+R^{2}\cdot\max\{0,\lambda_{\min}[S(y,z)]\}.
\]
\end{theorem}

Conversely, if a first-order optimal point $x$  satisfies the constraint qualification but is not globally optimal, then the dual multipliers $y,z$ generate a direction of global improvement towards the global minimum. The key idea is to lift to a higher
relaxation rank $r_{+}=r+1$, in order to make $x$ a saddle point. The statement below gives a direction to
escape the saddle-point and make a decrement.
\begin{theorem}[Escape lifted saddle point]\label{thm:escape}
Let $x$ be certifiably first-order optimal for (\ref{eq:bm}) with dual multipliers
$(y,z)$. If $x$ satisfies $\gamma=-\lambda_{\min}[S(y,z)]>0$, (\ref{eq:npcq}) and $\|x\|<R$,
then the eigenvector $\xi=(\xi_{0},\xi_{1},\xi_{2},\dots,\xi_{\ell})$
that satisfies $\xi^{T}S(y,z)\xi=-\gamma\|\xi\|^{2}$ implicitly defines
an escape path $x_{+}(t)=(u_{0},\{u_{k,+}(t)\}_{k=1}^{\ell},\{V_{k,+}(t)\}_{k=1}^{\ell})$
with
\begin{gather*}
u_{k,+}(t)=u_{k}+O(t^{2}),\\ V_{k,+}(t)=[V_{k},0]+t\cdot[0,u_{k}\xi_{0}/u_{0}+\xi_{k}]+O(t^{2})
\end{gather*}
that makes a second-order improvement to the objective while satisfying
all constraints, as in
\[
f(x_{+}(t))=f(x)-t^{2}\gamma,\quad g(x_{+}(t))\le0,\quad h(x_{+}(t))=0
\]
for all $t\in[0,\epsilon)$ with sufficiently small but nonzero $\epsilon>0$.
\end{theorem}

\begin{algorithm}[h!]
\caption{\label{alg:staircase}Summary of proposed algorithm}

\textbf{Input:} Initial relaxation rank $r\ge2$. Weights $W_{1},\dots,W_{\ell}$
and biases $b_{1},\dots,b_{\ell}$. Original
input $\hat{x}$, true label $\hat{c}$, target label $c$, and perturbation
size $\rho$. Variable radius bound $R$.

\textbf{Output:} Lower-bound $\phi_{\lb}[c]\le\phi[c]$ on the optimal
value of the semi-targeted attack problem (\ref{eq:attack}).

\textbf{Algorithm:}
\begin{enumerate}
\item (Solve rank-$r$ relaxation) Use a nonlinear programming solver to
solve the following
\begin{align*}
 & \min_{\|x\|\le R}\quad f(x) \equiv(\e_{c}-\e_{\hat{c}})^{T}(W_{\ell}u_{\ell}u_{0}+b_{\ell})\\
 & \text{subject to} \nonumber \\
 & g_{0}(x) \equiv\|u_{1}-u_{0}\hat{x}\|^{2}+\|V_{1}\|^{2}-\rho^{2} \le0, \tag*{$(y_{0})$}\\
 & h_{0}(x) \equiv1-u_{0}^{2} = 0, \tag*{$(z_{0})$}\\
 & g_{k}(x) \equiv\begin{bmatrix}-u_{0}u_{k+1} \\u_{0}(W_{k}u_{k}+b_{k}u_{0}-u_{k+1})\end{bmatrix} \le0, \tag*{$\begin{matrix}(y_{k,1})\\(y_{k,2})\end{matrix}$}\\
\begin{split}
 & h_{k}(x) \equiv\diag[(u_{k+1}-W_{k}u_{k}-b_{k}u_{0})u_{k+1}^{T}\\
 &\qquad\qquad+(V_{k+1}-W_{k}V_{k})V_{k+1}^{T}] =0,
\end{split} \tag*{$(z_{k})$}
\end{align*}
for all $k$ and over $x=(u_{0},\{u_{k}\}_{k=1}^{\ell},\{\vec(V_{k})\}_{k=1}^{\ell})$.
Retrieve the corresponding dual multipliers\\ $y=(y_{0},\{y_{k,1},y_{k,2}\}_{k=1}^{\ell-1})$
and $z=(z_{0},\{z_{k}\}_{k=1}^{\ell-1})$.
\item (Check certifiable first-order optimality) If $\|\nabla f(x)+\nabla g(x)y+\nabla h(x)z\|$
is sufficiently small, and if $g(x)\le0$ and $h(x)=0$ and $\|x\|<R$
hold to sufficient tolerance, then continue. Otherwise, return error due to solver's inability to achieve first-order optimality.
\item (Check dual feasibility) If $\epsilon_{\feas}=-\lambda_{\min}[S(y,z)]$
is sufficiently small, where the slack matrix $S(y,z)$ is defined
in (\ref{eq:sdd}), then return $\phi_{\lb}[c]=z_{0}-\epsilon_{\feas}\cdot R^2.$
Otherwise, continue.
\item (Escape lifted saddle point) Compute the eigenvector $\xi=(\xi_{0},\xi_{1},\dots,\xi_{\ell})$
satisfying $\|\xi\|=1$ and $\xi^{T}S(y,z)\xi=-\epsilon_{\feas}$.
Set up new primal initial point $x_{+}=(u_{0},\{u_{k}\}_{k=1}^{\ell},\{\vec(V_{+,k})\}_{k=1}^{\ell})$ where
\[
V_{+,k}=[V_{k},0]+\epsilon\cdot[0,u_{k}\xi_{0}/u_{0}+\xi_{k}].
\]
Increment $r\gets r+1$ and repeat Step 1 with $(x_{+},y,z)$ as the
initial point.
\end{enumerate}
\end{algorithm}

In practice, it suffices to move
along the \emph{straight} path $\tilde{u}_{k,+}(t)=u_{k}$ and $\tilde{V}_{k,+}(t)=[V_{k},0]+t\cdot[0,u_{k}\xi_{0}/u_{0}+\xi_{k}]$
then solve for feasibility $g(x)\le0$ and $h(x)=0$. Concretely,
after computing a first-order optimal $x$, we increment the relaxation
rank $r_{+}=r+1$, and initialize the nonlinear programming solver
using the lifted point $\tilde{x}_{+}(\epsilon)$ as the initial primal
point, and the old multipliers $y,z$ as the initial dual multipliers.
If this arrives at the global optimum, then the corresponding $y,z$
must certify $x$ as being so. Otherwise, we repeat the rank lifting
procedure.

Progressively lifting the relaxation rank $r$, in our experience
it takes no more than $r\le10$ to reduce the duality gap to values
of $10^{-8}$. To rigorously guarantee a zero duality gap, however,
can require a relaxation rank on the order of $r=O(\sqrt{n})$~\citep{boumal2020deterministic},
irrespective of the value of $r^{\star}$. Indeed, counterexamples
with $\epsilon_{\feas}>0$ exist for relaxation ranks $r$ that are
even slightly smaller than this threshold~\citep{waldspurger2020rank,ocarrol2022burer,zhang2022improved}.
As a purely theoretical result, the ability to achieve a zero duality
gap implies that the nonconvex relaxation (with $r\ge r^{\star}$)
can be solved in polynomial time (and is therefore not NP-hard). Of
course, setting $r=O(\sqrt{n})$ would also force us to optimize over
$O(n^{3/2})$ variables, thereby offsetting much of our computational
advantage against the usual convex SDP relaxation in practice.

\algref{staircase} summarizes our proposed approach in pseudocode
form, and introduces a number of small practical refinements. 

\begin{table*}[!b]
\captionsetup{font=footnotesize}
\caption{\textbf{Robustness verification for neural networks.} We compare the number of images certified as \emph{robust} by \texttt{BM-Full}, \texttt{BM}, $\alpha,\beta$-\texttt{CROWN}, \texttt{LP-Full}, \texttt{CROWN-Ada} and \texttt{Fast-Lip} within the first 1000 images for normally and robustly trained networks. The upper bound (denoted as UB in the table) on the true number of robust images is obtained by \texttt{PGD}.}
\begin{center}
\resizebox{\linewidth}{!}{%
\begin{tabular}{
lcc c@{\hskip 5pt}c c@{\hskip 5pt}c c@{\hskip 5pt}c c@{\hskip 5pt}c c@{\hskip 5pt}c c@{\hskip 5pt}c
} 
\toprule  
  \multicolumn{1}{c}{Network} & $\ell_2$ Radius & \texttt{PGD}   & \multicolumn{2}{c}{\texttt{BM-Full}} & \multicolumn{2}{c}{\texttt{BM}} & \multicolumn{2}{c}{$\alpha,\beta$\texttt{-CROWN}} & \multicolumn{2}{c}{\texttt{LP-Full}} & \multicolumn{2}{c}{\texttt{CROWN-Ada}} & \multicolumn{2}{c}{\texttt{Fast-Lip}} \\
   \midrule
   &  & \footnotesize{UB} & \footnotesize{Robust} & \footnotesize{Time} & \footnotesize{Robust} & \footnotesize{Time}&\footnotesize{Robust} & \footnotesize{Time}&\footnotesize{Robust} & \footnotesize{Time}& \footnotesize{Robust}&\footnotesize{Time}& \footnotesize{Robust}&\footnotesize{Time}\\
\cmidrule(lr){3-3}\cmidrule(lr){4-5}\cmidrule(lr){6-7}\cmidrule(lr){8-9}\cmidrule(lr){10-11}\cmidrule(lr){12-13}\cmidrule(lr){14-15}
   ADV-MNIST & 1.0 & 774 & $\bf762$ &  47s & 757 &  28s & 672 &  24s & 209 &  8s &  45 & 12ms &  30 & 13ms\\
   ADV-MNIST & 1.3 & 614 & $\bf569$ &  38s & 559 &  28s & 399 &  94s &  25 & 10s &   7 &  9ms &   2 & 12ms\\
   ADV-MNIST & 1.5 & 471 & $\bf411$ &  56s & 392 &  20s & 248 & 138s &  11 & 11s &   1 &  9ms &   1 & 13ms\\
   \midrule
   LPD-MNIST & 1.0 & 755 & $\bf730$ & 218s & 708 &  29s & 641 &  10s & 411 & 16s & 120 & 10ms &  66 & 13ms\\
   LPD-MNIST & 1.3 & 612 & $\bf514$ & 129s & 474 &  26s & 430 &  33s &  61 & 19s &  16 & 10ms &   8 & 13ms\\
   LPD-MNIST & 1.5 & 505 & $\bf391$ &  98s & 350 &  23s & 316 &  64s &  23 & 20s &   5 & 10ms &   2 & 14ms\\
   \midrule
   NOR-MNIST & 0.3 & 916 & $\bf911$ & 128s & 866 &  21s & 797 &  23s & 728 &  8s & 420 &  9ms & 348 & 12ms\\
   NOR-MNIST & 0.5 & 732 & $\bf696$ & 127s & 534 &  27s & 424 & 159s & 232 & 16s &  46 &  7ms &  27 & 13ms\\
   NOR-MNIST & 0.7 & 485 & $\bf381$ & 156s & 187 &  30s & 124 & 253s &  37 & 19s &   0 & 13ms &   0 & 17ms\\
\bottomrule
\end{tabular}
}
\end{center}
\label{table:certify_img}
\end{table*}

\section{Experiments}\label{sec:experiment}
We identically reproduce three models from \citet{salman2019convex}, two of which were trained to be robust against an $\ell_\infty$ adversary. We compare the  performance of our proposed 
verifier \texttt{BM}, which is based on solving (\ref{eq:bm}),  and \texttt{BM-Full}, which is an extension of \texttt{BM} with the addition of layer-wise preactivation bounds (see Appendix~\ref{app:implementation} for details), against state-of-the-art LP-based verifiers for certifying robustness against an $\ell_2$ adversary. Our experiments for certifying the robustness of the same models against an $\ell_\infty$ adversary are deferred to the appendix.
\paragraph{Methods.} 
\texttt{BM} and \texttt{BM-Full} denote the proposed method without and with preactivation bounds respectively. \textcolor{mycolor}{The source code for \texttt{BM} and \texttt{BM-Full} are available at \url{https://github.com/Hong-Ming/BM-r}.} \texttt{PGD} denotes the projected gradient descent algorithm for finding the upper bound on (\ref{eq:attack}). We compare \texttt{BM} and \texttt{BM-Full} to three state-of-the-art LP-based verifiers: \texttt{CROWN-Ada} of \citet{zhang2018efficient_nn}, \texttt{Fast-Lip} of \citet{weng2018towards}, and \texttt{LP-Full} of \citet{salman2019convex}. \texttt{CROWN-Ada} and \texttt{Fast-Lip} are both large-scale LP verifiers that have linear complexity with respect to the number of activations; the implementations that we used were taken directly from the authors' project page\footnote{\tiny\url{https://github.com/IBM/CROWN-Robustness-Certification}}. \texttt{LP-Full} is the optimal LP verifier that uses the tightest possible preactivation bounds by solving LP problems for each hidden neuron; its complexity is cubic with respect to the number of activations. We reimplemented this algorithm to work with $\ell_2$ adversaries, and then validated our implementation against that of the authors\footnote{\tiny\url{https://github.com/Hadisalman/robust-verify-benchmark}\label{refnote}} on $\ell_\infty$ adversaries. The preactivation bounds in \texttt{BM-Full} are set to coincide with those used in \texttt{LP-Full}. \textcolor{mycolor}{We also compare our methods against the state-of-the-art branch-and-bound verifier $\alpha,\beta$-\texttt{CROWN} \citep{zhang2018crown,wang2021beta,xu2021fast,zhang2022general}. The implementation of $\alpha,\beta$-\texttt{CROWN} are taken from authors' project page \footnote{\tiny\url{https://github.com/Verified-Intelligence/alpha-beta-CROWN}}. We set the timeout of $\alpha,\beta$-\texttt{CROWN} to be 300s.}

\begin{figure*}[t!]
    \centering
    \begin{subfigure}{0.333\textwidth}
      \centering
      \includegraphics[width=\linewidth]{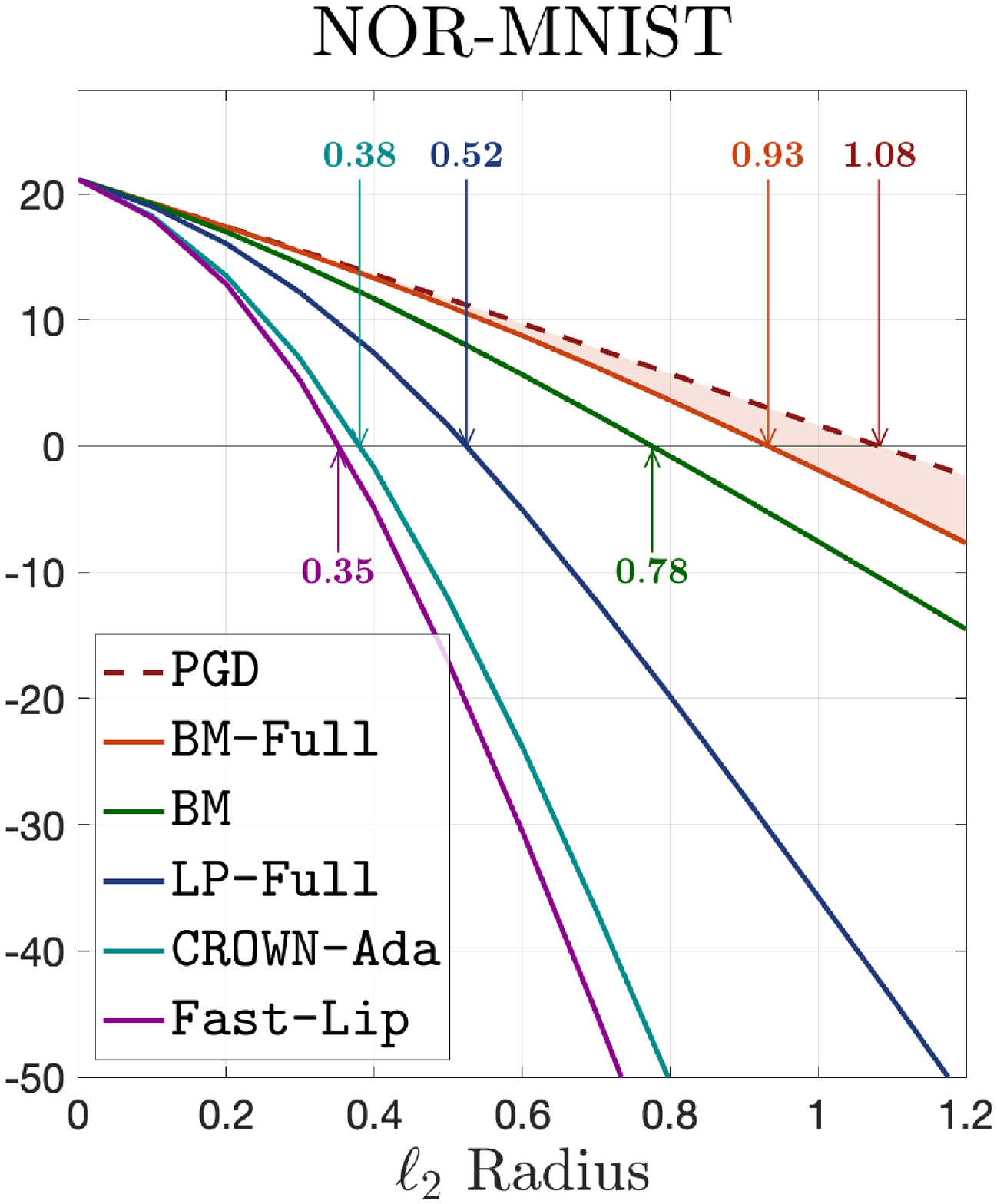}
    \end{subfigure}%
    \begin{subfigure}{0.333\textwidth}
      \centering
      \includegraphics[width=\linewidth]{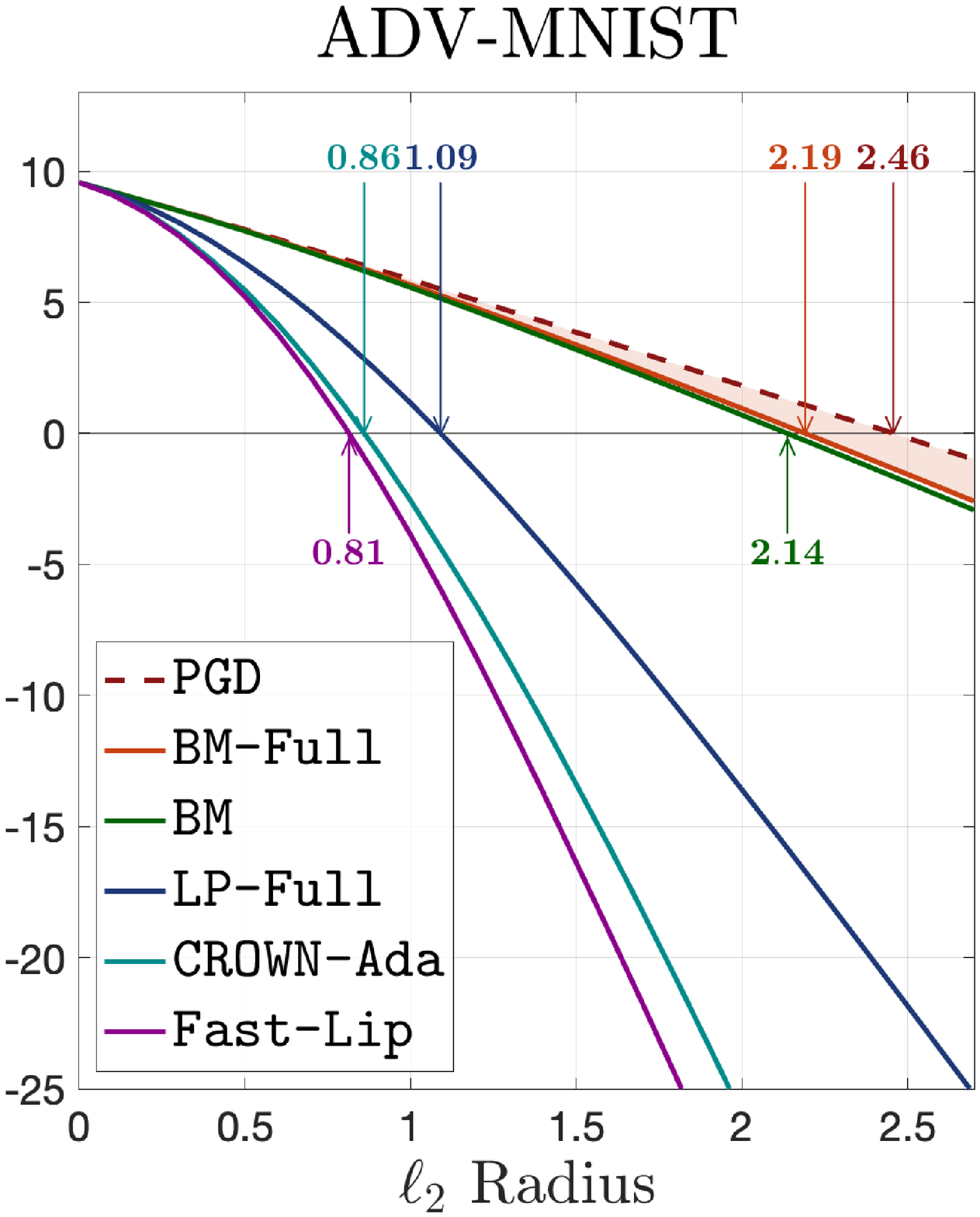}
    \end{subfigure}%
    \begin{subfigure}{0.333\textwidth}
      \centering
      \includegraphics[width=\linewidth]{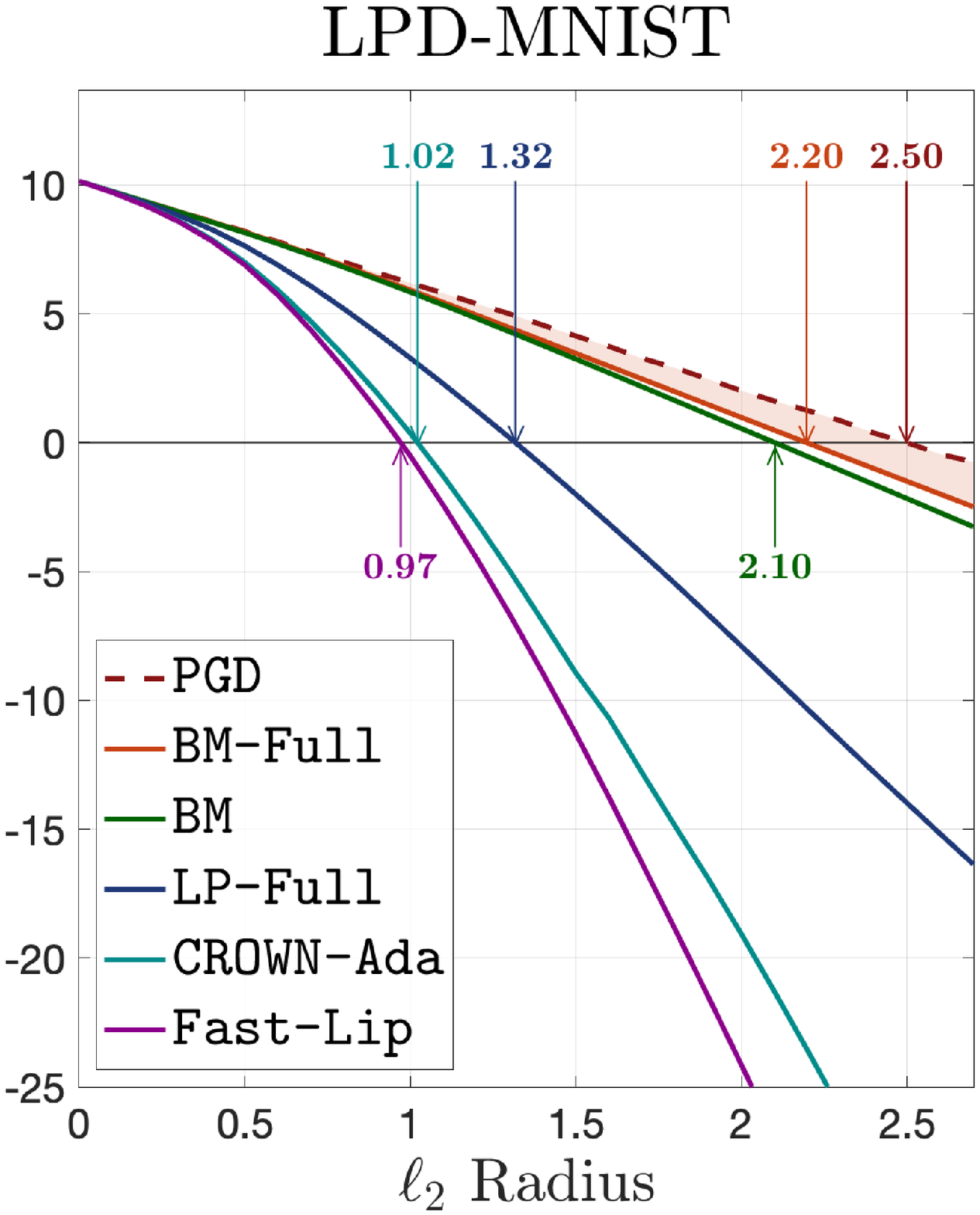}
    \end{subfigure}
    \captionsetup{font=footnotesize}
    \caption{\textbf{Lower bounds on the robustness margin}. We take \texttt{BM-Full}, \texttt{BM}, \texttt{LP-Full}, \texttt{CROWN-Ada} and \texttt{Fast-Lip} to compute their average lower bound on (\ref{eq:attack}), and then compare them to the average \texttt{PGD} upper bound on (\ref{eq:attack}) over a wide range of $\ell_2$ perturbation radius. \textbf{(Left.)} NOR-MNIST. \textbf{(Middle.)} ADV-MNIST. \textbf{(Right.)} LPD-MNIST.}
    \label{fig:tightness}
\end{figure*}
\paragraph{Setup.} 
We use an Apple laptop, running a silicon M1 pro chip with 10-core CPU, 16-core GPU, and 32GB of RAM. We implemented \texttt{BM} and \texttt{BM-Full} in MATLAB. The nonconvex problem (\ref{eq:bm}) is solved using the trust-region interior-point solver \texttt{knitro} \cite{byrd2006k} with a warm-start strategy.
\paragraph{Models.}
We perform simulation on three models: NOR-MNIST, ADV-MNIST and LPD-MNIST. All three models are trained by \citet{salman2019convex} and are taken directly from the authors' project page\footref{refnote}. In particular, the architecture and the numerical values of the weights for NOR-MNIST, ADV-MNIST and LPD-MNIST are identical to NOR-MLP-B, ADV-MLP-B and LPD-MLP-B respectively in \citet{salman2019convex}. All three models are fully-connected feedforward neural network models with 2 hidden layers of 100 neurons each, and were trained on the MNIST dataset with different training procedures. NOR-MNIST was trained normally using the cross-entropy loss and used as a control. ADV-MNIST was adversarially trained against the PGD attack using the method of \citet{madry2017towards}, with $\ell_\infty$ radius of 0.1. LPD-MNIST was robustly trained via the adversarial polytope perturbation of \citet{wong2018provable}, with the size of the adversarial polytope also set to 0.1.
 
\subsection{Robustness verification on neural network inputs.}
We certify the robustness of our three models against an $\ell_2$ adversary using our proposed method, and compare their performance against the existing state-of-the-art. 
In each trial, we fix an attack radius $\rho$, and mark a correctly-classified image $\hat x$ as \emph{robust} if the lower bound on (\ref{eq:attack}) is positive with respect to all incorrect classes, i.e. $0<\phi_{\lb}[c]\leq \phi[c]$ for all $c\neq \hat c$, and mark an image $\hat x$ as \emph{not robust} if an attack is found by \texttt{PGD}, i.e. $\phi[c]\leq \phi_{\ub}[c] < 0$ for some $c\neq \hat c$. We mark the image as \emph{status unknown} if a determination cannot be made either way. The lower bound for \texttt{BM} and \texttt{BM-Full} is obtained by \propref{dual}. 

\paragraph{Results and discussions.}
Table~\ref{table:certify_img} shows the number of images that are certified \emph{robust} within the first 1000 correctly classified images using \texttt{BM-Full}, \texttt{BM}, $\alpha,\beta$-\texttt{CROWN}, \texttt{LP-Full}, \texttt{CROWN-Ada} and  \texttt{Fast-Lip}. The average computation time per image for each verifier are also shown. In addition, to gauge the efficacy of our verifiers, we benchmark our results against the upper bound on the true number of \emph{robust} images (denoted as UB in the table), which is the number of images that does not get marked as \emph{not robust} by \texttt{PGD}. From the table, we see that our nonconvex verifiers are able to consistently outperform all the other verifiers under all cases, with time complexity only 5 to 10 times higher than \texttt{LP-Full}. Despite higher time complexity, our verifiers are the only verifiers that can still verify reasonable amount of images under larger perturbation radius. Notably, for small perturbation, our verifiers can nearly certify all images that cannot be attacked by \texttt{PGD}, leaving very few images as \emph{status unknown}.  

\subsection{Tightness of our lower bound} 
Since robust verification is an NP-hard problem, all relaxation methods must become loose for sufficiently large radius. Fortunately, robust verification is only needed when the \texttt{PGD} upper bound of (\ref{eq:attack}) is positive; therefore, we only need the relaxation to be tight when the \texttt{PGD} upper bound is still positive. In this experiment, we analyze the gap between our lower bound in Proposition~\ref{prop:dual} and the \texttt{PGD} upper bound over a wide range of perturbation radius. 
To accurately measure the gap between our lower bound and the \texttt{PGD} upper bound, we average each bound over all 9 incorrect classes of the first 10 correctly classified images in the test set, in total 90 samples are considered.

\paragraph{Results and discussions.}
Figures~\ref{fig:tightness} shows the average \texttt{PGD} upper bound, and the average lower bound on (\ref{eq:attack}) computed from \texttt{BM}, \texttt{BM-Full}, \texttt{LP-Full}, \texttt{CROWN-Ada} and \texttt{Fast-Lip}.
Our lower bounds are significantly tighter than all the other verifiers across a wide range of $\ell_2$ perturbation radius. Most importantly, our lower bounds are able to remain tight in regions where the \texttt{PGD} upper bound is still positive. We reiterate that it is not possible for our nonconvex verifiers to be exact with a large perturbation radius, because exact verification is NP-hard and our algorithm is polynomial-time. Nonetheless, so long as we can remain tight as the upper-bound crosses the zero line, our certification methods will be very close to exact.



\section{Conclusion}
In this work, we presented a neural network certification technique based on a nonconvex low-rank restricted SDP relaxation. Our experiments find that the method is able to overcome the convex relaxation barrier \citep{salman2019convex} with runtime only a small constant factor (5-10$\times$) worse than the existing state-of-the-art. Our results 
showed that even a basic nonconvex relaxation, \texttt{BM}, offers a significant reduction in relaxation gap, while augmenting with bound propagation, \texttt{BM-Full}, allows us to almost fully close the gap towards exact certification.

\section*{Acknowledgments}
\textcolor{mycolor}{The authors thank Zico Kolter for insightful discussions and pointers to the literature. Financial support for this work was provided by the NSF CAREER Award ECCS-2047462 and C3.ai Inc. and the Microsoft Corporation via the C3.ai Digital Transformation Institute.}
\newpage
\bibliography{arxiv.bib}
\bibliographystyle{icml2023}

\newpage
\appendix
\onecolumn
\section{Implementation details for \texttt{BM} and \texttt{BM-Full}}\label{app:implementation}
In this section, we present the implementation detail for our two proposed methods: \texttt{BM}, the nonconvex relaxation (\ref{eq:bm}) proposed in the main paper; and \texttt{BM-Full}, an extension of \texttt{BM} obtained by adding preactivation bounds on each hidden neuron in (\ref{eq:bm}). We focus our attention on how to efficiently implement both methods to verify $\ell_2$ and $\ell_\infty$ adversaries for neural networks trained on MNIST dataset. 

This section consists of three parts. First, we describe the valid input set constraint that we need to add into (\ref{eq:bm}) in order to certify MNIST images, and a few constraints in (\ref{eq:bm}) that can be simplified for improving efficiency. Second, in Appendix~\ref{sec:bm_l2} to \ref{sec:bm_full_linf}, we summarized the practical and efficient formulation for \texttt{BM} and \texttt{BM-Full} with respect to both $\ell_2$ and $\ell_\infty$ perturbation. Third, in Appendix~\ref{sec:algorithm}, we present more details on how to efficiently solve \texttt{BM} and \texttt{BM-Full} using the procedure described in Algorithm~\ref{alg:staircase}.

\paragraph{Valid input set constraints}
For model train on MNIST dataset, we add an extra constraint $0 \leq u_0\cdot u_1 \leq 1$ into (\ref{eq:bm}) because MNIST images are normalized to between 0 and 1 during training and testing. Notice that adding an extra inequality constraint $0 \leq u_0\cdot u_1 \leq 1$ does not alter any theoretical results in this paper as it only add an extra term to $s_1$ in the slack matrix $S(y,z)$.

\paragraph{Simplify constraints in (\normalfont\ref{eq:bm})}
In our practical implementation, we fix $u_0$ to $1$ in (\ref{eq:bm}). The reason is twofold. First, by fixing $u_0=1$, most constraints in (\ref{eq:bm}) become linear, and hence reduces the time complexity of our algorithm significantly. Second, the dual variable $z_0$, which is associated with the constraint $u_0=1$, can be solved via the KKT condition $S(x,y)U=0$.

\subsection{Efficient formulation of \texttt{BM} for $\ell_2$ norm}\label{sec:bm_l2}
We now turn to the practical aspect of implementing our proposed method \texttt{BM}. In particular, to verify $\ell_2$ adversaries of neural networks trained on MNIST dataset, we solve (\ref{eq:bm}) in the following form
\begin{alignat}{2}
\phi_{r}[c]=\min_{u,V}\quad &  w_{\ell}^{T} u_{\ell}\tag{BM-$\ell_2$} \label{eq:bm_l2}\\
\text{s.t. }\quad & \|u_{1}-\hat{x}\|^{2}+\|V_{1}\|^{2}\le \rho^{2}, & \qquad(y_{0})\nonumber \\
& u_{1}\ge0,\qquad u_{1}\le1 & \qquad(y_{0,1},y_{0,2})\nonumber \\
 & u_{k+1}\ge0,\qquad u_{k+1}-W_{k}u_{k}-b_{k}\ge 0, & (y_{k,1},y_{k,2})\nonumber \\
 & \diag\left[(u_{k+1}-W_{k}u_{k}-b_{k})u_{k+1}^{T}+(V_{k+1}-W_{k}V_{k})V_{k+1}^{T}\right]=0, & (z_{k})\nonumber\\
 & 1+\sum_{k=1}^{\ell-1}(\|u_{k}\|^{2}+\|V_{k}\|^{2})\le R^{2}, & (\mu)\nonumber
\end{alignat}
for $k\in\{1,\ldots,\ell-1\}$. Notice that we have substituted $u_0=1$, added the valid input set constraints $u_{1}\ge 0$ and $u_{1}\le1$, and assigned their associated dual variable $y_{0,1}$ and $y_{0,2}$. In Definition~\ref{def:bm_l2}, we summarized how to evaluate the slack matrices $S(y,z)$ and the dual variable $z_0$ in (\ref{eq:sdd}) using the primal and dual solution of (\ref{eq:bm_l2}) in order to calculate the bound in Proposition~\ref{prop:dual}.

\begin{definition} \label{def:bm_l2}
Let $y=(y_{0},\{y_{k,1},y_{k,2}\}_{k=0}^{\ell-1})$ , $z=(\{z_{k}\}_{k=1}^{\ell-1})$, $u = (u_1,\ldots,u_\ell)$ and $V=(V_1,\ldots,V_\ell)$ be any \emph{certifiably first-order optimal} point of (\ref{eq:bm_l2}). Each component in the slack matrices $S(y,z)$ and the dual variable $z_0$ in (\ref{eq:sdd}) can be evaluated as
\begin{gather*}
s_0=-\sum_{k=1}^\ell s_k^Tu_k, \quad z_{0}=y_{0}(\|\hat{x}\|^{2}-\rho^{2})-1^Ty_{0,2}+\sum_{k=1}^{\ell-1}b_{k}^{T}y_{k,2}-\frac{1}{2}s_{0},\\
s_{1}=W_{1}^{T}y_{1,2}-2\hat{x}y_{0}-y_{0,1}+y_{0,2},\quad s_{\ell}=w_{\ell}-\left[Z_{\ell-1}b_{\ell-1}+y_{\ell-1,1}+y_{\ell-1,2}\right],\\
s_{k+1}=W_{k+1}^{T}y_{k+1,2}-\left(Z_{k}b_{k}+y_{k,1}+y_{k,2}\right)\qquad\text{for }k\in\{1,\dots,\ell-2\},\\
S_{1,1}=2y_{0}I,\quad S_{k,k+1}=-W_{k}^{T}Z_{k},\quad S_{k+1,k+1}=2Z_{k}\qquad\text{for }k\in\{1,\dots,\ell-1\},
\end{gather*}
where $Z_{k}=\diag(z_{k})$ for all $k$.
\end{definition}

\subsection{Efficient formulation of \texttt{BM-Full} for $\ell_2$ norm}\label{sec:bm_full_l2}
We now describe the practical formulation for \texttt{BM-Full}. Let $lb_k$ and $ub_k$ denote the lower bound and the upper bound on preactivation neurons in the $k$-th layer. $lb_k$ and $ub_k$ gives us the postactivation bound constraints $\max\{lb_k,0\}\le x_k \le \max\{ub_k,0\}$ for each postactivation neuron $x_k$ in (\ref{eq:sdp}) . To incorporate these bound constraints into (\ref{eq:bm}), we first rewrite each of them into an elementwise $\ell_2$ constraint for which $\e_i^Tx_k$ is restricted in a $\ell_2$ norm ball centered at $\frac{1}{2}\e_i^T(\max\{ub_k,0\}+\max\{lb_k,0\})$ with radius $\frac{1}{2}\e_i^T(\max\{ub_k,0\}-\max\{lb_k,0\})$ as in
\[
\max\{lb_k,0\}\le x_k \le \max\{ub_k,0\} \quad\iff\quad \|\e_i^Tx_k-\e_i^T\hat x_k\|^2 \le \rho_k^2 \quad\forall\ i\in\{1,\ldots,n_k\}
\]
where $\hat x_k= \frac{1}{2}(\max\{ub_k,0\}+\max\{lb_k,0\})$ and $\rho_k=\frac{1}{2}(\max\{ub_k,0\}-\max\{lb_k,0\})$. The above elementwise $\ell_2$ constraint has the following Burer-Monteiro formulation
\[
\diag\left[\left(u_{k+1}-\hat x_{k+1}\right)\left(u_{k+1}-\hat x_{k+1}\right)^T+V_{k+1}V_{k+1}^T\right]\le \rho_{k+1}^2
\]
for $k\in\{1,\ldots,\ell-1\}$.
In turn, to verify neural networks trained on MNIST dataset with respect to $\ell_2$ perturbation, we add the above bound constraints into (\ref{eq:bm}) and solve the following
\begin{alignat}{2}
\phi_{r}[c]=\min_{u,V}\quad &  w_{\ell}^{T} u_{\ell}\tag{BM-Full-$\ell_2$}\label{eq:bm_full_l2}\\
\text{s.t. }\quad & \|u_{1}-\hat{x}\|^{2}+\|V_{1}\|^{2}\le \rho^{2}, & \qquad(y_{0})\nonumber \\
 & u_{1}\ge0,\qquad u_{1}\le1 & \qquad(y_{0,1},y_{0,2})\nonumber \\
 & \diag\left[\left(u_{k+1}-\hat x_{k+1}\right)\left(u_{k+1}-\hat x_{k+1}\right)^T+V_{k+1}V_{k+1}^T\right]\le \rho_{k+1}^2 & \qquad(y_{k})\nonumber \\
 & u_{k+1}-W_{k}u_{k}-b_{k}\ge 0, & (y_{k,2})\nonumber \\
 & \diag\left[(u_{k+1}-W_{k}u_{k}-b_{k})u_{k+1}^{T}+(V_{k+1}-W_{k}V_{k})V_{k+1}^{T}\right]=0, & (z_{k})\nonumber\\
 & 1+\sum_{k=1}^{\ell-1}(\|u_{k}\|^{2}+\|V_{k}\|^{2})\le R^{2}, & (\mu)\nonumber
\end{alignat}
for $k\in\{1,\ldots,\ell-1\}$. Notice that we delete the constraints $u_k \ge 0$ because they overlap with the bound constraints $\max\{lb_k,0\}\le u_k$. In Definition~\ref{def:bm_full_l2}, we summarized how to evaluate the slack matrices $S(y,z)$ and the dual variable $z_0$ in (\ref{eq:sdd}) using the primal and dual solution of (\ref{eq:bm_full_l2}) in order to calculate the bound in Proposition~\ref{prop:dual}.

\begin{definition} \label{def:bm_full_l2}
Let $y=(y_{0},y_{0,1},y_{0,2},\{y_{k},y_{k,2}\}_{k=1}^{\ell-1})$ , $z=(\{z_{k}\}_{k=1}^{\ell-1})$, $u = (u_1,\ldots,u_\ell)$ and $V=(V_1,\ldots,V_\ell)$ be any \emph{certifiably first-order optimal} point of (\ref{eq:bm_full_l2}). Each component in the slack matrices $S(y,z)$ and the dual variable $z_0$ in (\ref{eq:sdd}) can be evaluated as
\begin{gather*}
s_0=-\sum_{k=1}^\ell s_k^Tu_k, \quad z_{0}=y_{0}(\|\hat{x}\|^{2}-\rho^{2})+\sum_{k=1}^{\ell-1}y_{k}^T(\hat x_{(k+1)}^2-\rho_{(k+1)}^2)-1^Ty_{0,2}+\sum_{k=1}^{\ell-1}b_{k}^{T}y_{k,2}-\frac{1}{2}s_{0},\\
s_{1}=W_{1}^{T}y_{1,2}-2\hat{x}y_{0}-y_{0,1}+y_{0,2},\quad s_{\ell}=w_{\ell}-\left[Z_{\ell-1}b_{\ell-1}+2\hat X_{\ell}y_{\ell-1}+y_{\ell-1,2}\right],\\
s_{k+1}=W_{k+1}^{T}y_{k+1,2}-\left(Z_{k}b_{k}+2\hat X_{k+1}y_{k}+y_{k,2}\right)\qquad\text{for }k\in\{1,\dots,\ell-2\},\\
S_{1,1}=2y_{0}I,\quad S_{k,k+1}=-W_{k}^{T}Z_{k},\quad S_{k+1,k+1}=2(Z_{k}+\hat X_{k})\qquad\text{for }k\in\{1,\dots,\ell-1\},
\end{gather*}
where $Z_{k}=\diag(z_{k})$ and $\hat X_k=\diag(\hat x_k)$ for all $k$.
\end{definition}

\subsection{Efficient formulation of \texttt{BM} for $\ell_\infty$ norm}\label{sec:bm_linf}
We now describe how to implement \texttt{BM} for verifying $\ell_\infty$ adversaries. In the case of MNIST image, the $\ell_\infty$ norm ball constraint on the input $x_1$, i.e. $\|x_1-\hat x\|_\infty\leq \rho$, can be combined with the valid input set constraints $0\leq x_1 \leq 1$. Specifically, combining the two constraints yields: $\max\{0,\hat x-\rho\}\leq x_1\leq\min\{1,\hat x+\rho\}$. Similar to the postactivation bound constraints in \texttt{BM-Full}, this constraint can be written as an elementwise $\ell_2$ norm constraint as in
\[
	\max\{0,\hat x-\rho\}\leq x_1\leq\min\{1,\hat x+\rho\} \quad\iff\quad \|\e_i^Tx_1-\e_i^T\hat x_1\|^2 \leq \rho_1^2 \quad\forall\ i\in\{1,\ldots,n_1\}
\]
where $\hat x_1=\frac{1}{2}(\min\{1,\hat x+\rho\}+\max\{0,\hat x-\rho\})$ and $\rho_1=\frac{1}{2}(\min\{1,\hat x+\rho\}-\max\{0,\hat x-\rho\})$. The above constraint yields the following Burer-Monteiro formulation
\[
\diag\left[\left(u_{1}-\hat x_1\right)\left(u_{1}-\hat x_1\right)^T+V_{1}V_{1}^T\right]\le \rho_1^2
\]

In turn, to verify neural networks train on MNIST with respect to $\ell_\infty$ perturbation, we solve (\ref{eq:bm}) in the following form
\begin{alignat}{2}
\phi_{r}[c]=\min_{u,V}\quad &  w_{\ell}^{T} u_{\ell}\tag{BM-$\ell_\infty$} \label{eq:bm_linf}\\
\text{s.t. }\quad & \diag\left[\left(u_{1}-\hat x_1\right)\left(u_{1}-\hat x_1\right)^T+V_{1}V_{1}^T\right]\le \rho_1^2, & \qquad(y_{0})\nonumber \\
 & u_{k+1}\ge0,\qquad u_{k+1}-W_{k}u_{k}-b_{k}\ge 0, & (y_{k,1},y_{k,2})\nonumber \\
 & \diag\left[(u_{k+1}-W_{k}u_{k}-b_{k})u_{k+1}^{T}+(V_{k+1}-W_{k}V_{k})V_{k+1}^{T}\right]=0, & (z_{k})\nonumber\\
 & 1+\sum_{k=1}^{\ell-1}(\|u_{k}\|^{2}+\|V_{k}\|^{2})\le R^{2}, & (\mu)\nonumber
\end{alignat}
for $k\in\{1,\ldots,\ell-1\}$. Notice that the $\ell_\infty$ norm constraint, i.e. $\|u_1-\hat x\|_\infty\leq\rho$, has been combined with the valid input set constraint, i.e. $0\leq u_1 \leq 1$. In Definition~\ref{def:bm_linf}, we summarized how to evaluate the slack matrices $S(y,z)$ and the dual variable $z_0$ in (\ref{eq:sdd}) using the primal and dual solution of (\ref{eq:bm_linf}) in order to calculate the bound in Proposition~\ref{prop:dual}.

\begin{definition} \label{def:bm_linf}
Let $y=(y_{0},\{y_{k,1},y_{k,2}\}_{k=1}^{\ell-1})$ , $z=(\{z_{k}\}_{k=1}^{\ell-1})$, $u = (u_1,\ldots,u_\ell)$ and $V=(V_1,\ldots,V_\ell)$ be any \emph{certifiably first-order optimal} point of (\ref{eq:bm_linf}). Each component in the slack matrices $S(y,z)$ and the dual variable $z_0$ in (\ref{eq:sdd}) can be evaluated as
\begin{gather*}
s_0=-\sum_{k=1}^\ell s_k^Tu_k, \quad z_{0}=y_{0}^T(\hat{x}_1^2 -\rho_1^{2})+\sum_{k=1}^{\ell-1}b_{k}^{T}y_{k,2}-\frac{1}{2}s_{0},\\
s_{1}=W_{1}^{T}y_{1,2}-2Y_{0}\hat{x},\quad s_{\ell}=w_{\ell}-\left[Z_{\ell-1}b_{\ell-1}+y_{\ell-1,1}+y_{\ell-1,2}\right],\\
s_{k+1}=W_{k+1}^{T}y_{k+1,2}-\left(Z_{k}b_{k}+y_{k,1}+y_{k,2}\right)\qquad\text{for }k\in\{1,\dots,\ell-2\},\\
S_{1,1}=2Y_{0},\quad S_{k,k+1}=-W_{k}^{T}Z_{k},\quad S_{k+1,k+1}=2Z_{k}\qquad\text{for }k\in\{1,\dots,\ell-1\},
\end{gather*}
where $Z_{k}=\diag(z_{k})$ for all $k$. $Y_{0}=\diag(y_{0})$.
\end{definition}

\subsection{Efficient formulation of \texttt{BM-Full} for $\ell_\infty$ norm}\label{sec:bm_full_linf}
Combine the results in \ref{sec:bm_full_l2} and \ref{sec:bm_linf}, to verify $\ell_\infty$ adversaries via \texttt{BM-Full}, we solve (\ref{eq:bm}) in the following form
\begin{alignat}{2}
\phi_{r}[c]=\min_{u,V}\quad &  w_{\ell}^{T} u_{\ell}\tag{BM-Full-$\ell_\infty$}\label{eq:bm_full_linf}\\
\text{s.t. }\quad & \diag\left[\left(u_{1}-\hat x_1\right)\left(u_{1}-\hat x_1\right)^T+V_{1}V_{1}^T\right]\le \rho^2_1, & \qquad(y_{0})\nonumber \\
 & \diag\left[\left(u_{k+1}-\hat x_{k+1}\right)\left(u_{k+1}-\hat x_{k+1}\right)^T+V_{k+1}V_{k+1}^T\right]\le \rho_{k+1}^2 & \qquad(y_{k})\nonumber \\
 & u_{k+1}-W_{k}u_{k}-b_{k}\ge 0, & (y_{k,2})\nonumber \\
 & \diag\left[(u_{k+1}-W_{k}u_{k}-b_{k})u_{k+1}^{T}+(V_{k+1}-W_{k}V_{k})V_{k+1}^{T}\right]=0, & (z_{k})\nonumber\\
 & 1+\sum_{k=1}^{\ell-1}(\|u_{k}\|^{2}+\|V_{k}\|^{2})\le R^{2}, & (\mu)\nonumber
\end{alignat}
for $k\in\{1,\ldots,\ell-1\}$, where $\hat x_{1}$ and $\rho_1$ are defined in Appendix~\ref{sec:bm_linf}. $\hat x_{k}$ and $\rho_k$ for $k\in\{2,\ldots,\ell-1\}$ are defined in Appendix~\ref{sec:bm_full_l2}. Notice that we also delete the constraints $u_k \ge 0$ because they overlap with the bound constraints $\max\{lb_k,0\}\le u_k$. In Definition~\ref{def:bm_full_linf}, we summarized how to evaluate the slack matrices $S(y,z)$ and the dual variable $z_0$ in (\ref{eq:sdd}) using the primal and dual solution of (\ref{eq:bm_full_linf}) in order to calculate the bound in Proposition~\ref{prop:dual}.

\begin{definition} \label{def:bm_full_linf}
Let $y=(y_{0},\{y_{k},y_{k,2}\}_{k=1}^{\ell-1})$ , $z=(\{z_{k}\}_{k=1}^{\ell-1})$, $u = (u_1,\ldots,u_\ell)$ and $V=(V_1,\ldots,V_\ell)$ be any \emph{certifiably first-order optimal} point of (\ref{eq:bm_full_linf}). Each component in the slack matrices $S(y,z)$ and the dual variable $z_0$ in (\ref{eq:sdd}) can be evaluated as
\begin{gather*}
s_0=-\sum_{k=1}^\ell s_k^Tu_k, \quad z_{0}=\sum_{k=0}^{\ell-1}y_{k}^T(\hat x_{(k+1)}^2-\rho_{(k+1)}^2)+\sum_{k=1}^{\ell-1}b_{k}^{T}y_{k,2}-\frac{1}{2}s_{0},\\
s_{1}=W_{1}^{T}y_{1,2}-2Y_{0}\hat{x},\quad s_{\ell}=w_{\ell}-\left[Z_{(\ell-1)}b_{(\ell-1)}+2\hat X_{\ell}y_{\ell-1}+y_{(\ell-1),2}\right],\\
s_{k+1}=W_{k+1}^{T}y_{(k+1),2}-\left(Z_{k}b_{k}+2\hat X_{k+1}y_{k}+y_{k,2}\right)\quad\text{for }k\in\{1,\dots,\ell-2\},\\
S_{1,1}=2Y_{0},\;S_{k,(k+1)}=-W_{k}^{T}Z_{k},\;S_{(k+1),(k+1)}=2(Z_{k}+\hat X_{k})\quad\text{for }k\in\{1,\dots,\ell-1\},
\end{gather*}
where $Z_{k}=\diag(z_{k})$ and $\hat X_k=\diag(\hat x_k)$ for all $k$. $Y_{0}=\diag(y_{0})$.
\end{definition}

\subsection{Efficient algorithm for solving \texttt{BM} and \texttt{BM-Full}}\label{sec:algorithm}
The efficient formulations described in Appendix~\ref{sec:bm_l2} to \ref{sec:bm_full_linf} can be efficiently solved using a similar procedure described in Algorithm~\ref{alg:staircase}. In this section, we focus our attention on the practical and efficient algorithm for solving (\ref{eq:bm_l2}). We start by describing the initialization scheme for the primal variables $u_k$ and $V_k$ in (\ref{eq:bm_l2}), and then we summarize the efficient procedure for solving (\ref{eq:bm_l2}) in Algorithm~\ref{alg:bm_l2}. Notice that Algorithm~\ref{alg:bm_l2} can be easily extended for (\ref{eq:bm_full_l2}), (\ref{eq:bm_linf}) and (\ref{eq:bm_full_linf}).

\paragraph{Initialize the primal variables.} We initialize each $u_k$ and $V_k$ in (\ref{eq:bm_l2}) as close to their optimal as possible, which can be done as follows. First, apply \texttt{PGD} to estimate $x_1,\ldots,x_{\ell}$ in the following semi-targeted attack problem (\ref{eq:attack_l2}), which is the original semi-targeted attack problem (\ref{eq:attack}) with the valid input set constraint
\begin{align}\tag{A-$\ell_2$}\label{eq:attack_l2}
 	\phi[c]=\min_{x_{1},\dots,x_{\ell}}\ \ w_{\ell}^{T}x_{\ell}
 	\quad\text{s.t.}\quad & x_{k+1}=\max\{0,W_{k}x_{k}+b_{k}\},\quad
 	 0 \leq x_{1} \leq 1,\quad
 	 \|x_{1}-\hat{x}\|\le\rho.
\end{align}
Second, initialize each $u_{k}$ to $x_{k}$; notice that since (\ref{eq:bm_l2}) is a nonconvex relaxation of (\ref{eq:attack_l2}), $x_{k}$ would usually be a good initialization for $u_{k}$. Finally, initialize each $V_{k}$ to a random matrix that has small elements. The reason for applying small initialization to each $V_{k}$ is to reduce the degree of constraint violation at the initial point.

\begin{algorithm}[h!]
\caption{\label{alg:bm_l2} Efficient algorithm for (\ref{eq:bm_l2})}

\textbf{Input:} Initial relaxation rank $r\ge2$. Weights $W_{1},\dots,W_{\ell}$
and biases $b_{1},\dots,b_{\ell}$. Original
input $\hat{x}$, true label $\hat{c}$, target label $c$, and perturbation
size $\rho$. Variable radius bound $R$.

\textbf{Output:} Lower-bound $\phi_{\lb}[c]\le\phi[c]$ on the optimal
value of the semi-targeted attack problem (\ref{eq:attack_l2}).

\textbf{Initialization:} Initialize primal variable $x_{+}=(\{x_{k}\}_{k=1}^{\ell},\{\vec(M_{k})\}_{k=1}^{\ell})$ where $x_1,\ldots, x_\ell$ are estimated by solving (\ref{eq:attack_l2}) via \texttt{PGD}, and each $M_{k}$ is a random matrix of small elements that has the same shape as $V_{k}$. Initialize dual variables $y=(y_{0},\{y_{k,1},y_{k,2}\}_{k=0}^{\ell-1})=0$ and $z=(\{z_{k}\}_{k=1}^{\ell-1})=0.$

\textbf{Algorithm:}
\begin{enumerate}
\item (Solve rank-$r$ relaxation) Warm-start the nonlinear solver with the initial point $(x_{+},y,z)$, and then use the solver to solve (\ref{eq:bm_l2}) over $x=(\{u_{k}\}_{k=1}^{\ell},\{\vec(V_{k})\}_{k=1}^{\ell})$. After the solver converges, retrieve the corresponding dual multipliers $y$ and $z$. We choose \texttt{knitro} \cite{byrd2006k} as the nonlinear solver in our experiment.
\item (Check certifiable first-order optimality) Let $f(x)$, $g(x)$, and $h(x)$ denote the objective, the inequality constraints associated with $y$, and the equality constraints associated with $z$ in (\ref{eq:bm_l2}), respectively. If $\|\nabla f(x)+\nabla g(x)y+\nabla h(x)z\|$
is sufficiently small, and if $g(x)\le0$, $h(x)=0$ and $1+\|x\|<R$ hold to sufficient tolerance, then continue. Otherwise, return error due to solver's inability to achieve certifiable first-order optimality.
\item (Check dual feasibility) Compute $S(y,z)$ and $z_0$ using the formula in Definition~\ref{def:bm_l2}. If $\epsilon_{\feas}=-\lambda_{\min}[S(y,z)]$
is sufficiently small, then return $\phi_{\lb}[c]=z_{0}-\epsilon_{\feas}\cdot R^2.$
Otherwise, continue.
\item (Escape lifted saddle point) Compute the eigenvector $\xi=(0,\xi_{1},\dots,\xi_{\ell})$
satisfying $\|\xi\|=1$ and $\xi^{T}S(y,z)\xi=-\epsilon_{\feas}$.
Set up new primal initial point $x_{+}=(\{u_{k}\}_{k=1}^{\ell},\{\vec(V_{+,k})\}_{k=1}^{\ell})$ where $V_{+,k}=[V_{k},0]+\epsilon\cdot[0,\xi_{k}]$. Increment $r\gets r+1$ and repeat Step 1 with $(x_{+},y,z)$ as the initial point.
\end{enumerate}
\end{algorithm}

\section{Additional experiment for $\ell_2$ norm}\label{app:exp_l2}
\paragraph{Model architectures}
In this experiment, we consider three deepter feedforward ReLU networks trained on MNIST dataset: MLP-$4\!\times\!100$, MLP-$6\!\times\!100$ and MLP-$9\!\times\!100$. All three networks are adversarially trained using \cite{madry2017towards} with $\ell_\infty$ radius equals to 0.1\footnote{We train all three models using the code available at \url{https://github.com/locuslab/convex_adversarial/blob/master/examples/mnist.py}}. MLP-$4\!\times\!100$ has 4 hidden layers of 100 neurons each. MLP-$6\!\times\!100$ has 6 hidden layers of 100 neurons each. MLP-$9\!\times\!100$ has 9 hidden layers of 100 neurons each.

\subsection{Tightness plots for deeper neural networks}\label{app:tightness}
It is known that the relaxation gap generally increases along with the number of layers in the neural network. To demonstrate the performance of our proposed methods in deeper networks, in this experiment, we measure the relaxation gap of \texttt{BM-Full} and \texttt{BM} with respect to models of three different depths.

\paragraph{Results and discussions.}
Figure~\ref{fig:depth} plots the average bounds against $\ell_2$ perturbation radius for three MNIST networks with 4, 6, 9 hidden layers of 100 neurons each. Notably, \texttt{BM} become loose for MLP-$6\!\times\!100$ and become looser than \texttt{LP-Full} for MLP-$9\!\times\!100$. This result is expected and is consistent with \citet{zhang2020tightness}; in particular, the SDP relaxation for ReLU gate, without any bound constrains on preactivations, does become loose for multiple layers. On the other hand, \texttt{BM-Full} remain significantly tighter than \texttt{LP-Full} for all cases. Furthermore, since the preactivation bounds used in \texttt{BM-Full} and \texttt{LP-Full} are the same in this experiment, Figure~\ref{fig:tightness} and Figure~\ref{fig:depth} suggest that with the same quality of preactivation bounds, \texttt{BM-Full} would yield a tighter relaxation than \texttt{LP-Full}. Based on this finding, we note that the preactivation bounds for \texttt{BM-Full} can also be computed via nonconvex relaxation methods, which should yield a tighter bound on the preactivations and hence further reduces the relaxation gap for \texttt{BM-Full}; one example would be recursively apply \texttt{BM-Full} to compute the upper and lower bound on each neuron, however, such method can be extremely computational expensive for small to medium size networks. We leave \texttt{BM-Full} with better preactivation bounds to our future work.

\begin{figure}[h!]
    \centering
    \begin{subfigure}{0.33\textwidth}
      \centering
      \includegraphics[width=\linewidth]{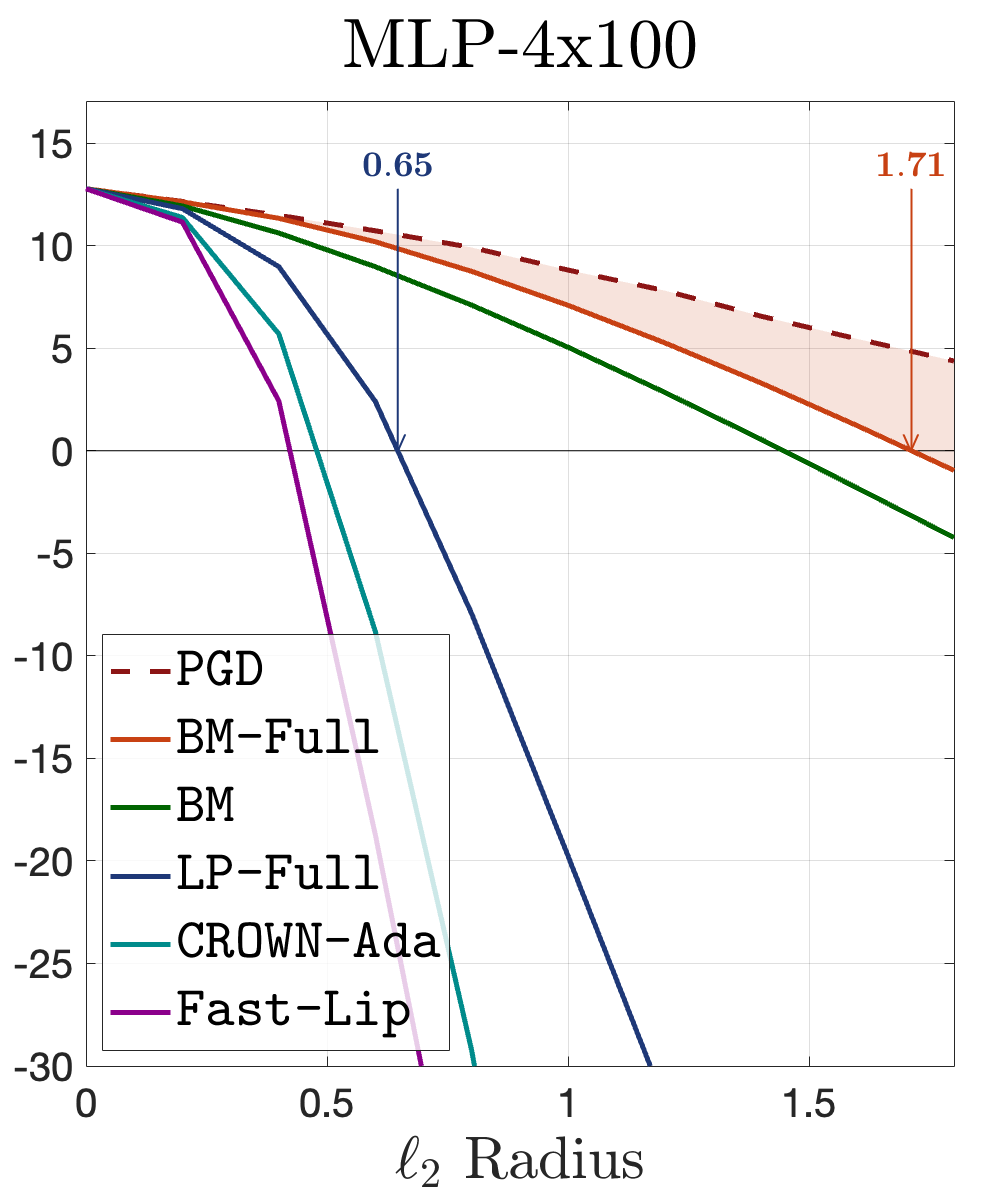}
    \end{subfigure}%
    \begin{subfigure}{0.33\textwidth}
      \centering
      \includegraphics[width=\linewidth]{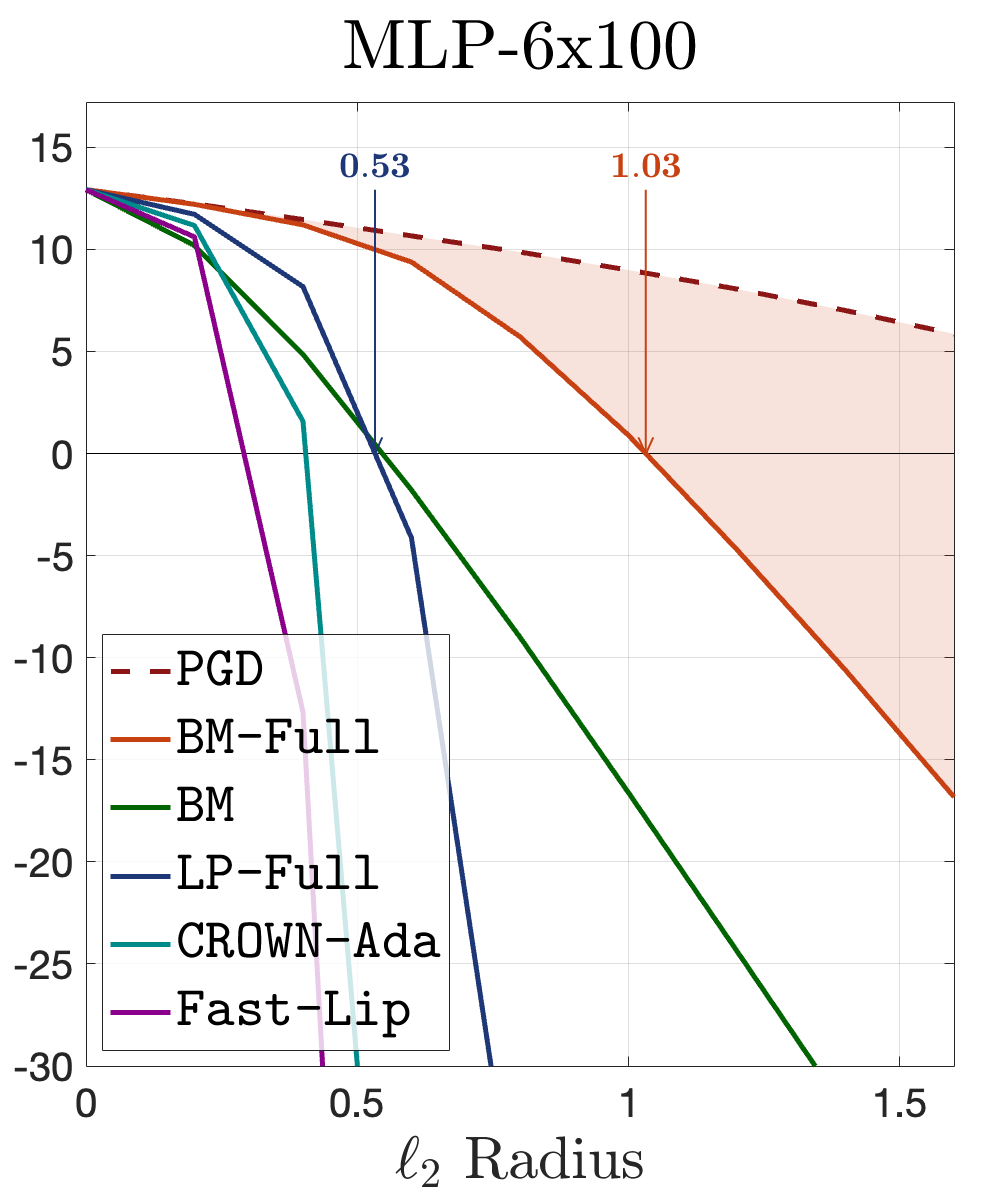}
    \end{subfigure}%
    \begin{subfigure}{0.33\textwidth}
      \centering
      \includegraphics[width=\linewidth]{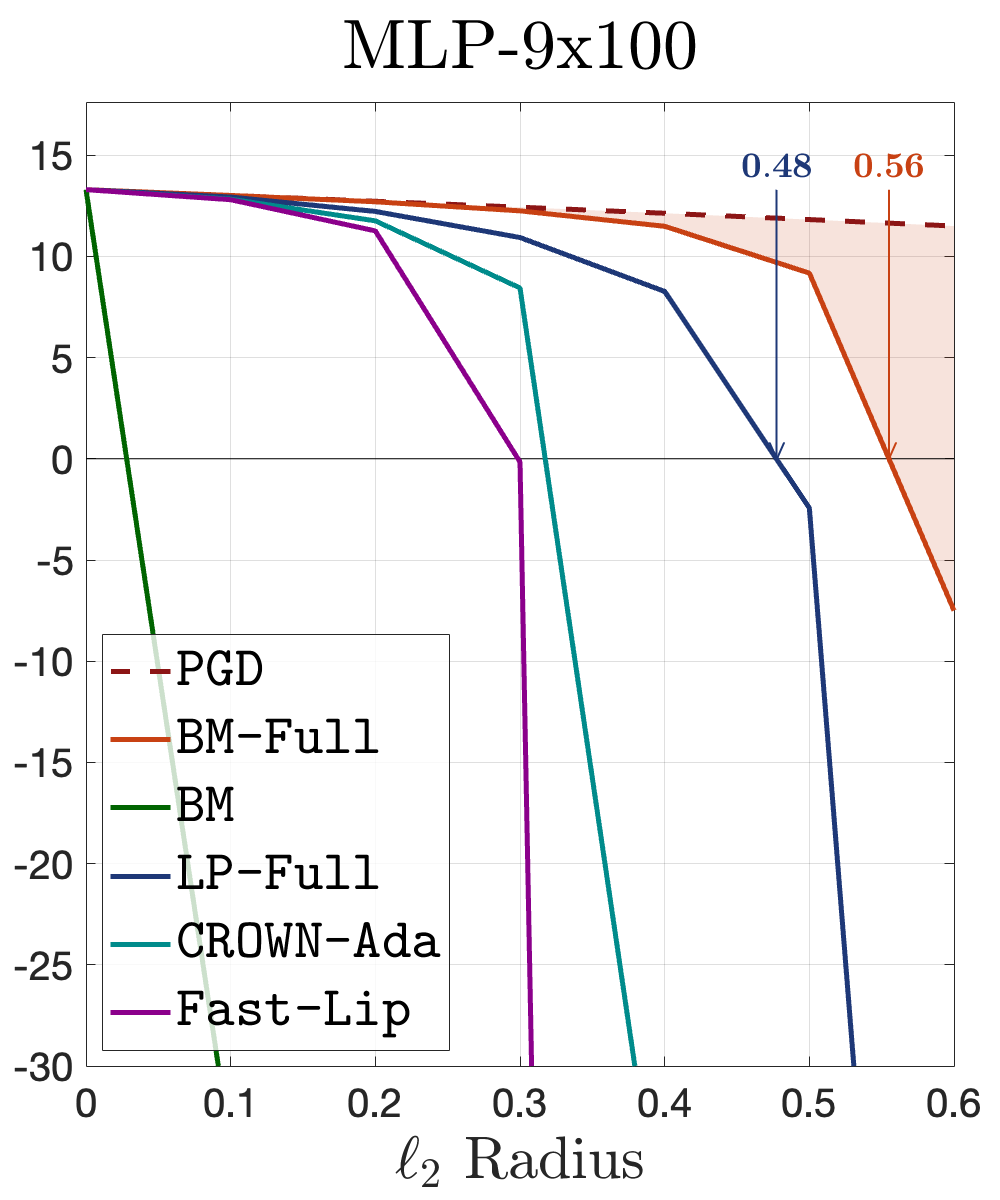}
    \end{subfigure}
    \caption{\textbf{Lower bound on (\ref{eq:attack}) with different network depth ($\ell_2$ norm).} We compute the average lower bound on (\ref{eq:attack}) for three models with 4, 6, 9 hidden layers of 100 neurons each, respectively. The upper bound on robustness margin is estimated via \texttt{PGD}. Observe that the gap between the \texttt{PGD} upper bound and lower bound from \texttt{BM-Full} are significantly smaller than that from \texttt{LP-Full}. We note that without bound propagations, \texttt{BM} does get loose when the number of hidden layers is more than 6. \textbf{(Left.)} MLP-$4\!\times\!100$. \textbf{(Middle.)} MLP-$6\!\times\!100$. \textbf{(Right.)} MLP-$9\!\times\!100$.}
    \label{fig:depth}
\end{figure}

\subsection{Visualizing adversarial attacks and robustness verification}
To illustrate why robustness verification is important in image classification, in this experiment, we perform a case study based on an image in the test set using the model ADV-MNIST. In particular, we focus on showing how would the $\ell_2$ adversaries look like in practice for four perturbation radius $\rho\in\{0.5, 1.0, 1.7,2.0\}$, as well as their corresponding lower bound on (\ref{eq:attack}) computed from our nonconvex and LP-based verifiers. We choose $\ell_2$ norm over $\ell_\infty$ norm for this experiment because $\ell_2$ norm allows perturbations to be concentrated on a small group of pixels, which produces adversaries that are more perceptually consistent to the original dataset when the perturbation radius is large.

\paragraph{Results and discussions.}
Figure~\ref{fig:case_study} shows: the adversarial attacks targeting the second most probable class of a MNIST image of class "1"; the probability of the top-2 classes for the attacked image (calculated using the softmax function); the \texttt{PGD} upper bound; and the lower bounds from different verifiers. For the MNIST image shown in the figure, the second most probable class is "7", and its adversarial attacks are computed via \texttt{PGD}.

The first and second column of Figure~\ref{fig:case_study} correspond to attacks within two small $\ell_2$ perturbation radius $\rho\in\{0.5, 1.0\}$. We see that every verifiers is able to verify robustness for $\rho=0.5$, however, only \texttt{BM-Full} and \texttt{BM} are able to verify robustness for $\rho=1.0$, all the other verifiers fail because their corresponding lower bounds become loose. In both cases, the adversaries look really similar to the original image.

The third and fourth column of Figure~\ref{fig:case_study} correspond to attack within two large $\ell_2$ perturbation radius $\rho\in\{1.7, 2.0\}$. Notice that \texttt{BM-Full} and \texttt{BM} can still verify robustness for $\rho=1.7$ even though the image is at the boundary of becoming not robust. In addition, the image is not robust for $\rho=2.0$ as the \texttt{PGD} upper bound is negative. Notably, both images start gaining features of the digit "7".

\begin{figure}[h!]
    \centering
    \includegraphics[width=\linewidth]{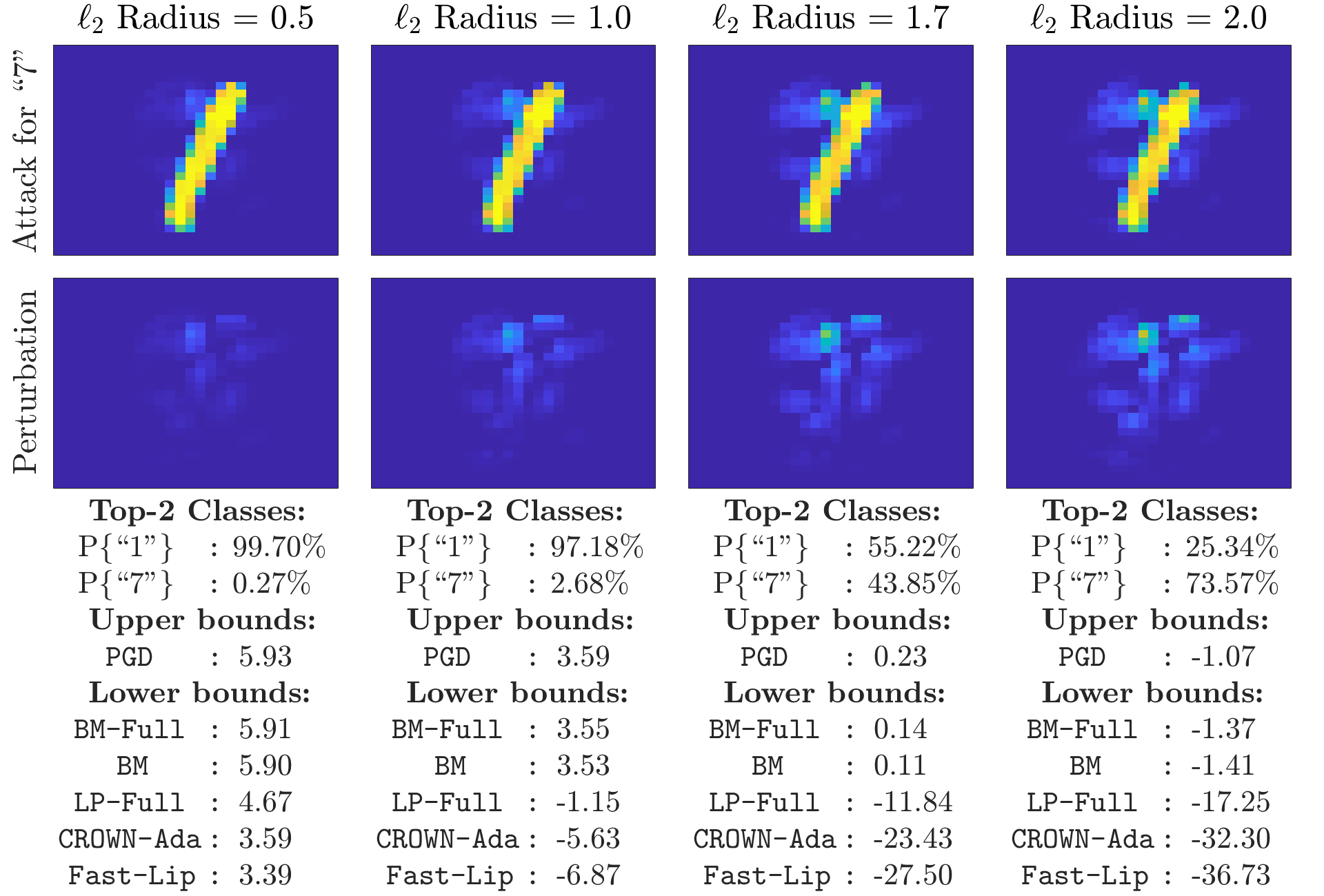}
    \caption{\textbf{Visualizing $\ell_2$ adversarial examples for ADV-MNIST.} We take an image in the test set to compare the capability of different verifiers for certifying robustness within four different $\ell_2$ radius $\rho\in\{0.5, 1.0, 1.7,2.0\}$. Each column in the figure shows (left to right): \textbf{(Column 1.)} The image is robust for $\rho=0.5$, and it can be certified by all verifiers. \textbf{(Column 2.)} The image is robust for $\rho=1.0$, but it can only be certified by our verifiers, \texttt{BM-Full} and \texttt{BM}. \textbf{(Column 3.)} The image is closed to become not robust for $\rho=1.7$, but it can still be certified by our verifiers. \textbf{(Column 4.)} The image is not robust for $\rho=2.0$.}
    \label{fig:case_study}
\end{figure}

\section{Additional experiment for $\ell_\infty$ norm}\label{app:exp_linf}
In this section, we apply \texttt{BM} and \texttt{BM-Full} to perform the same experiments in Section~\ref{sec:experiment} and Appendix~\ref{app:tightness} with respect to $\ell_\infty$ perturbation. Similar to the experimental results for $\ell_2$ perturbation, we set the preactivation bounds in \texttt{BM-Full} to be the same as those in \texttt{LP-Full}. The neural network models used in this section are defined in Section~\ref{sec:experiment} and Appendix~\ref{app:exp_l2}.

\subsection{Robust verification for $\ell_\infty$ adversaries}
We apply \texttt{BM} and \texttt{BM-Full} to verify robustness of the first 1000 correctly classified images. The simulation settings in this experiment are the same as those in Table~\ref{table:certify_img}.

\paragraph{Results and discussions.}
Table~\ref{table:certify_img_linf} shows the number of images certified as \emph{robust} within the first 1000 correctly classified images. We see that \texttt{BM-Full} is able to consistently outperform LP-based verifiers in all cases. \texttt{BM} outperforms \texttt{LP-Full} in ADV-MNIST and NOR-MNIST but achieves similar performance in LPD-MNIST. This is because LPD-MNIST is robustly trained by maximizing dual lower bound of LP relaxation over a convex outer polytope \citep{wong2018provable}; therefore, LP-based verifiers tend to work well for models that are robustly trained using this method. However, we note that even though models that are trained using \citet{wong2018provable} may be efficiently verified via LP-based verifiers, the robust training method of \citet{wong2018provable} is generally more conservative than the method of \citet{madry2017towards} as it is optimized over a convex outer polytope, which results in lower test accuracy in the final model. In this experiment, LPD-MNIST, the model that is trained using \citet{wong2018provable}, has test accuracy $95.91\%$, and ADV-MNIST, the model that is trained using \citet{madry2017towards} has accuracy $96.67\%$.
The performance of our verifiers, \texttt{BM-Full} and \texttt{BM}, is invariant to both robust training methods and achieve similar level of tightness with respect to all three model in Table~\ref{table:certify_img_linf}. We provide a thorough analysis on the tightness of \texttt{BM-Full} and \texttt{BM} in the next experiment.

\begin{table}[!h]
\caption{\textbf{Robustness verification for neural networks.} We compare the number of images certified as \emph{robust} by \texttt{BM-Full}, \texttt{BM}, $\alpha,\beta$\texttt{-CROWN}, \texttt{LP-Full}, \texttt{CROWN-Ada} and \texttt{Fast-Lip} within the first 1000 images for normally and robustly trained networks. The upper bound (denoted as UB) on the true number of robust images is obtained by \texttt{PGD}.}
\begin{center}
\resizebox{\linewidth}{!}{%
\begin{tabular}{
lcc c@{\hskip 5pt}c c@{\hskip 5pt}c c@{\hskip 5pt}c c@{\hskip 5pt}c c@{\hskip 5pt}c c@{\hskip 5pt}c
} 
\toprule  
  \multicolumn{1}{c}{Network} & $\ell_\infty$ Radius & \texttt{PGD}   & \multicolumn{2}{c}{\texttt{BM-Full}} & \multicolumn{2}{c}{\texttt{BM}} & \multicolumn{2}{c}{$\alpha,\beta$\texttt{-CROWN}} & \multicolumn{2}{c}{\texttt{LP-Full}} & \multicolumn{2}{c}{\texttt{CROWN-Ada}} & \multicolumn{2}{c}{\texttt{Fast-Lip}} \\
   \midrule
   &  & \footnotesize{UB} & \footnotesize{Robust} & \footnotesize{Time} & \footnotesize{Robust} & \footnotesize{Time}&\footnotesize{Robust} & \footnotesize{Time}&\footnotesize{Robust} & \footnotesize{Time}& \footnotesize{Robust}&\footnotesize{Time}& \footnotesize{Robust}&\footnotesize{Time}\\
\cmidrule(lr){3-3}\cmidrule(lr){4-5}\cmidrule(lr){6-7}\cmidrule(lr){8-9}\cmidrule(lr){10-11}\cmidrule(lr){12-13}\cmidrule(lr){14-15}
   ADV-MNIST & 0.10 & 831 & $\bf791$ &  87s & 760 & 124s & 791 &   2s & 314 & 16s &   8 & 10ms &   4 & 13ms\\
   ADV-MNIST & 0.13 & 731 & 632 & 102s & 574 & 144s & $\bf673$ &  11s &  46 & 17s &   1 &  9ms &   0 & 13ms\\
   ADV-MNIST & 0.15 & 626 & 484 & 127s & 366 & 125s & $\bf535$ &  22s &  11 & 15s &   0 &  9ms &   0 & 14ms\\
   \midrule
   LPD-MNIST & 0.10 & 868 & $\bf855$ & 125s & 828 & 163s & 818 & 0.2s & 829 & 13s & 589 & 12ms & 540 & 10ms\\
   LPD-MNIST & 0.13 & 791 & $\bf768$ & 104s & 713 & 126s & 743 & 0.2s & 689 & 13s & 154 & 10ms & 120 & 11ms\\
   LPD-MNIST & 0.15 & 727 & $\bf672$ & 120s & 597 & 132s & 672 & 0.3s & 545 & 12s &  43 &  9ms &  32 & 12ms\\
   \midrule
   NOR-MNIST & 0.02 & 910 & $\bf898$ &  99s & 859 & 116s & 881 & 1.4s & 686 &  3s & 130 &  9ms &  88 & 12ms\\
   NOR-MNIST & 0.03 & 775 & $\bf729$ & 143s & 617 & 107s & 713 &  16s & 278 &  4s &   8 &  8ms &   3 & 12ms\\
   NOR-MNIST & 0.05 & 401 & $\bf267$ & 229s & 127 & 154s & 238 &  43s &  10 &  5s &   0 & 12ms &   0 & 17ms\\
\bottomrule
\end{tabular}
}
\end{center}
\label{table:certify_img_linf}
\end{table}

\subsection{Tightness plots}
We apply \texttt{BM-Full} and \texttt{BM} to compute the average lower bound on (\ref{eq:attack}) with respect to $\ell_\infty$ perturbation radius. The simulation settings in this experiment are the same as those in Figure~\ref{fig:tightness}.

\paragraph{Results and discussions.}
As demonstrated in Figure~\ref{fig:tightness_linf}, both \texttt{BM-Full} and \texttt{BM} achieve tighter lower bound than the LP-based verifiers for all three models, across a wide range of $\ell_\infty$ perturbation radius. However, the gap between \texttt{BM} (green line) and \texttt{LP-Full} (blue line) is a lot smaller in LPD-MNIST compared to NOR-MNIST and ADV-MNIST, especially within the interval of $0.1$ to $0.15$. This explains why \texttt{BM} and \texttt{LP-Full} achieve similar performance for LPD-MNIST in Table~\ref{table:certify_img_linf}.

\begin{figure}[!h]
    \centering
    \begin{subfigure}{0.333\textwidth}
      \centering
      \includegraphics[width=\linewidth]{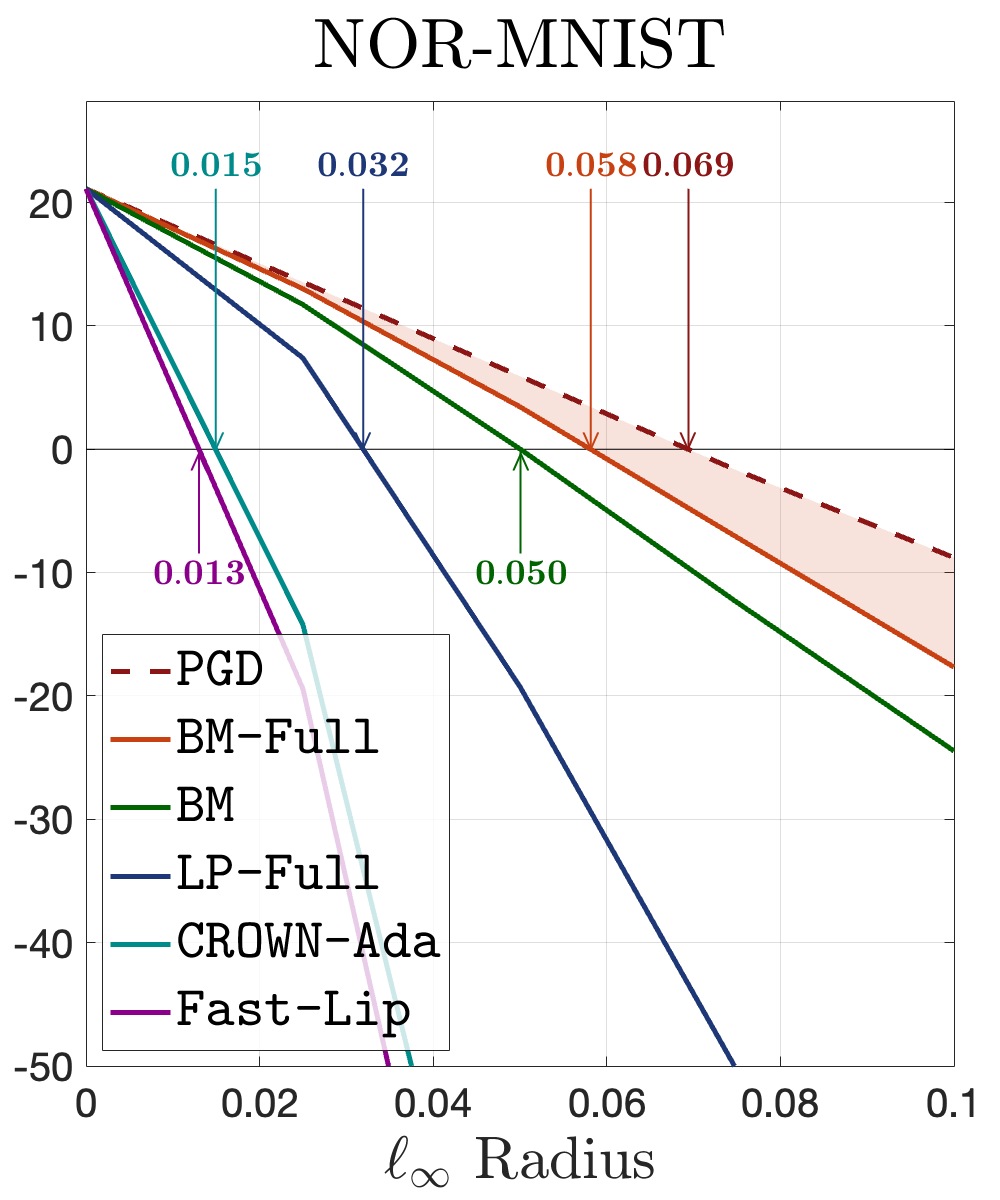}
    \end{subfigure}%
    \begin{subfigure}{0.333\textwidth}
      \centering
      \includegraphics[width=\linewidth]{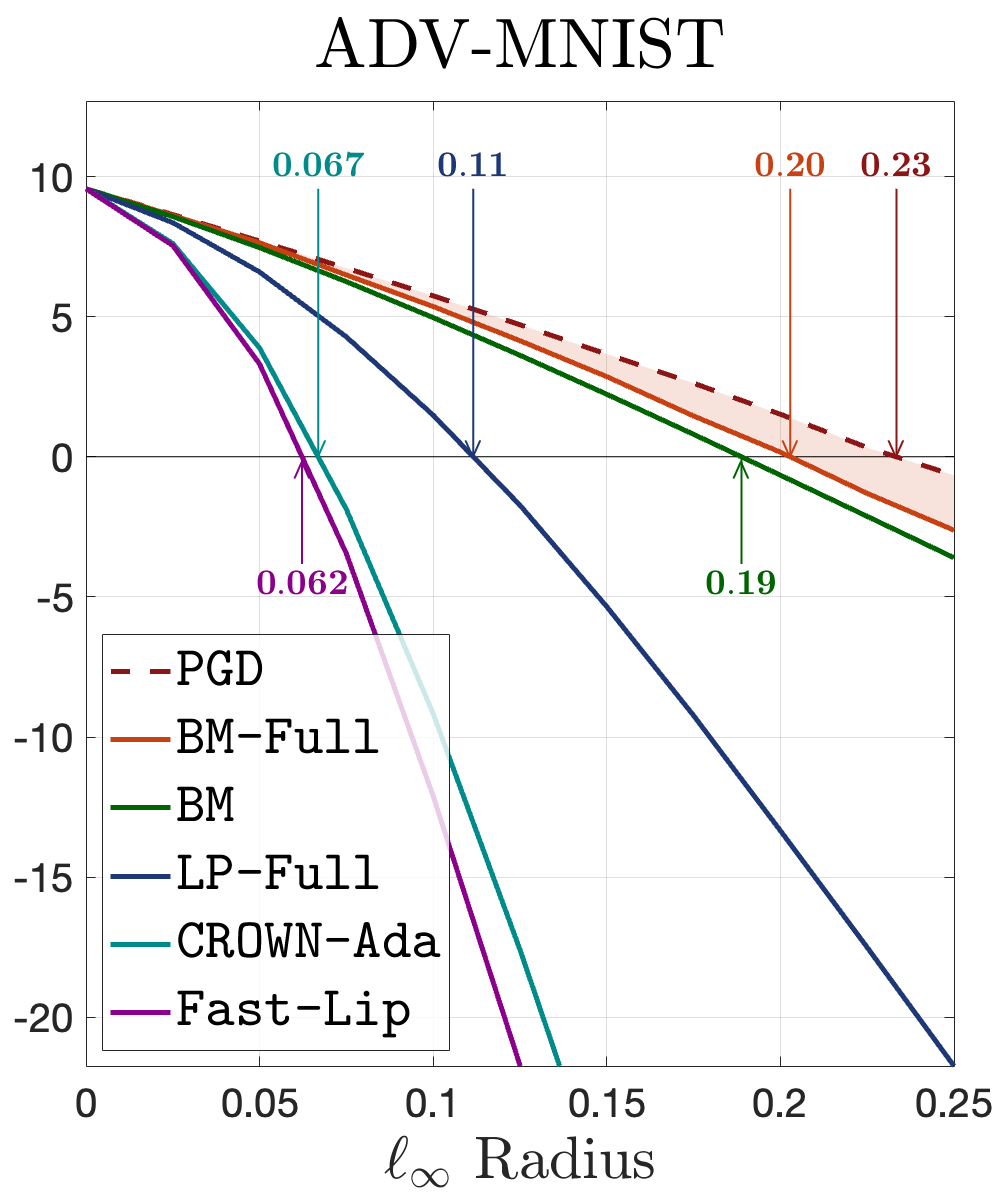}
    \end{subfigure}%
    \begin{subfigure}{0.333\textwidth}
      \centering
      \includegraphics[width=\linewidth]{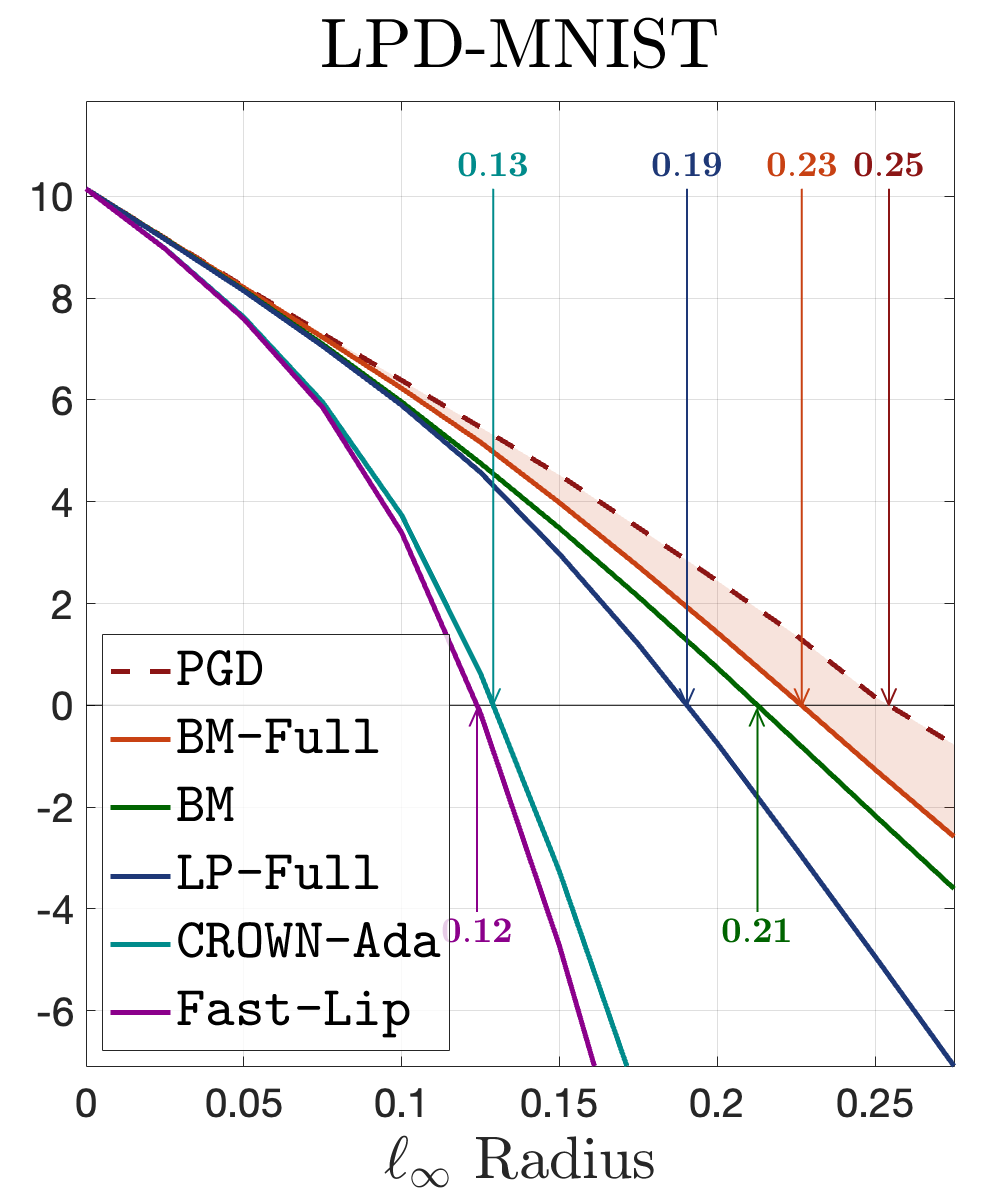}
    \end{subfigure}
    \caption{\textbf{Lower bounds on the robustness margin}. We take \texttt{BM-Full}, \texttt{BM}, \texttt{LP-Full}, \texttt{CROWN-Ada} and \texttt{Fast-Lip} to compute their average lower bound on (\ref{eq:attack}), and then compare them to the average \texttt{PGD} upper bound on (\ref{eq:attack}) over a wide range of $\ell_\infty$ perturbation radius. \textbf{(Left.)} NOR-MNIST. \textbf{(Middle.)} ADV-MNIST. \textbf{(Right.)} LPD-MNIST.}
    \label{fig:tightness_linf}
    \vspace{-0.5em}
\end{figure}

\subsection{Tightness plots for deeper neural networks}
We apply \texttt{BM-Full} and \texttt{BM} to compute the average lower bound on (\ref{eq:attack}) for three robustly trained MNIST models of different depths. The simulation settings in this experiment are the same as those in Figure~\ref{fig:depth}.

\paragraph{Results and discussions.}
We plot the average lower bound on (\ref{eq:attack}) for three MNIST models of different depths. Similar to the results in Figure~\ref{fig:depth}, we see that \texttt{BM} becomes loose when the network has more than 6 layers; this is expected as SDP relaxation, without any bound propagation, does become loose for multiple layers \citep{zhang2020tightness}. Though \texttt{BM} becomes loose in deeper networks, \texttt{BM-Full} remains significantly tighter than \texttt{LP-Full} in all cases. 
We again emphasize that the preactivation bounds used in \texttt{BM-Full} are the same as those in \texttt{LP-Full} in this experiment, and those bounds could be tighten by using the nonconvex relaxation techniques proposed in this paper. We defer it to our future work.

\begin{figure}[h!]
    \centering
    \begin{subfigure}{0.33\textwidth}
      \centering
      \includegraphics[width=\linewidth]{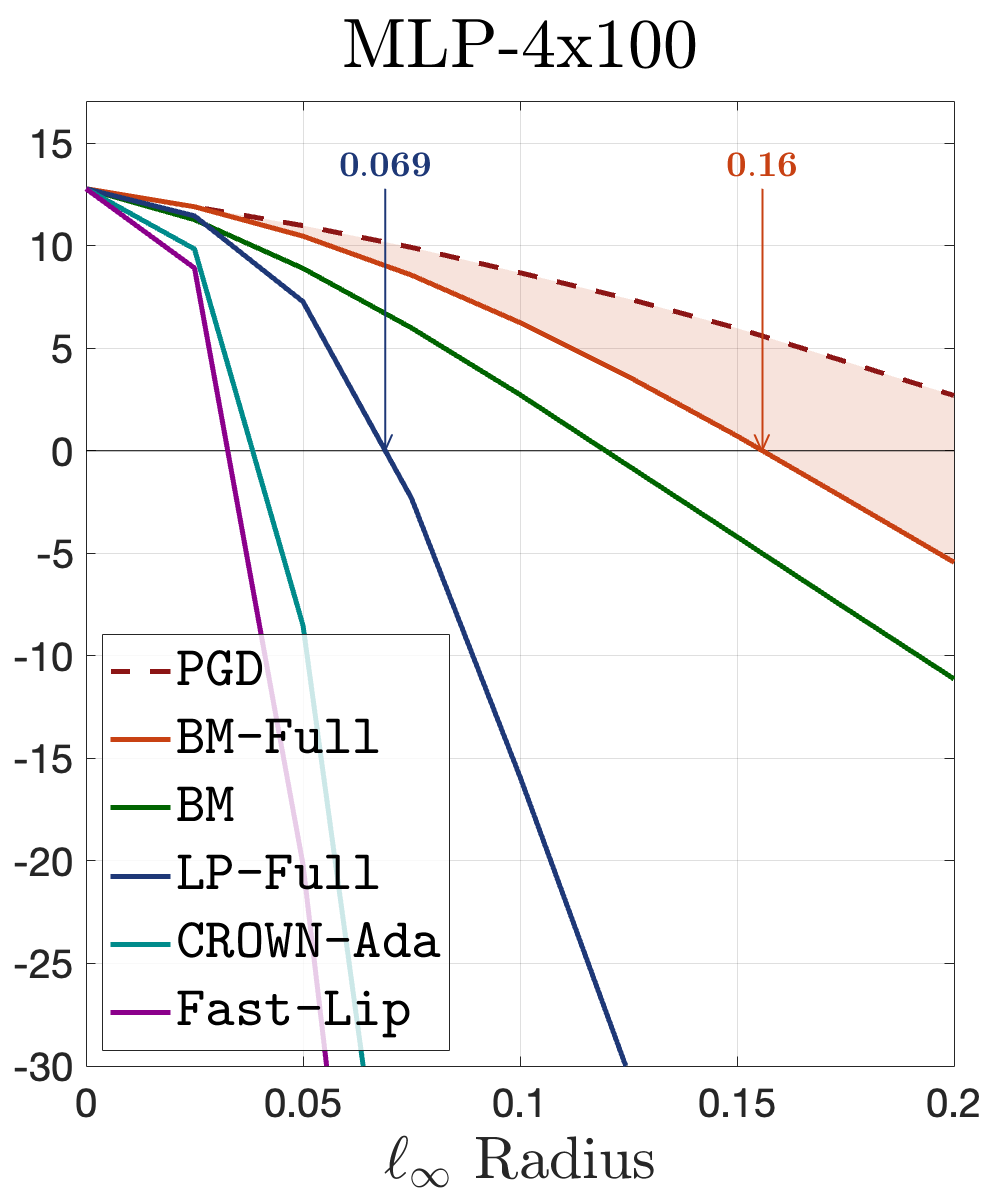}
    \end{subfigure}%
    \begin{subfigure}{0.33\textwidth}
      \centering
      \includegraphics[width=\linewidth]{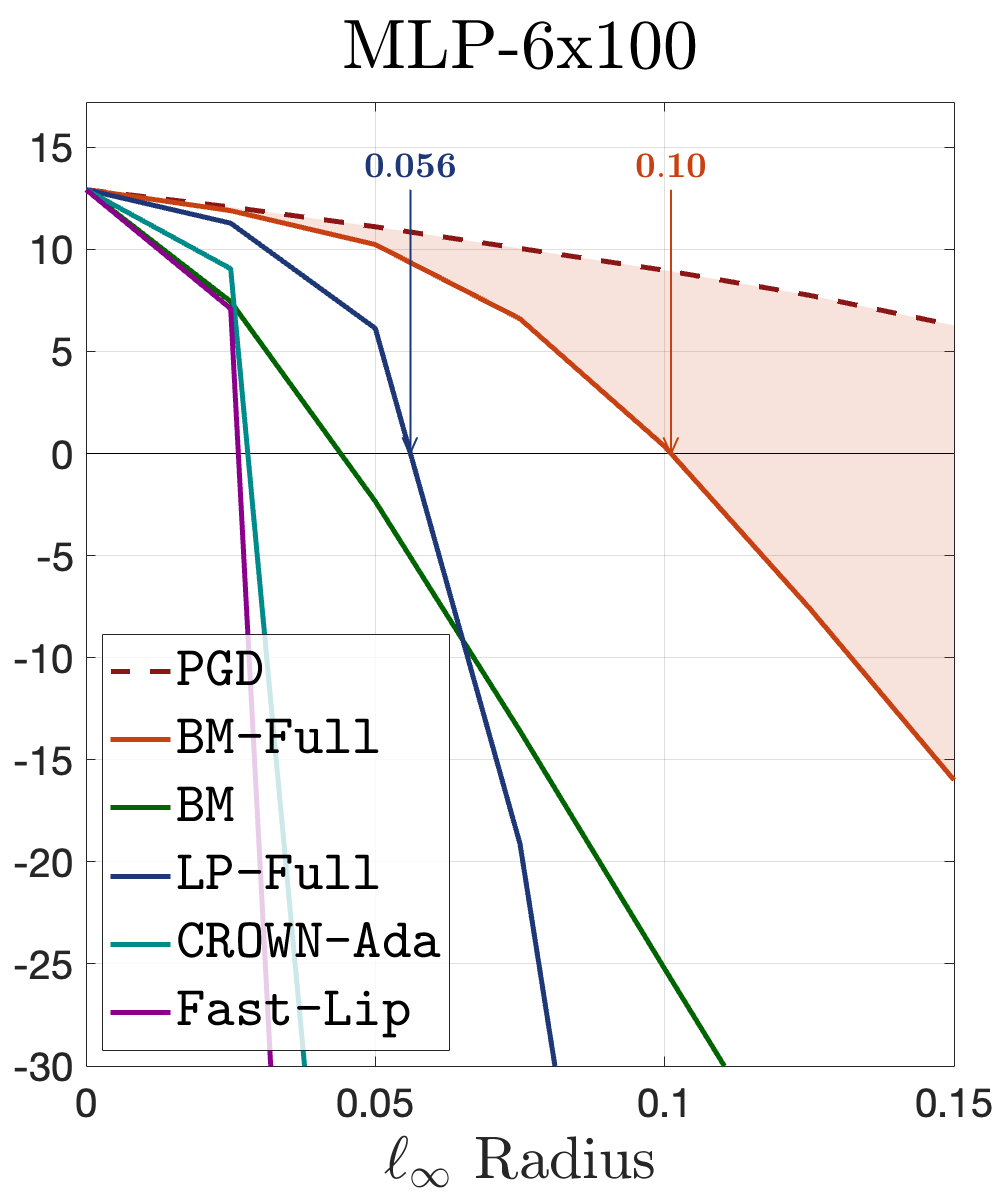}
    \end{subfigure}%
    \begin{subfigure}{0.33\textwidth}
      \centering
      \includegraphics[width=\linewidth]{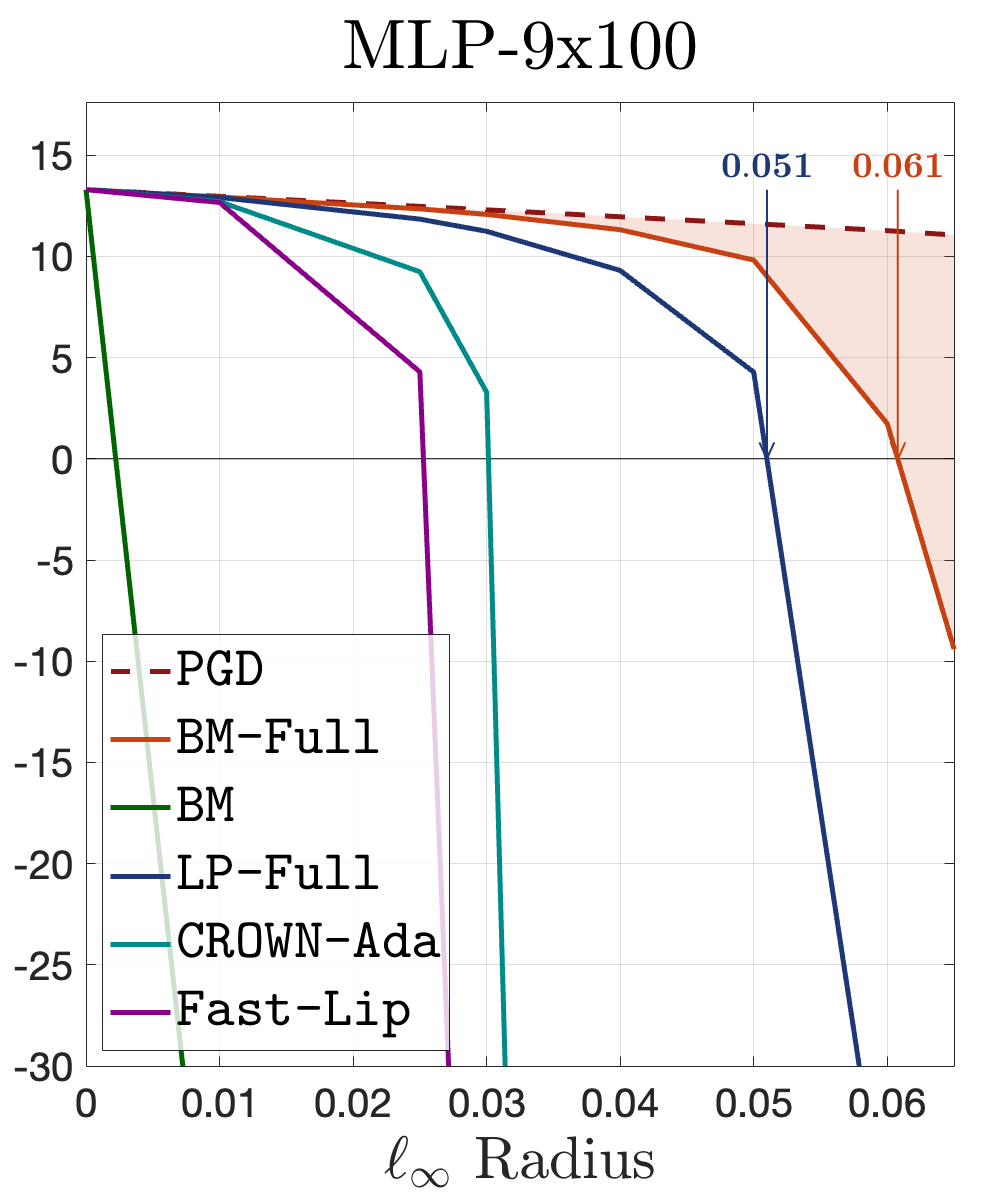}
    \end{subfigure}
    \caption{\textbf{Lower bound on (\ref{eq:attack}) with different network depth ($\ell_\infty$ norm).} We compute the average lower bound on (\ref{eq:attack}) for three models with 4, 6, 9 hidden layers of 100 neurons each, respectively. Observe that \texttt{BM-Full} is significantly tighter than \texttt{LP-Full} in all cases. We note that without bound propagations, \texttt{BM} does become loose when the number of hidden layers is more than 6. \textbf{(Left.)} MLP-$4\!\times\!100$. \textbf{(Middle.)} MLP-$6\!\times\!100$. \textbf{(Right.)} MLP-$9\!\times\!100$.}
    \label{fig:depth_linf}
    \vspace{-2em}
\end{figure}

\newpage
\section{Derivation of the dual problem (\ref{eq:sdd})}
Recall that primal problem (\ref{eq:sdp}) is written (with the corresponding
Lagrange multipliers given in parantheses) as the following
\begin{alignat*}{2}
\phi_{r}[c]=\min_{X\in\S^{n+1}}\quad & w_{\ell}^{T}x_{\ell} + w_{0}x_{0} & \tag{SDP-\ensuremath{r}}\\
\text{s.t. }\ \quad 
 & \tr(X_{1,1})-2x_{1}^{T}\hat{x}+x_{0}\|\hat{x}_{1}\|^{2}\le x_{0}\rho^{2},\quad x_{0}=1 & \qquad(y_{0},z_{0})\\
 & x_{k+1}\ge0,\qquad x_{k+1}\ge W_{k}x_{k}+b_{k}x_{0}, &  \qquad(y_{k,1},y_{k,2})\\
 & \diag(X_{k+1,k+1}-W_{k}X_{k,k+1}-b_{k}x_{k+1}^{T})=0 &  (z_{k})\\
 & \tr(X)\le R^{2}, &  (\mu)\\
 & X=\left[\begin{array}{c|ccc}
	x_{0} & x_{1}^{T} & \cdots & x_{\ell}^{T}\\
	\hline x_{1} & X_{1,1} & \cdots & X_{1,\ell}\\
	\vdots & \vdots & \ddots & \vdots\\
	x_{\ell} & X_{1,\ell}^{T} & \cdots & X_{\ell,\ell}
	\end{array}\right]\succeq0,\quad\rank(X)\le r. &  (S)
\end{alignat*}
for all $k\in\{1,\ldots,\ell-1\}$. This problem can be viewed as a generic instance of the following primal problem
\begin{align*}
 \min_{X\succeq0,\ \rank(X)\le r}\ \inner FX\quad\text{s.t.}\quad
\G_{0}(X)\le0,\quad\G_{k}(X)\le0,\quad
\H_{0}(X)+1=0,\quad\H_{k}(X)=0,\quad \tr(X)\le R^2
\end{align*}
for all $k\in\{1,\dots,\ell-1\}$. Here, we implicitly define $F$ to satisfy $\inner FX=w_{\ell}^{T}x_{\ell}+ w_{0}x_{0}$ and the linear constraint operators are respectively
\begin{gather*}
\G_{0}(X)=\tr(X_{1,1})-2x_1^{T}\hat{x}+(\|\hat{x}\|^{2}-\rho^{2})x_{0},\quad
\G_{k}(X)=\begin{bmatrix}-x_{k+1}\\W_{k}x_{k}+b_{k}x_{0}-x_{k+1}\end{bmatrix},\\
\H_{0}(X)=-x_{0},\quad
\H_{k}(X)=\diag(X_{k+1,k+1}-W_{k}X_{k,k+1}-b_{k}x_{k+1}^{T}).
\end{gather*}
Notice that $S$ is the dual variable associated with constraint $X\succeq 0$ that satisfies $S\succeq 0$ and $\rank(X)+\rank(S)\le n+1$ at optimality. Setting $y=(y_0,\{y_{k,1},y_{k,2}\}_{k=1}^{\ell-1})>0$, $z=(\{z_{k}\}_{k=1}^{\ell-1})$ and $\mu \le 0$, the Lagrangian of (\ref{eq:sdp}) reads
\begin{align*}
\L(X,y,z,\mu,S) & = \inner FX+\inner{z_{0}}{\H_{0}(X)+1}+\inner{y_{0}}{\G_{0}(X)}+\sum_{k=1}^{\ell-1}\left[\inner{\begin{bmatrix}y_{k,1}\\y_{k,2}\end{bmatrix}}{\G_{k}(X)}+\inner{z_{k}}{\H_{k}(X)}\right]\\
	& \qquad -\mu(\tr X-R^2)-\inner SX \\
	& = z_0 + R^2\mu + \inner{F+\G_{0}^{T}(y_{0})+\H_{0}^{T}(z_{0})+\sum_{k=1}^{\ell-1}\left[\G_{k}^{T}(y_{k,1},y_{k,2})+\H_{k}^{T}(z_{k})\right]-S-\mu I}{X}\\
	& = z_0 + R^2\mu + \inner{S(x,y)-S-\mu I}{X}
\end{align*}
where we use the superscript ``$T$'' to indicate the adjoint
operators. Setting $S=S(x,y)-\mu I\succeq 0$ yields the dual problem
\begin{align*}
\max_{y\ge0,\ z,\ \mu\le0}\ z_{0}+R^{2}\mu\quad \text{s.t.}\quad S(y,z)\succeq\mu I. \tag{SDD}
\end{align*}
Finally, we derive an explicit expression for the slack matrix
\[
S(y,z)\equiv\frac{1}{2}\left[\begin{array}{c|cccc}
s_{0} & s_{1}^{T} & s_{2}^{T} & \cdots & s_{\ell}^{T}\\
\hline s_{1} & S_{1,1} & S_{1,2}\\
s_{2} & S_{1,2}^{T} & S_{2,2} & \ddots\\
\vdots &  & \ddots & \ddots & S_{\ell-1,\ell}\\
s_{\ell} &  &  & S_{\ell-1,\ell}^{T} & S_{\ell,\ell}
\end{array}\right]
=F+\G_{0}^{T}(y_{0})+\H_{0}^{T}(z_{0})+\sum_{k=1}^{\ell-1}\left[\G_{k}^{T}(y_{k,1},y_{k,2})+\H_{k}^{T}(z_{k})\right].
\]
Here, we write out the adjoint operators in terms of the scalar, vector, and
matrix components of $X$:
\begin{alignat*}{3}
\inner XF & = x_{0}\cdot w_{0} &&+\inner{x_{\ell}}{w_{\ell}} &&\\
\inner X{\H_{0}^{T}(z_{0})} & =x_{0}\cdot(-z_{0})&& &&\\
\inner X{\G_{0}^{T}(y_{0})} & =x_{0}\cdot y_{0}(\|\hat{x}\|^{2}-\rho^{2}) &&+\inner{x_{1}}{-2y_{0}\hat{x}} &&+\inner{X_{1,1}}{y_{0}I}\\
\inner X{\G_{k}^{T}(y_{k,1},y_{k,2})} & =x_{0}\cdot y_{k,2}^{T}b_{k} &&+ \inner{\begin{bmatrix}x_{k}\\x_{k+1}\end{bmatrix}}{\begin{bmatrix}W_{k}^{T}y_{k,2}\\-(y_{k,1}+y_{k,2})\end{bmatrix}} && \\
\inner X{\H_{k}^{T}(z_{k})} & = &&+\inner{x_{k+1}}{-Z_{k}b_{k}} &&+\inner{\begin{bmatrix}X_{k,(k+1)}\\X_{(k+1),(k+1)}\end{bmatrix}}{\begin{bmatrix}-W_{k}^{T}Z_{k}\\Z_{k}\end{bmatrix}}
\end{alignat*}
where $Z_{k}=\diag(z_{k})$. 
Isolating the scalar terms $x_{0}$, we obtain the expression for $s_0$
\[
s_{0}=2\left[w_0+y_{0}(\|\hat{x}\|^{2}-\rho^{2})+\sum_{k=1}^{\ell-1}b_{k}^{T}y_{k,2}-z_{0}\right]
\]
Isolating the vector terms $x_{1},x_{2},\dots,x_{\ell}$, we see that the following indeed holds for $s_1,\ldots,s_\ell$
\begin{gather*}
s_{1}=W_{1}^{T}y_{1,2}-2\hat{x}y_{0},\quad s_{\ell}=w_{\ell}-\left[Z_{(\ell-1)}b_{(\ell-1)}+y_{(\ell-1),1}+y_{(\ell-1),2}\right],\\
s_{k+1}=W_{k+1}^{T}y_{(k+1),2}-\left(Z_{k}b_{k}+y_{k,1}+y_{k,2}\right)\quad\text{for }k\in\{1,\dots,\ell-2\}.
\end{gather*}
Finally, isolating the matrix terms $X_{i,j}$ for $i,j\in \{1,\ldots,\ell\}$, we have
\[
S_{1,1}=2y_{0}I,\quad S_{k,k+1}=-W_{k}^{T}Z_{k},\quad S_{k+1,k+1}=2Z_{k} \quad\text{for }k\in\{1,\dots,\ell-1\}.
\]

\section{Proof of the Main Results}
\subsection{Proof of Proposition~\ref{prop:dual}: the dual lower bounds for (\normalfont\ref{eq:sdp})}
Recall from the previous section, we have rewritten (\ref{eq:sdp}) into the following generic form
\begin{align*}
 \min_{X\succeq0,\ \rank(X)\le r}\ \inner FX\quad\text{s.t.}\quad
\G_{0}(X)\le0,\quad\G_{k}(X)\le0,\quad
\H_{0}(X)+1=0,\quad\H_{k}(X)=0,\quad \tr(X)\le R^2
\end{align*}
Now we are ready to prove \propref{dual}.
\begin{proof}
Let $X^{\star}$ denote the global solution of
(\ref{eq:sdp}) with rank $r\ge1$ and $\tr(X)\le R^2$. Then, for any dual multipliers $y=(y_{0},\{y_{k,1},y_{k,2}\}_{k=1}^{\ell-1})$ and $z=(z_{0},\{z_{k}\}_{k=1}^{\ell-1})$ that satisfy $y\ge0$, we have
\begin{align*}
	\phi[c]\ge \phi_{r}[c] &= \inner F{X^\star}\\
	&=\inner {S(x,y)-\G_{0}^{T}(y_{0})-\H_{0}^{T}(z_{0})-\sum_{k=1}^{\ell-1}\left[\G_{k}^{T}(y_{k,1},y_{k,2})+\H_{k}^{T}(z_{k})\right]}{X^\star}\\
	&=\inner{S(x,y)}{X^\star} - \underbrace{\inner{\begin{bmatrix}\G_{0}(X^\star)\\\vdots\\\G_{\ell-1}(X^\star)\end{bmatrix}}{y}}_{\leq 0} - \underbrace{\inner{\begin{bmatrix}\H_{0}(X^\star)\\\vdots\\\H_{\ell-1}(X^\star)\end{bmatrix}}{z}}_{=-z_0}\\
	&\ge z_{0} + \inner{S(x,y)}{X^\star}\\
	&\ge z_{0}+R^{2}\cdot\min\{0,\lambda_{\min}[S(y,z)]\}.
\end{align*}
\end{proof}

\subsection{Proof of the dual lower bound for (\normalfont\ref{eq:bm})}
we start by putting (\ref{eq:bm}) into the standard-form of nonlinear program (\ref{eq:nlp}). In the previous section, we have rewritten the original SDP problem (\ref{eq:sdp}) in the following generic form 
\begin{align*}
 \min_{X\succeq0,\ \rank(X)\le r}\ \inner FX\quad\text{s.t.}\quad
\G_{0}(X)\le0,\quad\G_{k}(X)\le0,\quad
\H_{0}(X)+1=0,\quad\H_{k}(X)=0,\quad \tr(X)\le R^2
\end{align*}
Since every $(n+1)\times (n+1)$ positive semidefinite matrix $X$ of rank at most $r$ admits an $(n+1)\times n+1$ lower-triangular Cholesky factorization $U$ that satisfies $X=UU^{T}$.  Substituting
\[
X=\begin{bmatrix}u_{0}^{2} & u_{0}\cdot u^{T}\\
u_{0}\cdot u & uu^{T}+VV^{T}
\end{bmatrix}=
\left[\begin{array}{c|c}
u_{0} & 0\\
\hline u_{1} & V_{1}\\
\vdots & \vdots\\
u_{\ell} & V_{\ell}
\end{array}\right]\left[\begin{array}{c|c}
u_{0} & 0\\
\hline u_{1} & V_{1}\\
\vdots & \vdots\\
u_{\ell} & V_{\ell}
\end{array}\right]^{T}=UU^{T},
\]
as in (\ref{eq:bmpart}), we obtain the generic form of our proposed Burer--Monteiro formulation (\ref{eq:bm})
\[
\min_{\|U\|\le R}\ \inner{F}{UU^{T}} \quad \text{s.t.}\quad\G_0(UU^{T})\le0,\quad \G_k(UU^{T})\le0,\quad\H_0(UU^{T})+1=0,\quad\H_k(UU^{T})=0.
\] 
To solve (\ref{eq:bm}) as an instance of the standard-form nonlinear program (\ref{eq:nlp}). Define 
\[
x \equiv \begin{bmatrix}
	u_0 \\ u_1 \\\vdots \\u_\ell \\ \vec(V_1) \\ \vdots \\\vec(V_\ell)
\end{bmatrix},\quad
f(x)\equiv \inner{F}{UU^{T}},\quad
g(x)\equiv \underbrace{\begin{bmatrix}
	\G_0(UU^{T}) \\ \G_1(UU^{T}) \\ \vdots \\ \G_{\ell-1}(UU^{T})
\end{bmatrix}}_{\G(UU^T)},\quad
h(x)\equiv\underbrace{\begin{bmatrix}
	\H_0(UU^{T}) \\ \H_1(UU^{T}) \\ \vdots \\ \H_{\ell-1}(UU^{T})
\end{bmatrix}}_{\H(UU^T)}+\underbrace{\begin{bmatrix}1\\0\\\vdots\\0\end{bmatrix}}_{h_0}
\]
we obtain the standard-form of nonlinear program (\ref{eq:nlp})
\[
\min_{\|x\|\le R}\ f(x)\quad\text{s.t.}\quad g(x)\le0,\quad h(x)=0.
\]
Similar to the previous section, let $y\ge 0$, $z$ and $\mu \le 0$ denote the Lagrangian multiplier associated with $g(x)\le 0$, $h(x) = 0$ and $\|x\|^2\le R^2$, respectively. The corresponding Lagrangian function reads
\begin{equation}\label{eq:langrange_BM}
\L(x,y,z,\mu)=f(x)+y^{T}g(x)+z^{T}h(x)-\mu(\|x\|^2-R^2)=z_0+\mu R^2+\inner{S(x,y)-\mu I}{U(x)U(x)^T}
\end{equation}
where $S(y,z)=F+\G^{T}(y)+\H^{T}(z)$ and $U(x)$ is a matricization operator 
\[U(x)\equiv U(u_{0},u,\vec(V))=\begin{bmatrix}u_{0} & 0\\u & V\end{bmatrix}.\]
General-purpose nonlinear programming solvers work by computing a
feasible primal point $x$ and dual multipliers $y,z,\mu$ that
satisfy the \emph{first-order optimality} condition (known as the
Karush--Kuhn--Tucker (KKT) conditions)
\begin{equation}
\nabla_{x}\L(x,y,z,\mu)=0,\quad y\ge0,\quad y\odot g(x)=0,\quad \mu\le0,\quad \mu\cdot (\|x\|^2-R^2)=0,\tag{FOC}\label{eq:foc}
\end{equation}
and also attempt to achieve the \emph{second-order
optimality} condition (known as the projected Hessian condition)
\begin{equation}
\dot{x}^{T}\nabla_{xx}^{2}\L(x,y,z,\mu)\dot{x}\ge0\quad\text{for all }\dot{x}\in\C(x,y)\tag{SOC}\label{eq:soc}
\end{equation}
in which the \emph{critical cone} is defined 
\begin{equation}
\mathcal{C}(x,y)=\left\{ \dot{x}:\begin{array}{rl}
\nabla g_{i}(x)^{T}\dot{x} & \ge0\quad\text{for all }i\text{ with }g_{i}(x)=0,\\
\nabla g_{i}(x)^{T}\dot{x} & =0\quad\text{for all }i\text{ with }y_{i}>0,\\
\nabla h_{i}(x)^{T}\dot{x} & =0\quad\text{for all }j,\\
2x^T\dot x & \ge 0\quad\text{if $\|x\|^2=R^2$},\\
2x^T\dot x & = 0\quad\text{if $\mu < 0$}.\\
\end{array}\right\} \label{eq:critic}
\end{equation}
However, in the constrained optimization setting, a local minimum
$x^{\star}$ does not need to satisfy (\ref{eq:foc}) and (\ref{eq:soc})\footnote{For an explicit counterexample, consider $f(x)=x$ and $g(x)=x^{2}$
and $h(x)=0$.}. Solvers tend to fail catastrophically when converging
towards a point that does not satisfy (\ref{eq:foc}) and (\ref{eq:soc}),
either by diverging to infinity or cycling through nonsensical solutions. 
Instead, a notion of \emph{constraint qualification }is required
to ensure convergence. Of all possibilities, the LICQ is one of the
stronger conditions that allow strong guarantees to be made.

\begin{definition}[LICQ]
We say that a given $x$ satisfies the \emph{linear independence
constraint qualification} (LICQ) if the following holds
\begin{equation}
\nabla g(x)y+\nabla h(x)z + 2\mu\cdot x = 0,\quad g(x)\odot y=0,\quad \mu\cdot(\|x\|^2-R^2)=0\quad\iff\quad y=0,\quad z=0,\quad \mu=0. \tag{LICQ}\label{eq:licq}
\end{equation}
\end{definition}

One of our main theoretical contribution in this paper is to state the conditions for a point $x$ to satisfy (\ref{eq:licq}). 

\begin{lemma}[LICQ for (\ref{eq:bm})]\label{lem:licq}
If $x$ satisfy the nonzero preactivation condition (\ref{eq:npcq}). Then, $x$ satisfies (\ref{eq:licq}).
\end{lemma}

We defer the proof of \lemref{licq} to the Appendix~\ref{app:LICQ}. \lemref{licq} implies that for every local minimum $x^{\star}$, there exists a unique set of dual variables $(y^{\star},z^{\star},\mu^{\star})$ such that $(x^{\star},y^{\star},z^{\star}\mu^{\star})$ is guaranteed to satisfy (\ref{eq:foc}) and (\ref{eq:soc}); therefore, the nonlinear programming solvers are guaranteed to work.

\begin{corollary}[Dual lower bound of (\ref{eq:bm})]
	Let $x^\star=(u_{0},\{u_k\}_{k=1}^{\ell-1},\{\vec(V_k)\}_{k=1}^{\ell-1})$ denote a local minimum for the Burer--Monteiro problem (\ref{eq:bm}) that satisfy nonzero activation (\ref{eq:npcq}). Then, there exists an unique dual multipliers $y=(y_{0},\{y_{k,1},y_{k,2}\}_{k=1}^{\ell-1})$
and $z=(z_{0},\{z_{k}\}_{k=1}^{\ell-1})$ that satisfy $y\ge0$ provide
the following lower-bound
\[
\phi[c]\ge\phi_{r}[c]\ge z_{0}+R^{2}\cdot\min\{0,\lambda_{\min}[S(y,z)]\}.
\]
\end{corollary}
\begin{proof}
If $x^\star$ is a local minimum for (\ref{eq:bm}) that satisfy nonzero activation (\ref{eq:npcq}). Then, it follows from \lemref{licq} that the point $x$ is first-order optimal, and the corresponding KKT equations (\ref{eq:foc}) yields a unique set of dual multipliers $(y,z,\mu)$ that certify the above lower-bound on the global minimum.
\end{proof}

\subsection{Proof for \thmref{escape}: Escape lifted saddle point.}
Now, let us explain how LICQ leads to our desired results in \thmref{escape}. We start by proving two technique lemmas regarding the first- and second-optimality of (\ref{eq:bm}).

\begin{lemma}[First-order optimality]\label{lem:first_order_opt}
Let $x=(u_{0},\{u_k\}_{k=1}^{\ell-1},\{\vec(V_k)\}_{k=1}^{\ell-1})$ and $y,z,\mu$ satisfy $\nabla_{x}\L(x,y,z,\mu)=0$.
Then, the slack matrix $S(y,z)=F+\G^{T}(y)+\H^{T}(z)$ satisfies the
following:
\begin{itemize}
\item $S(y,z)U(x)=0$.
\item $S(y,z)-\mu I=\begin{bmatrix}u^{T}/u_{0}\\
-I
\end{bmatrix}(S_{2}-\mu I)\begin{bmatrix}u^{T}/u_{0}\\
-I
\end{bmatrix}^{T}$ for some matrix $S_{2}$. 
\end{itemize}
\end{lemma}

\begin{proof}
Let $S(y,z)=\begin{bmatrix}s_0 & s_1^T \\ s_1 & S_2\end{bmatrix}$, the Lagrangian (\ref{eq:langrange_BM}) can be written as the following 
\begin{align*}
\L(x,y,z,\mu) & = z_0+\mu R^2+\inner{\begin{bmatrix}u_{0} & 0\\
u & V
\end{bmatrix}\begin{bmatrix}u_{0} & 0\\
u & V
\end{bmatrix}^{T}}{\begin{bmatrix}s_{0}-\mu & s_{1}^{T}\\
s_{1} & S_{2}-\mu I
\end{bmatrix}}\\
 & =z_0+\mu (R^2-u_0^2)+u_{0}^{2}s_{0}+2u_{0}s_{1}^{T}u+\inner{uu^{T}+VV^{T}}{S_{2}-\mu I}.
\end{align*}
The condition $\nabla_{x}\L(x,y,z,\mu)=0$ is equivalent to setting the
following three Jacobians to zero:
\[
\frac{\partial\L}{\partial u_{0}}=2(u_{0}(s_{0}-\mu)+s_{1}^{T}u),\quad\frac{\partial\L}{\partial u}=2(u_{0}s_{1}^{T}+u^{T}(S_{2}-\mu I)),\quad\frac{\partial\L}{\partial V}=2V^{T}(S_{2}-\mu I).
\]
It follows by substituting the above into the equation below that
\[
S(y,z)U(x)=\begin{bmatrix}s_{0}-\mu & s_{1}^{T}\\
s_{1} & S_{2}-\mu I
\end{bmatrix}\begin{bmatrix}u_{0} & 0\\
u & V
\end{bmatrix}=\begin{bmatrix}u_{0}(s_{0}-\mu)+s_{1}^{T}u & s_{1}^{T}V\\
s_{1}u_{0}+(S_{2}-\mu I)u & (S_{2}-\mu I)V
\end{bmatrix}=0,
\]
where $s_{1}^{T}V=0$ because $s_{1}^{T}=-u^{T}(S_{2}-\mu I)/u_{0}$ and $(S_{2}-\mu I)V=0$.

Similarly, substituting $s_{0}=-s_{1}^{T}u/u_{0}+\mu$ and $s_{1}=-(S_{2}-\mu I)u/u_{0}$
yields
\[
S(x,y)=\begin{bmatrix}s_{0} & s_{1}^{T}\\
s_{1} & S_{2}
\end{bmatrix}=\begin{bmatrix}u^{T}(S_{2}-\mu I)u/u_{0}^{2}+\mu & -u^{T}(S_{2}-\mu I)/u_{0}\\
-(S_{2}-\mu I)u/u_{0} & S_{2}
\end{bmatrix}=\begin{bmatrix}u^{T}/u_{0}\\
-I
\end{bmatrix}(S_{2}-\mu I)\begin{bmatrix}u^{T}/u_{0}\\
-I
\end{bmatrix}^{T}+\mu I.
\]
\end{proof}
\begin{lemma}[Rank-deficient second-order optimality]
Given $y,z$, let $x=(u_{0},\{u_k\}_{k=1}^{\ell-1},\{\vec(V_k)\}_{k=1}^{\ell-1})$ satisfy $\nabla_{x}\L(x,y,z,\mu)=0$.
If there exists unit vectors $\psi\in\R^{r},\|\psi\|=1$ and $(\xi_{0},\xi_{1})\in\R\times\R^{n},\|(\xi_{0},\xi_{1})\|=1$
such that
\[
V\psi=0,\qquad 2\begin{bmatrix}\xi_{0}\\
\xi_{1}
\end{bmatrix}^{T}(S(y,z)-\mu I)\begin{bmatrix}\xi_{0}\\
\xi_{1}
\end{bmatrix}=-\gamma<0
\]
where $S(y,z)=F+\G^{T}(y)+\H^{T}(z),$ then the vector $\dot{x}=(0,0_{n},\{\vec(\dot{V_k})\}_{k=1}^{\ell-1})$
with $\dot{V_k}=[(u_k/u_{0})\xi_{0}-\xi_{1}]\psi^{T}$ satisfies 
\[
\nabla g(x)^{T}\dot{x}=0,\qquad\nabla h(x)^{T}\dot{x}=0, \qquad \dot{x}^{T}\nabla_{xx}^{2}\L(x,y,z)\dot{x}=-\gamma.
\]
\end{lemma}

\begin{proof}
Note that in general, a function of the form $f(x)=\inner F{UU^T}$
with $U(x)=\begin{bmatrix}u_{0} & 0\\
u & V
\end{bmatrix}$ has directional derivatives 
\[
\nabla f(x)^{T}\dot{x}=\inner F{U(x){U(\dot{x})}^{T}+U(\dot{x})U(x)^{T}},\quad\dot{x}^{T}\nabla^{2}f(x)\dot{x}=2\inner F{U(\dot{x})U(\dot{x})^{T}}.
\]
For our specific choice of $\dot{x}$, we can verify that
\[
U(x)U(\dot{x})^{T}=\begin{bmatrix}u_{0} & 0\\
u & V
\end{bmatrix}\begin{bmatrix}0 & 0\\
0 & [(u/u_{0})\xi_{0}-\xi_{1}]\psi^{T}
\end{bmatrix}^{T}=\begin{bmatrix}0 & 0\\
0 & V\psi[(u/u_{0})\xi_{0}-\xi_{1}]^{T}
\end{bmatrix}=0.
\]
Since $g_{i}(x)=\inner{G_{i}}{UU^{T}}$ and $h_{j}(x)=h_{0,j}+\inner{H_{j}}{UU^{T}}$ for some matrix $G_{i}, H_{j}$ and scalar $h_{0,j}$,  it then follows
that $\nabla g_{i}(x)^{T}\dot{x}=0$ and $\nabla h_{j}(x)^{T}\dot{x}=0$
for all $i$ and $j$. 

Next, given that the Lagrangian (\ref{eq:langrange_BM}) is written
\[
\L(x,y,z,\mu)=\inner{h_{0}}z+\mu R^2 + \inner{S(x,y)-\mu I}{U(x)U(x)^{T}},
\]
For our specific choice of $\dot{x}$, the second-order directional derivative reads 
\[
\dot{x}^{T}\nabla_{xx}^{2}\L(x,y,z,\mu)\dot{x}=2\inner {S(x,y)-\mu I}{U(\dot{x})U(\dot{x})^{T}}=2\begin{bmatrix}0\\
\xi_{2}
\end{bmatrix}^{T}(S(x,y)-\mu I)\begin{bmatrix}0\\
\xi_{2}
\end{bmatrix}.
\]
where $\xi_{2}=(u/u_{0})\xi_{0}-\xi_{1}$. It follows from \lemref{first_order_opt} that for $\nabla_{x}\L(x,y,z,\mu)=0$, the slack matrix satisfies\\ $S(x,y)-\mu I=\begin{bmatrix}u^{T}/u_{0}\\
-I
\end{bmatrix}(S_{2}-\mu I)\begin{bmatrix}u^{T}/u_{0}\\
-I
\end{bmatrix}^{T}$ and therefore
\begin{align*}
2\begin{bmatrix}0\\\xi_{2}\end{bmatrix}^{T}
(S(x,y)-\mu I)
\begin{bmatrix}0\\\xi_{2}\end{bmatrix}
& =2\begin{bmatrix}0\\\xi_{2}\end{bmatrix}^{T}
\begin{bmatrix}u^{T}/u_{0}\\-I\end{bmatrix}
(S_{2}-\mu I)
\begin{bmatrix}u^{T}/u_{0}\\-I\end{bmatrix}^{T}
\begin{bmatrix}0\\\xi_{2}\end{bmatrix}\\
& =2\begin{bmatrix}\xi_{0}\\\xi_{1}\end{bmatrix}^{T}
\begin{bmatrix}u^{T}/u_{0}\\-I\end{bmatrix}
(S_{2}-\mu I)
\begin{bmatrix}u^{T}/u_{0}\\-I\end{bmatrix}^{T}
\begin{bmatrix}\xi_{0}\\\xi_{1}\end{bmatrix}\\
&=2\begin{bmatrix}\xi_{0}\\\xi_{1}\end{bmatrix}^{T}
(S(x,y)-\mu I)
\begin{bmatrix}\xi_{0}\\\xi_{1}\end{bmatrix}\\
&=-\gamma.
\end{align*}
\end{proof}
\begin{lemma}[Critical cone]\label{lem:licq-2}
Let $\Omega$ denote a set of feasible points of (\ref{eq:bm}) where (\ref{eq:licq}) hold. If $x\in\Omega$, and $(y,z,\mu)$ satisfy (\ref{eq:foc}), then there
exists a continuously differentiable path $x(t)$ with initial position
$x(0)=x$ that satisfies
\[
f(x(t))=\L(x(t),y,z),\quad x(t)\in\Omega\quad\text{for all }t\in[0,\epsilon)
\]
if and only if $\dot{x}(0)\in\C(x,y)$. Moreover, if $g(x)$ and $h(x)$
are $k$-times continuously differentiable, then $x(t)$ is also $k$-times
continuously differentiable.
\end{lemma}

We are now ready to prove the escape result. 
\begin{proof}
Let $x$ be first-order optimal for (\ref{eq:bm})
with dual multipliers $(y,z,\mu)$. If $x$ satisfies nonzero activation,
and $\gamma=-\lambda_{\min}[S(y,z)-\mu I]>0$, then the eigenvector
$\xi=(\xi_{0},\xi_{1})$ that satisfies $\xi^{T}(S(y,z)-\mu I)\xi=-\gamma\|\xi\|^{2}$
implicitly defines an escape path 
\[
u_{k,+}(t)=u_{k}+O(t^{2}),\qquad V_{k,+}(t)=[V_{k},0]+t\cdot[0,(u_{k}/u_{0})\xi_{0}-\xi_{k}]+O(t^{2})
\]
so that $x_{+}(t)=(u_{0},u(t),V(t))$ is feasible with sufficiently
small $t\ge0$, and for which the objective makes a decrement as follows
\[
w_{\ell}^{T}u_{\ell,+}(t)=w_{\ell}^{T}u_{\ell}-2t^{2}\gamma+O(t^{3})\text{ for all }t\in[0,\epsilon).
\]
\end{proof}

\section{Proof of Constraint Qualification}\label{app:LICQ}
In this section, we provide the proof for \lemref{licq}. Specifically, we show that the (\ref{eq:licq}) holds for (\ref{eq:bm}) under the assumption of nonzero preactivation (\ref{eq:npcq}). Throughout this section, we assume $R$ is chosen large enough such that $\|x\|>R$ holds at optimality with a \emph{strict inequality}, which in turn implies the corresponding dual variable $\mu=0$.

We start by showing that (\ref{eq:licq}) holds for single neuron and single layer ReLU networks.
\begin{lemma}[Single neuron]\label{lem:one-neuron} 
Define $\alpha_0,\alpha,\overline{\alpha}\in\R$ and $\beta,\overline{\beta},\in\R^{r-1}$. Let $x=(\alpha_0,\alpha,\overline{\alpha},\beta,\overline{\beta})$ satisfy $g_{1}(x)\ge0$, $g_{2}(x)\ge0$ and $h_1(x)=0$, where 
\[
g_{1}(x)=\alpha_{0}\overline{\alpha},\quad 
g_{2}(x)=\alpha_{0}(\overline{\alpha}-\alpha-\gamma\alpha_0),\quad 
h_{1}(x)=\overline{\alpha}(\overline{\alpha}-\alpha-\gamma\alpha_0)+\inner{\overline{\beta}}{\overline{\beta}-\beta}.
\]
Suppose that $\alpha_{0}^2\neq 1$, $\alpha+\gamma\alpha_0\ne0$ and $\beta\ne0$.
Then, the following holds
\begin{subequations}
\begin{gather}
\nabla g_{1}(x)y_{1}+\nabla g_{2}(x)y_{2}+ \nabla h_1(x)z_1=0,\label{eq:one-neuron-a}\\ 
g_{1}(x)\odot y_{1}=g_{2}(x)\odot y_{2}=0,\label{eq:one-neuron-b}
\end{gather}
\end{subequations}
if and only if $y_{1}=y_{2}=z_{1}=0$.
\end{lemma}

\begin{proof}
Explicitly computing the gradient terms in (\ref{eq:one-neuron-a}), we can write (\ref{eq:one-neuron-a}) and (\ref{eq:one-neuron-b}) as the following
\begin{equation*}
\begin{bmatrix}
\overline{\alpha} & \overline{\alpha}-\alpha-2\gamma\alpha_0 & -\overline{\alpha}\gamma\\
0 &  -\alpha_0 & -\overline{\alpha} \\
\alpha_0 & \alpha_0  & 2\overline{\alpha} - \alpha -\gamma\alpha_0 \\
0 &  0  & -\overline{\beta} \\
0 &  0  & 2\overline{\beta}-\beta\\
\alpha_0\overline{\alpha} & 0 & 0 \\
0 & \alpha_{0}(\overline{\alpha}-\alpha-\gamma\alpha_0) & 0
\end{bmatrix}
\begin{bmatrix}
y_{1}\\y_{2}\\z_1
\end{bmatrix}=0
\end{equation*}
The goal is to show that the above matrix has full column rank. To simplify our proof, we delete the first,  the second and the fourth row as these rows are obviously dependent to the the rest of the rows. Deleting those three rows reveals the desired claim as equivalent to the following
\begin{equation}\label{eq:claimneuron}
\begin{bmatrix}
\alpha_0 & \alpha_0  & 2\overline{\alpha} - \alpha -\gamma\alpha_0 \\
0 &  0  & 2\overline{\beta}-\beta\\
\alpha_0\overline{\alpha} & 0 & 0 \\
0 & \alpha_{0}(\overline{\alpha}-\alpha-\gamma\alpha_0) & 0
\end{bmatrix}
\begin{bmatrix}
y_{1}\\y_{2}\\z_1
\end{bmatrix}=0
\qquad\iff\qquad
\begin{bmatrix}
y_{1}\\y_{2}\\z_1
\end{bmatrix}=0.
\end{equation}
Next, completing the square $h_1(x)=\overline{\alpha}(\overline{\alpha}-\alpha-\gamma\alpha_0)+\inner{\overline{\beta}}{\overline{\beta}-\beta}=\|(\overline{\alpha},\overline{\beta})-\tfrac{1}{2}(\alpha+\gamma\alpha_0,\beta)\|^{2}-\|\tfrac{1}{2}(\alpha+\gamma\alpha_0,\beta)\|^{2}$
reveals that 
\begin{equation}\label{eq:h1}
h_1(x)=0\quad\implies\quad\|(2\overline{\alpha}-\alpha-\gamma\alpha_0,2\overline{\beta}-\beta)\|=\|(\alpha+\gamma\alpha_0,\beta)\|
\end{equation}
If additionally $\alpha_0\overline{\alpha}=\max\{0,\alpha_0(\alpha+\gamma\alpha_0)\}$, i.e. when $g_1(x)=0$ or $g_2(x)=0$, or $g_1(x)=g_2(x)=0$, then substituting
$g_1(x)g_2(x)=\alpha_0^2\overline{\alpha}(\overline{\alpha}-\alpha-\gamma\alpha_0)=\overline{\alpha}(\overline{\alpha}-\alpha-\gamma\alpha_0)=0$ into (\ref{eq:h1}) further yields
\begin{equation}\label{eq:h2}
\alpha_0\overline{\alpha}=\max\{0,\alpha_0(\alpha+\gamma\alpha_0)\},\quad 
h_1(x)=0\quad\implies\quad
\|2\overline{\beta}-\beta\|=\|\beta\|.
\end{equation}
Finally, from $|g_{1}(x)-g_{2}(x)|=|\alpha_0(\alpha+\gamma\alpha_0)|=|\alpha+\gamma\alpha_0|>0$,
it follows that we cannot jointly have both $g_{1}(x)=0$ and $g_{2}(x)=0$
at the same time. We proceed by analyzing that (\ref{eq:claimneuron}) holds true for the other three cases one at a time:
\begin{itemize}
\item If $g_{1}(x)=0$ and $g_{2}(x)>0$, then $y_{2}=0$.
Substituting $y_{2}=0$ into (\ref{eq:claimneuron}), it follows from $\|2\overline{\beta}-\beta\|=\|\beta\|>0$ via (\ref{eq:h2}) and $\alpha_0^2 = 1$ that
\[
\begin{bmatrix}
\alpha_0   & 2\overline{\alpha} - \alpha -\gamma\alpha_0 \\
0 & 2\overline{\beta}-\beta\\
\alpha_0\overline{\alpha} & 0
\end{bmatrix}
\begin{bmatrix}
y_{1}\\z_1
\end{bmatrix}=0\quad\iff\quad
\begin{bmatrix}
y_{1}\\z_1
\end{bmatrix}=0.
\]
\item If $g_{1}(x)>0$ and $g_{2}(x)=0$, then $y_{1}=0$.
Substituting $y_{1}=0$ into (\ref{eq:claimneuron}), it again follows from $\|2\overline{\beta}-\beta\|=\|\beta\|>0$ via (\ref{eq:h2}) and $\alpha_0^2 = 1$ that the following holds
\[
\begin{bmatrix}
\alpha_0  & 2\overline{\alpha} - \alpha -\gamma\alpha_0 \\
0  & 2\overline{\beta}-\beta\\
\alpha_{0}(\overline{\alpha}-\alpha-\gamma\alpha_0) & 0
\end{bmatrix}
\begin{bmatrix}
y_{2}\\z_1
\end{bmatrix}=0\quad\iff\quad
\begin{bmatrix}
y_{2}\\z_1
\end{bmatrix}=0.
\]
\item Finally, if $g_{1}(x)>0$ and $g_{2}(x)>0$,
then both $y_{1}=y_{2}=0$. Substituting $y_{1}=y_{2}=0$ into (\ref{eq:claimneuron}),
it follows from $\quad\|(2\overline{\alpha}-\alpha-\gamma\alpha_0,2\overline{\beta}-\beta)\|=\|(\alpha+\gamma\alpha_0,\beta)\| > 0$ via (\ref{eq:h1}) and $\alpha_0^2 = 1$ that
\[
\begin{bmatrix}
2\overline{\alpha} - \alpha -\gamma\alpha_0 \\
2\overline{\beta}-\beta
\end{bmatrix}
z_1=0\quad\iff\quad 
z_1=0.
\]
\end{itemize}
\end{proof}
\begin{lemma}[Single layer]\label{lem:one-layer}
Define $u_0\in\R$, $u\in\R^{n}$, $\overline{u}\in\R^{\overline{n}}$, $V\in\R^{n\times(r-1)}$ and $\overline{V}\in\R^{\overline{n}\times(r-1)}$. Let $x=(u_0,u,\overline{u},\vec(V),\vec(\overline{V}))$ satisfy $g_{1}(x)\ge0$, $g_{2}(x)\ge0$ and $h_1(x)=0$, where
\[
g_{1}(x)=u_0\cdot\overline{u},\quad 
g_{2}(x)=u_0\cdot(\overline{u}-Wu-bu_0),\quad
h(x)=\diag[(\overline{u}-Wu-bu_0)\overline{u}^{T}+(\overline{V}-WV)\overline{V}^{T}].
\]
Suppose that $u_0^2=1$, $\e_{i}^{T}(Wu+bu_0)\ne0$ and $\e_{i}^{T}WV\ne0$
hold for all $i\in\{1,2,\dots,\overline{n}\}$. Then, the following holds
\begin{subequations}
\begin{gather}
\nabla g_{1}(x)y_{1}+\nabla g_{2}(x)y_{2}+ \nabla h_1(x)z_1=0, \label{eq:one-layer-a}\\ 
g_{1}(x)\odot y_{1}=g_{2}(x)\odot y_{2}=0, \label{eq:one-layer-b}
\end{gather}
\end{subequations}
if and only if $y_{1}=y_{2}=z_{1}=0$.
\end{lemma}

\begin{proof}
Let $V=\begin{bmatrix}v_1&\cdots v_{r-1}\end{bmatrix}$ and $\overline{V}=\begin{bmatrix}\overline{v}_1&\cdots \overline{v}_{r-1}\end{bmatrix}$. Explicitly write out the gradient terms in (\ref{eq:one-layer-a}) and stack it together with (\ref{eq:one-layer-b}), we have
\begin{equation*}
\begin{bmatrix}
\overline{u}^T & (\overline{u}-Wu-2bu_0)^T & -\overline{u}^T\diag(b)\\
0 & -u_0\cdot W^T & -W^T\diag(\overline{u})\\
u_0\cdot I & u_0\cdot I & \diag(2\overline{u}-Wu-bu_0)\\
0 & 0 & -W^T\diag(\overline{v}_{1})\\
\vdots & \vdots & \vdots\\
0 & 0 & -W^T\diag(\overline{v}_{r-1})\\
0 & 0 & \diag(2\overline{v}_{1}-Wv_1)\\
\vdots & \vdots & \vdots\\
0 & 0 & \diag(2\overline{v}_{r-1}-Wv_{r-1})\\
u_0\cdot\diag(\overline u) & 0 & 0 \\
0 & u_0\cdot\diag(\overline{u}-Wu-bu_0) & 0 \\
\end{bmatrix}
\begin{bmatrix}
y_{1}\\y_{2}\\z_1
\end{bmatrix}=0
\end{equation*}
Similar to \lemref{one-neuron}, deleting dependent rows allows us to restate the desired claim as the following
\begin{equation}\label{eq:one-layer-claim1}
\begin{bmatrix}
u_0\cdot I & u_0\cdot I & \diag(2\overline{u}-Wu-bu_0)\\
0 & 0 & \diag(2\overline{v}_{1}-Wv_1)\\
\vdots & \vdots & \vdots\\
0 & 0 & \diag(2\overline{v}_{r-1}-Wv_{r-1})\\
u_0\cdot\diag(\overline u) & 0 & 0 \\
0 & u_0\cdot\diag(\overline{u}-Wu-bu_0) & 0 \\
\end{bmatrix}
\begin{bmatrix}
y_{1}\\y_{2}\\z_1
\end{bmatrix}=0
\quad\iff\quad\begin{bmatrix}
y_{1}\\y_{2}\\z_1
\end{bmatrix}=0.
\end{equation}
Collecting rows correspond to each $(\e_i^Ty_1,\e_i^Ty_2,\e_i^Tz_1)=(y_{1,i},y_{2,i},z_{1,i})$, we see that (\ref{eq:one-layer-claim1}) holds true if and only if the following holds true for all $i\in\{1,2,\dots,\overline{n}\}$:
\begin{equation}\label{eq:one-layer-claim2}
\begin{bmatrix}
\alpha_0 & \alpha_0 & 2\overline{\alpha}_{i} - \alpha_{i} -\gamma_{i}\alpha_0 \\
0 &  0 & 2\overline{\beta}_{i}-\beta_{i}\\
\alpha_0\overline{\alpha}_{i} & 0 & 0 \\
0 & \alpha_{0}(\overline{\alpha}_{i}-\alpha_{i}-\gamma_{i}\alpha_0) & 0
\end{bmatrix}
\begin{bmatrix}
y_{1,i}\\y_{2,i}\\z_{1,i}
\end{bmatrix}=0
\qquad\iff\qquad
\begin{bmatrix}
y_{1,i}\\y_{2,i}\\z_{1,i}
\end{bmatrix}=0.
\end{equation}
where 
\[
\alpha_0=u_0,\quad \alpha_{i}=\e_{i}^{T}Wu,\quad \overline{\alpha}_{i}=\e_{i}^{T}\overline{u},\quad \beta_{i}^{T}=\e_{i}^{T}WV,\quad \overline{\beta}_{i}^{T}=\e_{i}^{T}\overline{V},\quad \gamma_{i} = \e_{i}^Tb.
\]
By hypothesis, $u_0^2=\alpha_0^2=1$, $\e_{i}^{T}(Wu+bu_0)=\alpha_{i}+\gamma_{i}\alpha_0\ne0$ and $\e_{i}^{T}WV=\beta_{i}^T\ne0$, it then follows from \lemref{one-neuron} that (\ref{eq:one-layer-claim2}) holds true for all $i$. This proves the lemma.
\end{proof}
The results from \lemref{one-neuron} and \lemref{one-layer} can be easily extended to the multiple layers case.
\begin{lemma}[Multiple layers]\label{lem:multiple-layers}
Define $u_0\in\R$, $u=(u_{1},\dots,u_{\ell})\in\R^{n}$, $V=(V_{1},\dots,V_{\ell})\in\R^{n\times(r-1)}$. Let $x=(u_0,\{u_k\}_{k=1}^{\ell-1},\{\vec(V_k)\}_{k=1}^{\ell})$ satisfy $g_{k,1}(x)\ge0$, $g_{k,2}(x)\ge0$ and $h_{k}(x)=0$ for all $k\in\{1,\ldots,\ell-1\}$, where 
\begin{gather*}
g_{k,1}(x)=u_0\cdot u_{k+1},\qquad g_{k,2}(x)=u_0\cdot(u_{k+1}-W_{k}u_{k}-b_{k}u_0),\\
h_{k}(x)=\diag[(u_{k+1}-W_{k}u_{k}-b_{k}u_0)u_{k+1}^{T}+(V_{k+1}-WV_{k})V_{k+1}^{T}].
\end{gather*}
Suppose that $u_0^2 = 1$, $\e_{i}^{T}(W_{k}u_{k}+b_{k}u_0)\ne0$ and $\e_{i}^{T}W_{k}V_{k}\ne0$ hold for all $k\in\{1,\ldots,\ell-1\}$ and $i\in\{1,2,\dots,n_{k+1}\}$.
Then, the following holds
\begin{subequations}
\begin{gather}
\sum_{k=1}^{\ell-1}\left[\nabla g_{k,1}(x)y_{k,1}+\nabla g_{k,2}(x)y_{k,2}+\nabla h_{k}(x)z_{k}\right]=0,\label{eq:multiple-layer-a} \\
g_{k,1}(x)\odot y_{k,1}=g_{k,2}(x)\odot y_{k,2}=0\quad\text{for all }k\in\{1,2,\dots,\ell-1\}, \label{eq:multiple-layer-b}
\end{gather}
\end{subequations}
if and only if $y_{k,1}=y_{k,2}=z_{k}=0$ for all $k\in\{1,2,\dots,\ell-1\}$.
\end{lemma}

\begin{proof}
Let us assume $r=2$ without loss of generality. Similar to the proof in \lemref{one-layer}, we start by writing (\ref{eq:multiple-layer-a}) and (\ref{eq:multiple-layer-b}) into a matrix-vector product. Let $\frac{\partial f(x)}{\partial y}$ denote the gradient of $f(x)$ with respect to variable $y$ and let $x_k=(u_0,u_k,u_{k+1},\vec(V_k),\vec(V_{k+1}))$. Analogous to the one layer case, for each $k\in\{1,\ldots,\ell-1\}$, we define the following block matrix
\begin{equation*}
\renewcommand\arraystretch{2}
\begin{bmatrix}
\displaystyle\frac{\partial g_{1,k}(x)}{\partial x_k} & \displaystyle\frac{\partial g_{2,k}(x)}{\partial x_k} & \displaystyle\frac{\partial h_k(x)}{\partial x_k}\\
\diag(g_{1,k}(x)) & 0 & 0 \\
0 & \diag(g_{2,k}(x)) & 0
\end{bmatrix}=
\renewcommand\arraystretch{1.2}
\begin{bmatrix}
	a_k^T \\B_k \\C_k \\D_k \\E_k \\F_k
\end{bmatrix},\qquad
\mathbf{M}_{k}=\begin{bmatrix}
	C_k\\E_k\\F_k
\end{bmatrix}.
\end{equation*}
Notice that $\mathbf{M}_{k}$ has the same structure as in (\ref{eq:one-layer-claim1}). Each block above is assigned as the following
\newcommand{\colwidth}{12em}
\newcommand{\colone}{7em}
\newcommand{\coltwo}{12em}
\newcommand{\colthr}{12em}
\begin{align*}
a_k^T=\frac{\partial (g_{k,1},g_{k,2},h_k)}{\partial u_{0}}=&
\begin{bmatrix}\begin{array}{C{\colone}C{\coltwo}C{\colthr}}
	u_{k+1}^T & (u_{k+1}-Wu_{k}-2b_{k}u_0)^T & -u_{k+1}^T\diag(b)
\end{array}\end{bmatrix},\\
B_k=\frac{\partial (g_{k,1},g_{k,2},h_k)}{\partial u_{k}}=&
\begin{bmatrix}\begin{array}{C{\colone}C{\coltwo}C{\colthr}}
	0 & -u_0\cdot W_k^T & -W_k^T\diag(u_{k+1})
\end{array}\end{bmatrix},\\
C_k=\frac{\partial (g_{k,1},g_{k,2},h_k)}{\partial u_{k+1}}=&
\begin{bmatrix}\begin{array}{C{\colone}C{\coltwo}C{\colthr}}
	u_0\cdot I & u_0\cdot I & \diag(2u_{k+1}-W_{k}u_{k}-b_{k}u_0)
\end{array}\end{bmatrix},\\
D_k=\frac{\partial (g_{k,1},g_{k,2},h_k)}{\partial \vec(V_{k})}=&
\begin{bmatrix}\begin{array}{C{\colone}C{\coltwo}C{\colthr}}
	0 & 0 & -W_k^T\diag(v_{k+1,1}) \\ 
	\vdots & \vdots & \vdots \\ 
	0 & 0 & -W_k^T\diag(v_{k+1,r-1})
\end{array}\end{bmatrix},\\
E_k=\frac{\partial (g_{k,1},g_{k,2},h_k)}{\partial \vec(V_{k+1})}=&
\begin{bmatrix}\begin{array}{C{\colone}C{\coltwo}C{\colthr}}
	0 & 0 & \diag(2v_{k+1,1}-W_{k}v_{k,1}) \\ 
	\vdots & \vdots & \vdots \\ 
	0 & 0 & \diag(2v_{k+1,r-1}-W_{k}v_{k,r-1})
\end{array}\end{bmatrix},\\
F_k=&\begin{bmatrix}\begin{array}{C{\colone}C{\coltwo}C{\colthr}}
	u_0\cdot\diag(u_{k+1}) & 0 & 0 \\ 
	0 & u_0\cdot\diag(u_{k+1}-W_{k}u_{k}-b_{k}u_0) & 0
\end{array}\end{bmatrix}.
\end{align*}

Since each $g_{k,1}(x),g_{k,2}(x),h_k(x)$ depends only on $x_k$. It follows that (\ref{eq:multiple-layer-a}) and (\ref{eq:multiple-layer-b}) can be written as a matrix-vector product in which the corresponding matrix admit a block tri-diagonal structure. This allows us to restated the desire claim as the following
\begin{equation}\label{eq:blockdiag}
\begin{bmatrix}
	a_1^T & a_2^T & \cdots & a_{\ell-1}^T\\
    B_1 \\
 	C_1 & B_2 \\
 	    & C_2 & \ddots \\
        &     & \ddots & B_{\ell-1} \\
        &     &        & C_{\ell-1} \\
    D_1 \\
 	E_1 & D_2 \\
  	    & E_2 & \ddots \\
        &     & \ddots & D_{\ell-1} \\
        &     &        & E_{\ell-1} \\
    F_1 \\
        & F_2 \\
  	    &     & \ddots \\
        &     &        & F_{\ell-1} \\
\end{bmatrix}
\begin{bmatrix}
\lambda_{1}\\\lambda_{2}\\\vdots\\\lambda_{\ell-1}
\end{bmatrix}=0
\qquad\iff\qquad
\begin{bmatrix}
\lambda_{1}\\\lambda_{2}\\\vdots\\\lambda_{\ell-1}
\end{bmatrix}=0,
\end{equation}
where $\lambda_{k}=(y_{k,1},y_{k,2},z_{k})$.
 
We now process to show that (\ref{eq:blockdiag}) is true. Focusing our attention on the ($\ell-1$)-th block-row. Observe that the blocks $C_{\ell-1}$, $E_{\ell-1}$ and $F_{\ell-1}$ are the only nonzero blocks in their row; therefore, the left-hand side of (\ref{eq:blockdiag}) implies $\mathbf{M}_{\ell-1}\lambda_{\ell-1}=0$.
Given that $\e_{i}^{T}(W_{\ell-1}u_{\ell-1}+b_{\ell-1}u_0)\ne0$ and $\e_{i}^{T}W_{\ell-1}V_{\ell-1}\ne0$ hold for all $i\in\{1,2,\dots,n_{\ell}\}$ by hypothesis, it then follows from \lemref{one-layer} that $\mathbf{M}_{\ell-1}\lambda_{\ell-1}=0 \iff \lambda_{\ell-1}=0$.

Next, substituting $\lambda_{\ell-1}=0$ back to the left-hand side of (\ref{eq:blockdiag}) to eliminate the ($\ell-1$)-th block-row. Repeat the same process for the ($\ell-2$)-th block-row all the way down to the first block-row to show that $\mathbf{M}_{\ell-2}\lambda_{\ell-2}=0\iff\lambda_{\ell-2}=0,\ \mathbf{M}_{\ell-3}\lambda_{\ell-3}=0\iff\lambda_{\ell-3}=0,\ldots,\ \mathbf{M}_{1}\lambda_{1}=0\iff\lambda_{1}=0$. This proves that the left-hand side of (\ref{eq:blockdiag}) does indeed imply the right-hand side under our stated hypotheses,
as desired.

\end{proof}
\begin{theorem}[LICQ for (\ref{eq:bm})]\label{thm:licq_bm}
Define $u_0\in\R$, $u=(u_{1},\dots,u_{\ell})\in\R^{n}$, $V=(V_{1},\dots,V_{\ell})\in\R^{n\times(r-1)}$. Let $x=(u_0,\{u_k\}_{k=1}^{\ell-1},\{\vec(V_k)\}_{k=1}^{\ell})$ satisfy constraints in (\ref{eq:bm}), $g_0(x)\ge 0$, $g_{k,1}(x)\ge0$, $g_{k,2}(x)\ge0$, $h_0(x)=0$ and $h_{k}(x)=0$ for all $k\in\{1,\ldots,\ell-1\}$, where 
\begin{gather*}
g_{0}(x)=\rho^{2}-\|u_{1}-\hat{x}u_{0}\|^{2}-\|V_{1}\|^{2},\qquad h_{0}(x)=u_{0}^2-1,\\
g_{k,1}(x)=u_{0}\cdot u_{k+1},\qquad g_{k,2}(x)=u_{0}\cdot\left(u_{k+1}-W_{k}u_{k}-b_{k}u_{0}\right),\\
h_{k}(x)=\diag[(u_{k+1}-W_{k}u_{k}-b_{k}u_0)u_{k+1}^{T}+(V_{k+1}-WV_{k})V_{k+1}^{T}].
\end{gather*}
Suppose that $x$ satisfies (\ref{eq:npcq}), i.e. $\e_{i}^{T}(W_{k}u_{k}+b_{k}u_0)\ne0$ and $\e_{i}^{T}W_{k}V_{k}\ne0$ hold for all $k\in\{1,2,\dots,\ell-1\}$ and $i\in\{1,2,\dots,n_{k+1}\}$.
Then, $x$ satisfies (\ref{eq:licq}) stated as the following:
\begin{gather}
\nabla g_{0}(x)y_{0}+\nabla h_{0}(x)z_{0}+\sum_{k=1}^{\ell-1}\left[\nabla g_{k,1}(x)y_{k,1}+\nabla g_{k,2}(x)y_{k,2}\right]=0,\tag{LICQ-a}\label{eq:licq-a} \\
g_{0}\odot y_0=0,\qquad g_{k,1}(x)\odot y_{k,1}=g_{2}(x)\odot y_{k,2}=0\quad\text{for all }k\in\{1,\dots,\ell-1\},\tag{LICQ-b}\label{eq:licq-b}
\end{gather}
if and only if $y_0=z_0=0$ and $y_{k,1}=y_{k,2}=z_{k}=0$ \textup{for all }$k\in\{0,1,\dots,\ell-1\}$.
\end{theorem}

\begin{proof}
Let us assume $r=2$ and $\rho>0$ without loss of generality. 
Similar to \lemref{multiple-layers}, (\ref{eq:licq-a}) and (\ref{eq:licq-b}) admits a tri-diagonal structure which allows us to restate the desired claim as the following
\begin{equation}\label{eq:blockdiag-1}
\begin{bmatrix}
2u_0 & a_0 &  a_1^T & a_2^T & \cdots & a_{\ell-1}^T\\
     & b_0 & B_1 \\
 	 &     & C_1 & B_2 \\
 	 &     &     & C_2 & \ddots \\
     &     &     &     & \ddots & B_{\ell-1} \\
     &     &     &     &        & C_{\ell-1} \\
     & d_0 & D_1 \\
 	 &     & E_1 & D_2 \\
  	 &     &     & E_2 & \ddots \\
     &     &     &     & \ddots & D_{\ell-1} \\
     &     &     &     &        & E_{\ell-1} \\
     & f_0 \\
	 &     & F_1 \\
	 &     &     & F_2 \\
	 &     &	 &     & \ddots \\
     &     &     &     &        & F_{\ell-1} \\
\end{bmatrix}
\begin{bmatrix}
z_{0}\\y_{0}\\\lambda_{1}\\\lambda_{2}\\\vdots\\\lambda_{\ell-1}
\end{bmatrix}=0
\qquad\iff\qquad
\begin{bmatrix}
z_{0}\\y_{0}\\\lambda_{1}\\\lambda_{2}\\\vdots\\\lambda_{\ell-1}
\end{bmatrix}=0,
\end{equation}
in which $\lambda_{k}=(y_{k,1},y_{k,2},z_{k})$. 
The blocks $(a_k,B_k,C_k,D_k,E_k)$ for all $k\in\{1,\ldots,\ell-1\}$ are defined in \lemref{multiple-layers}, and the blocks $(a_0,b_0,d_0,f_0)$ are written
\begin{gather*}
a_0=\frac{\partial g_0}{\partial u_0}=2(\hat x^Tu_1-\|\hat x\|^2u_0),\quad b_0=\frac{\partial g_0}{\partial u_1}=2(u_1-\hat x u_0),\quad d_0=\frac{\partial g_0}{\partial \vec(V_1)}=2\vec(V_1), \quad f_0 = g_0(x).
\end{gather*}
Under the stated assumptions, we can verify the matrix at the left-hand side of (\ref{eq:blockdiag-1}) indeed has full column rank. First, we apply \lemref{multiple-layers} to show that
\[
\mathbf{M}_{\ell-1}\lambda_{\ell-1}=0\iff\lambda_{\ell-1}=0,\quad \mathbf{M}_{\ell-2}\lambda_{\ell-2}=0\iff\lambda_{\ell-2}=0,\quad\ldots\quad,\mathbf{M}_{1}\lambda_{1}=0\iff\lambda_{1}=0.
\]
Substituting $\lambda_{\ell-1}=\lambda_{\ell-2}=\cdots=\lambda_{1}=0$ allows us to simplify (\ref{eq:blockdiag-1}) as
\begin{equation}\label{eq:blockdiag-2}
\begin{bmatrix}
2u_0 & a_0 \\
     & b_0 \\
     & d_0 \\
     & f_0
\end{bmatrix}
\begin{bmatrix}
z_{0}\\y_{0}
\end{bmatrix}=0
\qquad\iff\qquad
\begin{bmatrix}
z_{0}\\y_{0}
\end{bmatrix}=0.
\end{equation}

To show that the matrix at the left-hand side of (\ref{eq:blockdiag-2}) has full column rank. We consider two cases:
\begin{itemize}
	\item If $g_0(x)=0$. It follows from $\|u_{1}-\hat{x}u_{0}\|^{2}+\|V_{1}\|^{2}=\|(u_1-\hat x u_0,\vec(V_1))\|^2=\|\frac{1}{2}(b_0,d_0)\|^2=\rho^2>0$ and $u_0^2=1$ that
\[
\begin{bmatrix}
2u_0 & a_0 \\
     & b_0 \\
     & d_0 \\
     & f_0
\end{bmatrix}
\begin{bmatrix}
z_{0}\\y_{0}
\end{bmatrix}=0
\qquad\iff\qquad
\begin{bmatrix}
z_{0}\\y_{0}
\end{bmatrix}=0
\]
	\item If $g_0(x)>0$, then $y_0=0$. Substituting $y_0=0$ into (\ref{eq:blockdiag-2}), it then follows from  $u_0^2=1$ that
\[
2u_0z_0=0\qquad\iff\qquad z_0=0.
\]
\end{itemize}
\end{proof}

We can now extend \thmref{licq_bm} to show that (\ref{eq:bm_l2}) satisfies \ref{eq:licq} under an extra mild assumption, $\|V_1\|\neq 0$. The proof is summarized in the following corollary.
\begin{corollary}[LICQ for (\ref{eq:bm_l2})]
	Define $u_0$, $u$ and $V$ as in \thmref{licq_bm}. Let $x=(u_0,\{u_k\}_{k=1}^{\ell-1},\{\vec(V_k)\}_{k=1}^{\ell})$ satisfy constraints in (\ref{eq:bm_l2}), $g_0(x)\ge 0$, $g_{k,1}(x)\ge0$, $g_{k,2}(x)\ge0$ and $h_{k}(x)=0$ for all $k\in\{0,\ldots,\ell-1\}$, where 
\begin{gather*}
g_{0}(x)=\rho^{2}-\|u_{1}-\hat{x}u_{0}\|^{2}-\|V_{1}\|^{2},\qquad h_{0}(x)=u_{0}^2-1,\\
g_{0,1}(x)=u_{0}\cdot u_{1},\qquad g_{0,2}(x)=u_{0}\cdot\left(u_{0}-u_{1}\right),\\
g_{k,1}(x)=u_{0}\cdot u_{k+1},\qquad g_{k,2}(x)=u_{0}\cdot\left(u_{k+1}-W_{k}u_{k}-b_{k}u_{0}\right)\quad\forall k\in\{1,\ldots,\ell-1\},\\
h_{k}(x)=\diag[(u_{k+1}-W_{k}u_{k}-b_{k}u_0)u_{k+1}^{T}+(V_{k+1}-WV_{k})V_{k+1}^{T}]\quad\forall k\in\{1,\ldots,\ell-1\}.
\end{gather*}
Suppose that $x$ satisfies (\ref{eq:npcq}), and $\|V_1\|\neq 0$. Then, $x$ satisfies (\ref{eq:licq})
\end{corollary}
\begin{proof}
From \thmref{licq_bm}, to prove this corollary, it is suffice to show the following
\begin{equation*}
\begin{bmatrix}
2u_0 & a_0 & u_1^T & 2u_0-u_1^T \\
     & b_0 & u_0I  & -u_0I\\
     & d_0 \\
     & f_0 \\
     &     & g_{0,1}(x) \\
     &     &            & g_{0,2}(x)
\end{bmatrix}
\begin{bmatrix}
z_{0}\\y_{0}\\y_{0,1}\\y_{0,2}
\end{bmatrix}=0
\qquad\iff\qquad
\begin{bmatrix}
z_{0}\\y_{0}\\y_{0,1}\\y_{0,2}
\end{bmatrix}=0
\end{equation*}
where $(a_0,b_0,d_0,f_0)$ are defined in \thmref{licq_bm}. By hypothesis, $\|\frac{1}{2}d_0\|=\|V_1\|\neq 0$. This allows us to restate the desired claim as the following
\begin{equation}\label{eq:bm_l2_claim}
\begin{bmatrix}
2u_0 & u_1^T & 2u_0-u_1^T \\
     & u_0I  & -u_0I\\
     & g_{0,1}(x) \\
     &            & g_{0,2}(x)
\end{bmatrix}
\begin{bmatrix}
z_{0}\\y_{0,1}\\y_{0,2}
\end{bmatrix}=0
\qquad\iff\qquad
\begin{bmatrix}
z_{0}\\y_{0,1}\\y_{0,2}
\end{bmatrix}=0.
\end{equation}
Notice that we cannot jointly have $\e_i^Tg_{0,1}(x)=\e_i^Tg_{0,2}(x)=0$ for all $i\in\{1,\ldots,n_1\}$. Hence, each pair of $(\e_i^Ty_{0,1},\e_i^Ty_{0,2})$ has only three possible cases: $\e_i^Tg_{0,1}(x)=0,\e_i^Tg_{0,2}(x)\neq 0$; $\e_i^Tg_{0,1}(x)\neq 0,\e_i^Tg_{0,2}(x)= 0$; and $\e_i^Tg_{0,1}(x)\neq0,\e_i^Tg_{0,2}(x)\neq 0$. Of all three cases, it is clear that (\ref{eq:bm_l2_claim}) is true because $u_0^2=1$.

\end{proof}

\end{document}